\documentclass[10pt,letter]{article}

\usepackage{microtype}
\usepackage{graphicx}
\usepackage{subfigure}
\usepackage{booktabs} 
\usepackage{hyperref}
\RequirePackage{natbib}
\usepackage{algorithm}
\usepackage{amsmath}
\usepackage{amssymb}
\usepackage{mathtools}
\usepackage{amsthm}
\usepackage{authblk}
\usepackage{adjustbox}

\usepackage[letterpaper, lmargin=0.12\paperwidth, rmargin=0.12\paperwidth, tmargin=0.1\paperheight, bmargin=0.1\paperheight]{geometry}
\usepackage{mystyle}

\usepackage[capitalize,noabbrev]{cleveref}

\theoremstyle{plain}
\newtheorem{theorem}{Theorem}[section]
\newtheorem{proposition}[theorem]{Proposition}
\newtheorem{lemma}[theorem]{Lemma}

\theoremstyle{definition}

\newtheorem{assumption}[theorem]{Assumption}
\theoremstyle{remark}
\newtheorem{remark}[theorem]{Remark}

\title{Online Assortment and Price Optimization Under \\ Contextual Choice Models}
\date{}
\author{Yigit Efe Erginbas}
\author{Thomas A.~Courtade}
\author{Kannan Ramchandran}
\affil{UC Berkeley}

\begin{document}

\def\ARXIV{1}

\maketitle

\begin{abstract}
  We consider an assortment selection and pricing problem in which a seller has $N$ different items available for sale. In each round, the seller observes a $d$-dimensional contextual preference information vector for the user, and offers to the user an assortment of $K$ items at prices chosen by the seller. The user selects at most one of the products from the offered assortment according to a multinomial logit choice model whose parameters are unknown. The seller observes which, if any, item is chosen at the end of each round, with the goal of maximizing cumulative revenue over a selling horizon of length $T$. For this problem, we propose an algorithm that learns from user feedback and achieves a revenue regret of order \smash{$\wt{\mathcal{O}}(d \sqrt{K T} / L_0 )$} where $L_0$ is the minimum price sensitivity parameter. We also obtain a lower bound of order \smash{$\Omega(d \sqrt{T}/ L_0)$} for the regret achievable by any algorithm. 
\end{abstract}

\section{INTRODUCTION}

In online marketplaces, dynamic assortment selection and pricing for sequentially arriving buyers presents a challenge for online learning. Since the preferences of buyers are varying and uncertain, adaptive strategies are essential to meet their needs and maximize the effectiveness of offers. To address this problem, we investigate the application of online learning techniques for contextual assortment selection and pricing. Assortment selection involves the seller choosing a subset of items from a vast catalog to present to buyers, and dynamically assigning prices to the offered items. The overall goal is to maximize revenue over the course of repeated interactions.

Dynamic assortment selection and pricing strategies are deployed in a variety of online sectors including e-commerce (e.g., Amazon), food delivery (e.g., Uber Eats), and hospitality (e.g., Airbnb). With similar systems becoming ubiquitous in our daily lives, there is a growing opportunity to deliver tailored product recommendations and pricing adjustments. Therefore, it is crucial to consider data-driven approaches that can enhance user experiences and boost profitability in today's highly competitive digital industry. 

\begin{figure}[t]
\centering
\if\ARXIV1
    \includegraphics[width= 0.50 \columnwidth]{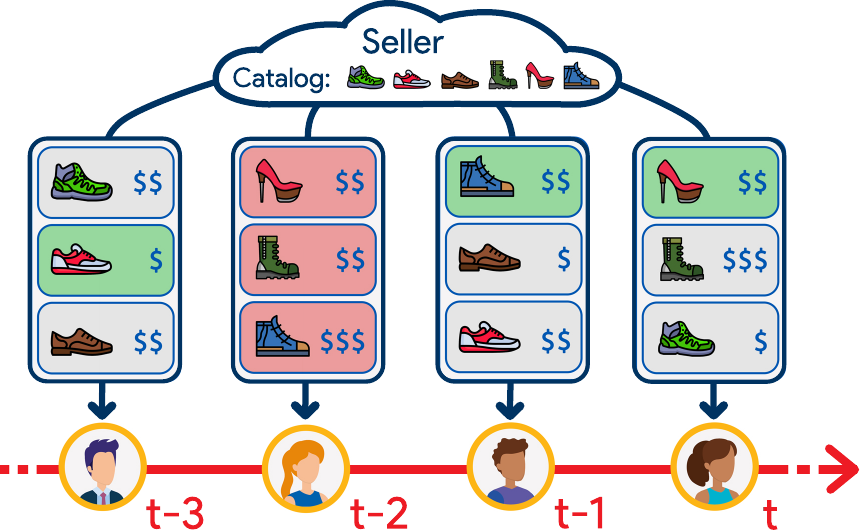}
\else
    \includegraphics[width= 0.99 \columnwidth]{images/assortment_system_diagram.pdf}
\fi
\caption{A seller has access to a catalog (set) of $N = 6$ distinct items, from which it can advertise to sequentially arriving users. In each round, the seller offers an assortment of $K = 3$ items at well-chosen prices. The user selects one of the products from the offered assortment (represented with a green background), or rejects all offered items (represented with a red background).}
\label{fig:system}
\vspace{-10pt}
\end{figure}

We design sequential \emph{assortment selection and pricing} algorithms that offer a sequence of assortments (menus) of up to $K$ items from a catalog of $N$ possible items. The learning agent (seller) sequentially selects assortments to offer and sets prices for the included items. After making assortment and pricing decisions in each round, the learning agent receives user feedback, which consists of the specific item chosen from the offered assortment. We assume that the item choice follows a multinomial logistic (MNL) model \citep{McFadden_1978}, which is one of the most widely used models in dynamic assortment optimization literature \citep{Caro_Gallien_2007, Agrawal_Avadhanula_Goyal_Zeevi_2017, Aouad_Levi_Segev_2018}. Because assortment-based offers are relevant to many industries that involve access to additional information about users, contextual choice models have gained significant traction in recent years \citep{Chen_Wang_Zhou_2020, Javanmard_Nazerzadeh_Shao_2020}. In alignment with this approach, we assume that the utility parameters in the MNL choice model are linear functions of $d$-dimensional context vectors that are revealed at each round. 

To address a range of real-world scenarios where price optimization is essential for maximal revenue, we incorporate the \emph{pricing} of items as a second component of the seller's problem. This largely differs from previous literature on sequential assortment selection, wherein prices are assumed to be predetermined \citep{Chen_Wang_Zhou_2020, Oh_Iyengar_2021}. The main challenge in our work is \textbf{the complex interdependence between assortment and pricing decisions}, an issue that existing methods are not designed to address.

In the process of offering a sequence of assortments with judiciously chosen prices, the seller's goal is to maximize the expected revenue accumulated over a time horizon of $T$ rounds. However, since the seller does not have knowledge of the parameters of the contextual choice model ahead of time, the decisions involve a trade-off between learning the choice model in order to increase long-term revenues and earning short-term revenues by leveraging the already-acquired information. 

\subsection{Overview of Our Algorithm}

Our algorithm selects \emph{optimistic} assortment and prices that balance the trade-off between exploration and exploitation. This is accomplished by deriving tight upper bounds for the utility functions in the MNL model. In contrast to the dynamic assortment selection literature, which only establishes a pointwise upper bound for the value of an assortment, we construct price-dependent functions that upper bound the values across all price points. This construction allows us to quantify the varying uncertainty for different prices and successfully assess the trade-off involved in joint optimization of assortment and prices. 

To construct these utility upper bounds, we need to obtain estimates of the parameters in the MNL model. However, dynamic estimation of MNL parameters has been exclusively studied under the assumption of fixed prices and the state-of-the-art techniques result in a dependence on a problem-dependent parameter $\kappa$ \citep{ Oh_Iyengar_2021}~\footnotemark[1]. If we consider extending these analyses to include price selection, we observe that the $\kappa$ parameter strongly depends on the assortment size $K$ and the minimum price sensitivity $L_0$~\footnotemark[2]. In particular, the worst-case dependency is $\kappa=K^{2 + 1 / L_0}$ which would translate into a $\mathcal{O}(K^{2 + 1 / L_0} d \sqrt{T})$ regret bound. Hence, a direct extension of existing approaches is far from optimal, especially when the minimum price sensitivity parameter $L_0$ is small.

We tackle this issue by constructing better estimates of the Fisher Information Matrix for the parameters of the MNL model, which enables us to eliminate the $\kappa$ dependence. The key to our analysis is a novel Bernstein-type inequality for self-normalized vector-valued martingales which we derive based on techniques introduced in \citet{Faury_Abeille_Calauzenes_Fercoq_2020}. 

\footnotetext[1]{The parameter $\kappa$ inversely scales with the minimum probability of each item being chosen. For a precise definition of $\kappa$ and additional details, please refer to Appendix~\ref{sect:estimation_importance}.}
\footnotetext[2]{The minimum price sensitivity $L_0$ is a lower bound for the rate of decay of the utility as a function of the prices.}

Consistent with the sequential decision-making literature, we measure the performance of algorithms using a relevant notion of regret, defined as the difference between the expected revenue generated by the algorithm and the offline optimal expected revenue when all parameters are known. We show that our algorithm achieves a revenue regret of order \smash{$\wt{\mathcal{O}}(d \sqrt{K T} / L_0)$}, which, as we show, is the best possible up to logarithmic factors in $d$, $T$, and minimum price sensitivity $L_0$~\footnotemark[2].



\vspace{-5pt}
\subsection{Our Contributions}
To the best of our knowledge, we are the first to address the problem of dynamic contextual assortment selection and pricing \textbf{simultaneously}. Our contributions are:
\begin{itemize}[wide=0.5pt, leftmargin=*]
    \item \emph{Formulation:} We introduce and formalize the problem of sequential assortment and price optimization under contextual multinomial logit choice model.
    \item \emph{Regret upper bound:} We develop an algorithm for the contextual assortment selection and pricing problem (Algorithm \ref{alg:seq_assortment}). We show that it achieves \smash{$\wt{O}(d \sqrt{K T}/L_0)$} regret in $T$ rounds where $d$ is the dimension of the context vectors, $K$ is the assortment size, and $L_0$ is the minimum price sensitivity. We further improve the time and space complexity of our algorithm by leveraging online Newton step (ONS) techniques for parameter estimation in Algorithm \ref{alg:seq_assortment_online}.
    \item \emph{Regret lower bound:} We show that for any algorithm, there exists an adversarial problem instance such that it incurs \smash{$\Omega(d \sqrt{T}/ L_0)$} regret. Therefore, Algorithm \ref{alg:seq_assortment} enjoys optimal regret up to logarithmic terms in $d$, $T$, $N$, and $L_0$.
    \item \emph{Assortment and price optimization algorithm:} As a part of our solution, we develop an efficient algorithm (Algorithm~\ref{alg:optimization_algo}) to find the optimal assortment and prices under the MNL model with any differentiable and strictly decreasing utility function.
\end{itemize}

\begin{remark}
The gap between our upper and lower bounds for regret is on the order of $\mathcal{O}(\sqrt{K})$, but given that the maximum assortment size is typically small (e.g., 5 to 20) in most real-world scenarios, this difference might be considered non-critical. 
\end{remark}

 \vspace{-5pt}
\subsection{Related Works}

\paragraph{Generalized Linear Bandits} Linear bandits, generalized linear bandits, and their variants have been extensively studied in the context of sequential decision-making with contextual information \citep{abbasi_2011, Chu_Li_Reyzin_Schapire_2011, Li_Lu_Zhou_2017}. Building on this literature, recent works by \cite{ban2021personalized, xu2024pricing, wang2025dynamic} have explored parametric contextual pricing for a single item under generalized linear demand models, where demand depends solely on the item's own price. In contrast, the MNL model we consider captures demand through a choice model, accounting for the influence of all item prices in the assortment. Another line of research examines combinatorial variants of the contextual bandit problem, often incorporating semi-bandit or cascading feedback \citep{chen_2013, Qin_Chen_Zhu_2014, kveton_2015, Zong_Ni_Sung_Ke_Wen_Kveton_2016}. However, these approaches cannot account for substitution effects, as their choice models fail to consider which other items are included in the assortment.

\renewcommand{\arraystretch}{1.3}
\renewcommand\theadfont{\normalsize\bfseries}

\begin{table*}[t]
\centering
\caption{Comparison of related works, provided regret bounds, and computational complexity per time step of given algorithms. $T$ is the number of rounds, $K$ is the assortment size, $N$ is the total number of items, $d$ is the feature dimension. The big-$\mathcal{O}$ and big-$\Omega$ notations denote the regret upper and lower bounds, respectively. To the best of our knowledge, we are the first to address the problem of simultaneous contextual assortment selection and pricing.}
\begin{adjustbox}{width=\textwidth}
\begin{tabular}{llllll}
 \toprule
 &  \thead{Context}&  \thead{Assortment}&  \thead{Pricing}& \thead{Regret}& \thead{Computational \\ Complexity \footnotemark[3]}\\ 
 \midrule
 \citet{Agrawal_Avadhanula_Goyal_Zeevi_2019}& No & Yes & No & $\wt{\mathcal{O}}( \sqrt{N T})$, $\Omega( \sqrt{N T/K})$& $\Theta \left( N \right)$\\ 
 \citet{Miao_Chao_2018} & No & Yes & Yes & $\wt{\mathcal{O}}(\sqrt{N T})$ \footnotemark[4] & unknown \footnotemark[5]\\
 \citet{Chen_Wang_Zhou_2020}& Yes & Yes & No & $\wt{\mathcal{O}}( d\sqrt{T})$, $\Omega( d \sqrt{T} / K)$& $\Theta ( K T +  \binom{N}{K} )$\\ 
 \citet{Oh_Iyengar_2021}& Yes & Yes & No & $\wt{\mathcal{O}}(  \kappa d \sqrt{T} )$ & $\Theta \left( N \right)$\\ 
 \citet{Javanmard_Nazerzadeh_Shao_2020}& Yes & No & Yes & $\mathcal{O}(\log(dT)\sqrt{T})$& $\Theta ( N \sqrt{T} )$ \\ 
\citet{Perivier_Goyal_2022} & Yes & No & Yes & $\wt{\mathcal{O}}(d\sqrt{T})$ \footnotemark[6] & $\Theta \left( N \right)$ \\
\citet{Perivier_Goyal_2022} & Yes & Yes & No & $\wt{\mathcal{O}}(d K \sqrt{T})$  & unknown \footnotemark[7]\\
\thead{CAP} (Algorithm \ref{alg:seq_assortment})& Yes & Yes & Yes & $\wt{\mathcal{O}}(d \sqrt{K T} / L_0)$, $\Omega(d \sqrt{T} / L_0)$ & $\Theta \left( K T + N \right)$\\ 
\thead{CAP-ONS} (Algorithm \ref{alg:seq_assortment_online})& Yes & Yes & Yes & $\wt{\mathcal{O}}(d K \sqrt{T} / L_0)$ & $\Theta \left( N \right)$\\ 
 \bottomrule
\end{tabular}
\end{adjustbox}
\label{tab:literature_comparison}
\if\ARXIV0
    \vspace{-15pt}
\fi
\end{table*}

\paragraph{Dynamic Assortment Selection} There has been an emerging body of literature on multinomial logit (MNL) bandits in both non-contextual \citep{Cheung_Simchi_2017, Agrawal_Avadhanula_Goyal_Zeevi_2019} and contextual settings \citep{oh2019thompson, Chen_Wang_Zhou_2020, agrawal2020tractable, Oh_Iyengar_2021}. While these studies address the sequential assortment selection problem under the MNL choice model, they assume fixed prices for the items.

Incorporating variable prices directly into these algorithms, such as in \citet{Chen_Wang_Zhou_2020} or \citet{agrawal2020tractable}, proves impractical as they compute separate upper confidence bounds for the value of each of the $\binom{N}{K}$ possible assortments and choose the one with maximum value. With the addition of pricing into the problem, these upper bounds become functions of the prices for all items and make the optimization even harder. \citet{Oh_Iyengar_2021} offers a polynomial-time contextual MNL-bandit algorithm that computes upper confidence bounds for the value of each item rather than each assortment. However, their algorithm and analysis translates into a $\mathcal{O}(K^{2 + 1 / L_0} d \sqrt{T})$ regret bound when we introduce price optimization (see Appendix \ref{sect:estimation_importance} for details). Hence, this approach is also far from optimal. Recently, \citet{Perivier_Goyal_2022} has also provided an assortment selection algorithm with improved regret bounds. However, their analysis only works under the assumption of uniform prices across items, which does not hold in our setting.

\footnotetext[3]{We provide the amortized computational complexity per time step with respect to parameters $N$, $K$, and $T$.}
\footnotetext[4]{The given regret bounds are for Bayesian regret.}
\footnotetext[5]{It depends on the prior of the parameter.}
\footnotetext[6]{This result considers an adversarial arrival model.}
\footnotetext[7]{Their confidence sets are expensive to compute.}

\paragraph{Dynamic Pricing} The problem of dynamic pricing has been typically modeled as a variant of the multi-armed bandit problem that aims to maximize revenue from selling copies of a single good to sequentially arriving users \citep{kleinberg2003, Besbes_Zeevi_2009, Bubeck_Devanur_Huang_Niazadeh_2019, Paes_Leme_Schneider_2018, Xu_Wang_2021}. Our contribution stands out by considering the combinatorial aspect of the assortment selection problem faced in simultaneously offering multiple items. Recent studies by \citet{Javanmard_Nazerzadeh_Shao_2020} and \citet{Perivier_Goyal_2022} consider the problem of pricing multiple items that are offered under the MNL choice model. However, in contrast to our work, these frameworks assume that all available items are offered to the buyer. To the best of our knowledge, there is only one work \citep{Miao_Chao_2018} which considers the joint problem of assortment optimization and pricing under unknown demand information. However, this work does not utilize a contextual model and assumes that the unknown parameter is randomly drawn from a prior distribution known to the algorithm.

\section{PROBLEM DEFINITION}

\paragraph{Notation:} We use bold lowercase font for vectors $\vect{x}$ and bold uppercase font for matrices $\vect{X}$. For a vector $\vect{x}$, we denote its $i$-th entry by $x_i$ and we use $\|\vect{x}\|$ to denote its $\ell^2$-norm. For two vectors $\vect{x}$ and $\vect{y}$, we use $(\vect{x}; \vect{y})$ to denote their concatenation and use $\langle \vect{x}, \vect{y} \rangle$ to denote their inner product. For a vector $\vect{x}$ and a positive-definite matrix $\vect{W}$, we use $\|\vect{x}\|_{W}$ to denote the weighted $\ell^2$-norm. For any positive integer $N$, we use $[N]$ to denote the set $\{1, 2, \dots, N\}$. 

We consider the problem of online assortment selection and pricing for selling items to sequentially arriving buyers. We denote the set of available items by $[N]$ and consider that the seller is constrained to offer at most $K$ items to each arriving buyer. Accordingly, we let $\mathcal{S}_K := \{S \subseteq [N] : |S| \leq K\}$ denote the set of all possible assortments that the seller can choose to offer.

At each time $t \in [T]$, the seller observes random feature vectors $\vect{x}_{ti} \in \mathbb{R}^d$ for each item $i \in [N]$. Given this contextual information, the seller offers an assortment of items $S_t \in \mathcal{S}_K$ and chooses a price $p_{ti} \in \bbR$ for each offered item $i \in S_t$. At the end of each round $t$, the seller observes only the purchase decision $i_t \in S_t \cup \{0\}$ of the buyer and obtains revenue $p_{t i_t}$. Here, $\{ 0 \}$ represents the no-purchase option (or outside option), which indicates that the user did not choose any item offered in $S_t$,  resulting in revenue $p_{t 0} = 0$. For convenience, we let $\vect{p}_t \in \bbR^N$ denote the collection of prices chosen for all items.

For a given assortment $S_t$ and price vector $\vect{p}_t$, the buyer's decision $i_t$ is a categorical random variable with support $S_t \cup \{ 0 \}$. We model this decision via the widely used multinomial logit (MNL) choice model \citep{McFadden_1978} under a linear contextual utility function. Formally, the choice probability for each item $i \in S_t$ (and the no-purchase option) is assumed to be given as in the following assumption.


\begin{assumption}[Multinomial logit choice under linear contextual utility]

The utility of the buyer at time $t$ for item $i$ is given by the linear model
\begin{align*}
    u_{ti}(p) = \langle \vect{\psi}^*, \vect{x}_{ti} \rangle - \langle \vect{\phi}^*, \vect{x}_{ti} \rangle \cdot p
\end{align*}
where $\vect{\psi}^* \in \bbR^d$ and $\vect{\phi}^* \in \bbR^d$ are time-invariant parameter vectors unknown to the seller. In this model, the $\alpha_{ti} := \langle \vect{\psi}^*, \vect{x}_{ti} \rangle$ term represents the buyer's base valuation of the item while the $\beta_{ti} := \langle \vect{\phi}^*, \vect{x}_{ti} \rangle$ term represents the buyer's price sensitivity.

Then, given an assortment $S_t$ with prices $\vect{p}_t$, the probability that the buyer selects item $i \in S_t$ is
\begin{align*}
    q_t( i | S_t,  \vect{p}_t) := \frac{\exp\{u_{ti}(p_{ti})\}}{1 + \sum_{j \in S_t} \exp\{u_{tj}(p_{tj})\}}.
\end{align*}
Consequently, the probability of no purchase is
\begin{align*}
    q_t( 0 | S_t,  \vect{p}_t) := \frac{1}{1 + \sum_{j \in S_t} \exp\{u_{tj}(p_{tj})\}}.
\end{align*}
\label{assumption_mnl}
\vspace{-10pt}
\end{assumption}

Under this model, the expected revenue at time $t$ is
\begin{align}
    R_t(S_t, \vect{p}_t) := \sum_{i \in S_t} p_{ti} \cdot q_t( i | S_t,  \vect{p}_t)
\label{eqn_rev_t}
\end{align}
for any selection of assortment $S_t \in \mathcal{S}_K$ and price vector $\vect{p}_t \in \bbR^N$. Thus, for a sequence of assortments $S_t \in \mathcal{S}_K$ and price vectors $\vect{p}_t \in \bbR^N$ chosen over time, the cumulative expected revenue can be written as $\sum_{t = 1}^{T} R_t(S_t, \vect{p}_t)$.

After the seller decides on the assortment $S_t \in \mathcal{S}_K$ and prices $\vect{p}_t \in \bbR^N$ to offer to the user at each time $t$, the user reports the item $i_t \in S_t \cup \{0\}$ that they have decided to purchase. We denote by $H_t$ the history $\{\{\vect{x}_{\tau i}\}_{i \in [N]}, S_{\tau}, \vect{p}_{\tau}, i_{\tau}\}_{\tau = 1}^{t-1}$ of observations available to the seller when choosing the next set of assortment $S_t \in \mathcal{S}_K$ along with the next price vector $\vect{p}_t$. Then, the seller agent employs a policy $\bm{\pi} = \{ \pi^t | t \in [T]\}$, which is a sequence of functions, each mapping the history $H_t$ and the context vectors $\{\vect{x}_{t i}\}_{i \in [N]}$ to an action $(S_t, \vect{p}_t) \in \mathcal{S}_K \times \bbR^N$.

Given the contextual information at every round $t$, the task of the seller is to sequentially offer the items to users at well-chosen prices such that it can achieve maximal revenue. To evaluate policies in achieving this objective, we define the \emph{regret} metric that measures the gap between the expected revenue of policy $\vect{\pi}$ and that of the offline optimal sequence of assortments and prices. The regret $\mathcal{R}_T$ for a time horizon of $T$ periods is defined as
\begin{align*}
    \mathcal{R}_T := \sum_{t = 1}^{T} R_t(S_t^*, \vect{p}_t^*) - \sum_{t = 1}^{T} R_t(S_t, \vect{p}_t),
\end{align*}
where $(S_t^*, \vect{p}_t^*)$ denotes an offline optimal assortment and price selection that satisfies
\begin{align}
    (S_t^*, \vect{p}_t^*) \in \argmax_{\substack{S \in \mathcal{S}_K \\ \vect{p} \in \bbR^N}} R_t(S, \vect{p}).
\label{opt_action}
\end{align}
\vspace{-15pt}

Based on the definition of the regret metric, we see that regret minimization is equivalent to maximizing the cumulative expected revenue.
\vspace{-5pt}
\section{ASSORTMENT AND PRICE OPTIMIZATION}
\vspace{-5pt}
\label{sect_optimal_assortment}

As stated in Assumption \ref{assumption_mnl}, we assume that buyers' purchase decisions are given by a multinomial logit (MNL) model. Therefore, the assortment and price optimization at time $t$ can be formulated as
\begin{align*}
    \max_{\substack{S_t \in \mathcal{S}_K \\ \vect{p}_t \in \bbR^N}} R_t(S_t, \vect{p}_t) &= \max_{\substack{S_t \in \mathcal{S}_K \\ \vect{p}_t \in \bbR^N}} \frac{\sum_{i \in S_t} p_{ti} \exp\{u_{ti}(p_{ti})\}}{1 + \sum_{j \in S_t} \exp\{u_{tj}(p_{tj})\}}.
\end{align*}

We recall that the utility functions are given by linear form $u_{ti}(p) =  \langle \vect{\psi}^*, \vect{x}_{ti} \rangle - \langle \vect{\phi}^*, \vect{x}_{ti} \rangle \cdot p$ and make the following regularity assumption.

\begin{assumption}[Minimum Price Sensitivity]
    There exists a constant $L_0 > 0$ such that utility functions satisfy $u_{ti}'(p) = - \langle \vect{\phi}^*, \vect{x}_{ti} \rangle \leq - L_0$ for all $t \in [T]$ and $i \in [N]$, almost surely.
    \label{assumption_positive_sens}
\end{assumption}

This assumption ensures that the utility function $u_{ti}(p)$ is strictly decreasing in price and hence infinity is a so-called null price, i.e. $\lim_{p \to \infty} p e^{u_{ti}(p)} = 0$, so that the objective function $R_t(S_t, \vect{p}_t)$ has a finite maximum. 


\paragraph{Characterization of Optimality} Even though the true utility functions are assumed to be linear, our learning algorithm will require us to solve for the optimum assortment and prices under broader classes of utility functions. Hence, in the next proposition, we characterize optimality under any differentiable and strictly decreasing utility function $h_{ti}(p)$.

\begin{restatable}[Optimum assortment and prices]{proposition}{optassortment}
Suppose utility functions $h_{ti}(p)$ are differentiable and strictly decreasing for all items $i \in [N]$. Let $B_t$ be the unique solution of the fixed point equation
\begin{align}
    B = \max_{S \in \mathcal{S}_K} \sum_{i \in S} v_{ti}(B)
    \label{eq_fixed_point_assortment}
\end{align}
where $v_{ti}(B) := \max_{p \in \bbR} \left\{ f_{ti}(p) : p + 1 / h_{ti}'(p) = B \right\}$ and $f_{ti}(p) := - e^{h_{ti}(p)} / h_{ti}'(p)$. Then, the optimum assortment $S^*_t$ is the assortment $S$ that achieves the maximum in the optimization problem~\eqref{eq_fixed_point_assortment}, the optimum prices are
\begin{align*}
p_{ti}^* = \argmax_{p \in \bbR} \left\{ f_{ti}(p) : p + 1 / h_{ti}'(p) = B_t \right\},
\end{align*}
and the optimum revenue achieved by $(S_t^*, \vect{p}_t^*)$ is $B_t$.
\label{prop_opt_assortments_and_prices}
\end{restatable}

\begin{proof}

(Sketch) First, we write the first-order necessary conditions for the optimality of prices as 
\if\ARXIV1
    \begin{align*}
    \nabla_{\mathbf{p}} \left\{\sum_{i \in S} p_{i} q_{t}(i | S, \mathbf{p}) \right\} = \mathbf{0}
    \end{align*} 
\else
    $\nabla_{\mathbf{p}} \left\{\sum_{i \in S} p_{i} q_{t}(i | S, \mathbf{p}) \right\} = \mathbf{0}$
\fi
under any fixed assortment $S$. Using the structure of the MNL model, this necessary condition reduces to $\sum_{ i \in S} p_{i} q_{t}(i | S, \mathbf{p}) = p_{j} + \frac{1}{h_{tj}'(p_{j})} $ for all $j \in S$. Note that the left-hand side of the equation is equal to the revenue obtained at prices $\mathbf{p}$.  Therefore, the pricing problem can be written as maximizing $B$ subject to $B = \sum_{ i \in S} p_{i} q_{ti}(\mathbf{p})$ and $B = p_i + \frac{1}{h_i'(p_i)}$ for all $i \in S$. 

Furthermore, using the form of MNL, we can show that the condition $B = \sum_{ i \in S} p_{i} q_{ti}(\mathbf{p})$ is equivalent to $B = \sum_{ i \in S} f_{ti}(p_{i})$ where $f_{ti}(p) = - e^{h_{ti}(p)} / h_{ti}'(p)$. Therefore, the pricing problem can be written as maximizing $B$ subject to conditions (a) $B = \sum_{ i \in S} f_{ti}(p_i)$ and (b) $B = p_i + \frac{1}{h_i'(p_i)}$ for all $i$.

To convert this problem into a fixed point equation, we define $v_{ti}(B) =  \max_{p \in \mathbb{R}} \left\{ f_{ti}(p) : p + 1 / h_{ti}'(p) = B \right\}$, which corresponds to the maximum value the right hand side of condition (a) can take when the condition (b) is satisfied. As we show in our proof, $v_{ti}(B)$ is a continuous and strictly decreasing function of $B$. This implies that the optimum $B$ value uniquely satisfies the fixed point equation $B = \sum_{i \in S} v_{ti}(B)$. Lastly, we incorporate the assortment selection into this optimization problem and show that the assortment and price optimization can be achieved by solving the fixed point equation \ref{eq_fixed_point_assortment}. See Appendix~\ref{sect:appendix_solving_assortment_price_optimization} for details.
\end{proof}

\begin{remark}
    \citet{Wang_2013} provides a weaker version of Proposition~\ref{prop_opt_assortments_and_prices} that requires the additional assumption that the utility functions $h_{ti}(p)$ are twice-differentiable and concave in $p$. Even though this assumption holds for linear utility functions, the learning algorithm that we will introduce in the following sections requires us to solve the assortment and price optimization problem under non-concave utility functions.
\end{remark}

\paragraph{Optimization Algorithm} To solve the fixed point equation~\eqref{eq_fixed_point_assortment}, we start by showing that its right-hand side is a positive and strictly decreasing function in $B$. We also show that if the utility functions satisfy $h_{ti}(0) \leq 1$ and $h_{ti}'(p) \leq -L_0$ for all $p \in \bbR$, then the solution to~\eqref{eq_fixed_point_assortment} lies in the interval $[0, P_0]$ for some $P_0 = \mathcal{O}(\log K /L_0)$. Note that this condition holds for the true utility function $u_{ti}(p)$. Under this condition, we can use a binary-search based algorithm to find the fixed point over the interval $[0, P_0]$. For future reference, we describe this procedure in Algorithm~\ref{alg:optimization_algo}.

\begin{algorithm}[ht]
\caption{Assortment and price optimization}
\begin{algorithmic}[1]
\State \textbf{Input:} utility functions $h_{ti}(p)$ for $i \in [n]$
\State \textbf{Input:} precision parameter $\epsilon$ 
\State \textbf{Input:} search interval $[0, P_0]$ 
\State $B_{\ell} = 0$, $B_{r} = P_0$
\While{$B_{r} - B_{\ell} > \epsilon$}
\State $B \gets (B_{r} + B_{\ell})/2$
\For{$i \in [N]$}
\State Find $\mathcal{P}_{ti}(B) = \{p : p + 1 / h_{ti}'(p) = B\}$
\State $v_{ti} \gets \max\{f_{ti}(p) : p \in \mathcal{P}_{ti}(B)\}$
\EndFor
\State $B^* = \max_{S \in \mathcal{S}_K} \sum_{i \in S} v_{ti}$
\IfThenElse{$B > B^*$}{$B_{r} \gets B$}{$B_{\ell} \gets B$}
\EndWhile
\State \textbf{Output:} $B^*$
\end{algorithmic}
\label{alg:optimization_algo}
\end{algorithm}

\paragraph{Computational Complexity} The main difficulty in running Algorithm~\ref{alg:optimization_algo} is finding the set $\mathcal{P}_{ti}(B)$ that contains the solutions for the equation $p + 1 / h_{ti}'(p) = B$ for any given $B$. Fortunately, for utility functions $h_{ti}(p)$ that we will use in Algorithms~\ref{alg:seq_assortment} and~\ref{alg:seq_assortment_online}, we can show that there are only a small number of solutions (i.e., $\Theta(1)$), and these solutions can be efficiently computed. (See Appendix~\ref{sect:computation} for details.) Since each iteration of this binary-search based algorithm requires us to compute the $v_{ti}$ value for all $i \in [N]$, the algorithm has an overall computational complexity of $\Theta(N \log(P_0 / \epsilon))$ for any arbitrary precision $\epsilon$.
\vspace{-5pt}

\section{ONLINE LEARNING}

In this section, we discuss how to estimate parameters based on user choices, introduce our online learning algorithms, and provide our regret bounds.

\subsection{MLE for Multinomial Logistic Regression}

Since the seller does not have access to problem parameters $\vect{\psi}^* \in \bbR^d$ and $\vect{\phi}^* \in \bbR^d$, it cannot directly compute the optimum assortments and prices. Therefore, it needs to construct an estimate of the parameters based on the history $H_t$ of observations.

For convenience, we let $\vect{\theta} = (\vect{\psi}, \vect{\phi})$ and $\widetilde{\vect{x}}_{t i} = ( \vect{x}_{t i}, - p_{t i} \vect{x}_{t i} )$ denote the extended parameter and feature vectors such that $\langle \vect{\theta}, \widetilde{\vect{x}}_{t i} \rangle = \langle \vect{\psi}, \vect{x}_{ti} \rangle - \langle \vect{\phi}, \vect{x}_{ti} \rangle \cdot p_{ti}$.

Then, we write the MNL choice probabilities under some parameter $\vect{\theta} = (\vect{\psi}, \vect{\phi})$ using the notation
\begin{align*}
    q_{ti}(\vect{\theta}) &= \frac{e^{\langle \vect{\theta}, \widetilde{\vect{x}}_{t i} \rangle}}{1 + \sum_{j \in S_t} e^{\langle \vect{\theta}, \widetilde{\vect{x}}_{t j} \rangle}}.
\end{align*}
With this notation, the negative log-likelihood function for the observations up to time $t$ is given by 
\vspace{-2pt}
\begin{align}
    \ell_t (\vect{\theta}) := -  \sum_{s = 1}^{t-1} \log q_{s i}( \vect{\theta} ).
    \label{log_likelihood}
\end{align}

\vspace{-5pt}
The maximum likelihood estimator is the parameter $\widehat{\vect{\theta}}_t$ that minimizes $\ell_t (\vect{\theta})$ over the parameter space. Since $\ell_t (\vect{\theta})$ is convex, we can use gradient-based convex optimization methods to find an MLE solution \citep{Boyd_Vandenberghe_2004}. See Appendix \ref{sect:mle} for details.
\vspace{-5pt}
\subsection{Algorithm}\label{sect:algorithm}

\begin{figure}[t]
\centering
\if\ARXIV1
\includegraphics[width= 0.6 \linewidth]{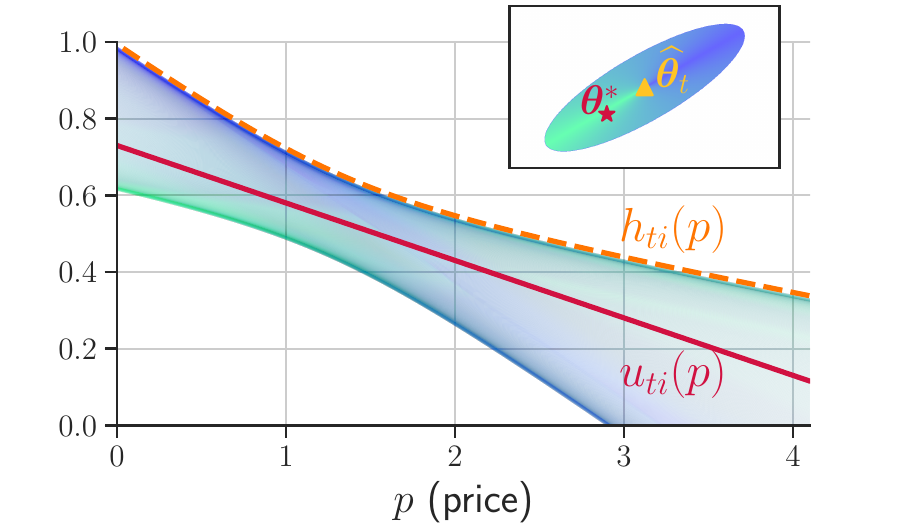}
\else
\includegraphics[width= 0.95 \linewidth]{images/utility_function.pdf}
\fi
\caption{The confidence region depicted in the top right corner contains the true parameter $\vect{\theta}^*$ with high probability. Each parameter in the confidence set corresponds to a different linear function and we construct $h_{ti}(p)$ as a tight upper bound to $u_{ti}(p)$.}
\label{fig:utility_function}
\if\ARXIV0
    \vspace{-10pt}
\fi
\end{figure}

Our core idea is to construct a tight, high-confidence upper bound for the revenue $R_t(S, \vect{p})$ as a function of $S \in \mathcal{S}_K$ and $\vect{p} \in \bbR^N$, and to determine the assortments and prices according to the \emph{optimisim} principle in order to ensure low regret. The upper confidence bound (UCB) techniques and the \emph{optimism in the face of uncertainty} principle have been widely known to be effective in balancing the exploration and exploitation in many bandit problems \citep{Lattimore_Szepesvari_2020, abbasi_2011, Li_Lu_Zhou_2017}. The key distinction of our approach lies in the construction and analysis of functional upper bounds, which capture the continuous dependence of revenue on prices. In particular, we construct a pointwise confidence upper bound $h_{ti}(p)$ for each utility function $u_{ti}(p)$, i.e., $h_{ti}(p) \geq u_{ti}(p)$ for all $p \in \bbR$ with high probability. In order to achieve low regret rates, it is crucial to obtain tight upper bounds as depicted in Figure \ref{fig:utility_function}.

\begin{algorithm*}[t]
\caption{CAP: Contextual Assortment and Pricing under MNL Model}
\begin{algorithmic}[1]
\State \textbf{Input:} initialization rounds $T_0$, confidence parameters $\{\alpha_t\}_{t \in [T]}$, minimum price sensitivity $L_0$
\State $\vect{V}_1 \gets \vect{0} \in \bbR^{2d \times 2d}$
\For{$t = 1, 2, \dots, T_0-1$} \Comment{Initialization rounds}
    \State Choose $S_t$ uniformly at random from $\{S \subseteq [n] : |S| \leq K\}$
    \State Choose $p_{ti}$ independently and uniformly at random from $[1,2]$ for all $i \in S_t$
    \State Offer assortment $S_t$ at price $\vect{p}_t$ and observe $i_t$
    \State $\vect{V}_{t+1} \gets \vect{V}_{t} + \frac{1}{K^2} \sum_{i \in S_t} \wt{\vect{x}}_{ti} \wt{\vect{x}}_{ti}^\top$
\EndFor
\For{$t = T_0, T_0 + 1, \dots, T$}
    \State Compute $\vect{\wh{\theta}}_t = (\vect{\wh{\psi}}_t, \vect{\wh{\phi}}_t)$ by minimizing \eqref{log_likelihood} \Comment{MLE computation}
    \State Let $g_{ti}(p) := \alpha_t \|(\vect{x}_{ti}, -p\vect{x}_{ti})\|_{\vect{V}_{t}^{-1}}$ for all $i \in [n]$ \Comment{Price-dependent confidence function}
    \State Let $\wt{h}_{ti}(p) := \langle \widehat{\vect{\psi}}_t, \vect{x}_{ti} \rangle - \langle \widehat{\vect{\phi}}_t, \vect{x}_{ti} \rangle \cdot p + g_{ti}(p)$ for all $i \in [n]$
    \vspace{2pt}
    \State Let $ h_{ti}(p) := \min_{p' \leq p} \left\{ \wt{h}_{ti} (p') - L_0 (p - p') \right\}$ for all $i \in [n]$ \Comment{Utility function estimate}
    \vspace{2pt}
    \State Choose $(S_t, \vect{p}_t)$ using Algorithm \ref{alg:optimization_algo} with estimated utility functions $h_{ti}(p)$ 
    \vspace{2pt}
    \State Offer assortment $S_t$ at price $\vect{p}_t$ and observe $i_t$
    \vspace{2pt}
    \State $\vect{V}_{t+1} \gets \vect{V}_{t} + \sum_{i \in S_t} q_{ti}(\vect{\wh{\theta}}_t) \wt{\vect{x}}_{ti} \wt{\vect{x}}_{ti}^\top - \sum_{i \in S_t} \sum_{j \in S_t} q_{ti}(\vect{\wh{\theta}}_t) q_{tj}(\vect{\wh{\theta}}_t) \wt{\vect{x}}_{ti} \wt{\vect{x}}_{tj}^\top$ \Comment{Information estimate}
\EndFor
\end{algorithmic}
\label{alg:seq_assortment}
\vspace{-2pt}
\end{algorithm*}

We offer randomly selected assortments and prices for the first $T_0$ rounds to ensure that our maximum likelihood estimates $\wh{\vect{\theta}}_t = (\wh{\vect{\psi}}_t, \wh{\vect{\phi}}_t)$ in subsequent rounds are sufficiently close to the true parameter $\vect{\theta}^*$. This allows us to construct a matrix $\vect{V}_t$ as a tight estimate of the Fisher Information Matrix around $\vect{\theta}^*$ (please refer to Algorithm~\ref{alg:seq_assortment} for the definition of $\vect{V}_t$). Then, we obtain confidence regions of the form $\{ \vect{\theta} : \|\vect{\theta} - \wh{\vect{\theta}}_t\|_{\vect{V}_{t}} \leq \alpha_t\}$ for some confidence radius $\alpha_t$ such that $\vect{\theta}^*$ is contained within the region with high probability. In contrast to prior works \citep{Chen_Wang_Zhou_2020, Oh_Iyengar_2021}, we use estimated choice probabilities $q_{ti}(\wh{\vect{\theta}}_t)$ in our $\vect{V}_t$ construction, which is the key in achieving a better scaling of $\alpha_t$ with respect to $K$ and $L_0$.

Based on these confidence regions for the parameter, we obtain an intermediate utility upper bound
\begin{align*}
\wt{h}_{ti}(p) := \langle \widehat{\vect{\psi}}_t, \vect{x}_{ti} \rangle - \langle \widehat{\vect{\phi}}_t, \vect{x}_{ti} \rangle \cdot p + g_{ti}(p)
\end{align*}
where $g_{ti}(p) := \alpha_t \|(\vect{x}_{ti}, -p \cdot \vect{x}_{ti})\|_{\vect{V}_{t}^{-1}}$ is a \emph{price-dependent} confidence bonus. Note that \smash{$\wt{h}_{ti}(p)$} is a convex and differentiable function. However, it is not necessarily a decreasing function and hence we cannot immediately use our Proposition \ref{prop_opt_assortments_and_prices} to find optimum assortments and prices under $\wt{h}_{ti}(p)$. To resolve this problem, we use the fact that $u_{ti}'(p) \leq -L_0$ for all $p \in \bbR$, and construct a tighter upper bound
\begin{align*}
h_{ti}(p) := \min_{p' \leq p} \left\{ \wt{h}_{ti} (p') - L_0 (p - p') \right\}.
\end{align*}
As a result, we can replace each $u_{ti}(p)$ in \eqref{eqn_rev_t} with $h_{ti}(p)$ to obtain an upper bound for the revenue function as
\begin{align}
    \wt{R}_t(S, \vect{p}) := \frac{\sum_{i \in S_t} p_{ti} \exp\{h_{ti}(p_{ti})\}}{1 + \sum_{j \in S_t} \exp\{h_{tj}(p_{tj})\}}.
\label{eqn_rev_ub}
\end{align}

As we verify in our proofs, this estimate satisfies $\wt{R}_t(S, \vect{p}) \geq R_t(S, \vect{p})$ for any $S \in \mathcal{S}_K$ and any $\vect{p} \in \bbR^N$. Using $\wt{R}_t$ as a proxy for $R_t$, we choose the assortments and prices according to 
\begin{align}
(S_t, \vect{p}_t) \in \argmax_{\substack{S \in \mathcal{S}_K \\ \vect{p} \in \bbR_{+}^{n}}} \wt{R}_t(S, \vect{p}).
\label{estimated_action}
\end{align}

As discussed in Section~\ref{sect_optimal_assortment}, we can efficiently solve this optimization problem using Algorithm~\ref{alg:optimization_algo} since $h_{ti}(p)$ are differentiable and strictly decreasing.

\subsection{Regret Analysis}

Our main result presented in Theorem \ref{thm_regret_ub} concerns the regret upper bound for Algorithm \ref{alg:seq_assortment}. We show this result under the following regularity assumption on the context process which is a standard assumption made in the generalized linear bandit \citep{Li_Lu_Zhou_2017} and MNL contextual bandit \citep{Chen_Wang_Zhou_2020, Oh_Iyengar_2021} literature.

\begin{assumption}[Stochastic and bounded features]
    Each feature vector $\vect{x}_{ti}$ is an independent random variable with unknown distribution;  they satisfy $\|\vect{x}_{ti}\| \leq 1$, and there exists a constant $\sigma_0 > 0$ such that $\bbE[\vect{x}_{ti} \vect{x}_{ti}^\top] \succcurlyeq \sigma_0 \vect{I}$. Furthermore, parameter vectors satisfy $\|(\vect{\psi}^*, \vect{\phi}^*)\| \leq 1$. 
\label{assumption_stochastic_contexts}
\end{assumption}

Accordingly, we can demonstrate in Theorem \ref{thm_regret_ub} that Algorithm~\ref{alg:seq_assortment} enjoys $\wt{\mathcal{O}}(d \sqrt{K T} / L_0)$ regret bound in terms of key problem primitives $N$, $K$, $d$, $L_0$, and $T$. This regret rate is independent of the number of items $N$, and is thus applicable in settings with a large number of candidate items.

\begin{theorem}
    Suppose Assumptions~\ref{assumption_mnl}, \ref{assumption_positive_sens}, and \ref{assumption_stochastic_contexts} hold and we run CAP (Algorithm \ref{alg:seq_assortment}) with initialization length $T_0$ given in \eqref{eqn_T0} and confidence width $\alpha_t$ given in \eqref{eqn_alpha_t}. Then, the expected regret for a sufficiently large time horizon $T$ is upper-bounded as
    \begin{equation*}
        \mathcal{R}_T \leq C_1 \cdot \frac{\log K}{L_0} \, d  \sqrt{K \, T \log T \log (T/d) }
    \end{equation*}
    for a constant $C_1$ independent of $N$, $K$, $d$, $L_0$, and $T$.
    \label{thm_regret_ub}
\if\ARXIV0
    \vspace{-15pt}
\fi
\end{theorem}

\begin{proof}
(Sketch) In proving our regret bounds, we first show that the optimum prices $p_{ti}^*$ are bounded within $[0, P]$ for some $P = \mathcal{O}(\log K / L_0)$ under our utility estimations $h_{ti}(p)$. Then, we show that $T_0 = \Theta(\sigma_0^{-3}d P^2 K \log^2 T)$ initialization steps are enough to ensure $\|\wh{\vect{\theta}}_{t} - \vect{\theta}^*\|_2 = \mathcal{O}(1/P)$ for all $t \geq T_0$. This result enables us to estimate the Fisher Information Matrix around $\vect{\theta}^*$ within a constant factor using $\vect{V}_{t}$. Next, we establish a confidence region $\{ \vect{\theta} : \|\vect{\theta} - \wh{\vect{\theta}}_t\|_{\vect{V}_{t}} \leq \alpha_t\}$ for $\vect{\theta}^*$ with $\alpha_t = \mathcal{O}(\sigma_0^{-1} d \log T)$, importantly noting that $\alpha_t$ is independent of both $K$ and $L_0$. Here, we use a novel Bernstein-type inequality for self-normalized vector-valued martingales which allows us to fully capture the correlation structure between our observations with the help of $q_{ti}(\vect{\wh{\theta}}_t)$. Based on these confidence regions, we construct optimistic utility estimate functions $h_{ti}(p)$ as described in Section~\ref{sect:algorithm}. The selection of assortment and prices according to $h_{ti}(p)$ allows us to obtain an upper bound for the regret incurred at each time step, and hence an upper bound for $\mathcal{R}_T$. Please see Appendix \ref{proof_regret_ub} for details.
\end{proof}

\begin{remark}
    Our analysis in this work assumes that the $L_0$ parameter, or a lower bound of it, is known to the algorithm. However, as we describe in Appendix \ref{sect:estimate_L0}, it is possible to relax this assumption and estimate $L_0$ to achieve the same regret rates.
\end{remark}

\subsection{Extension to Online Parameter Update}

Algorithm~\ref{alg:seq_assortment} is simple to implement and enjoys provable regret bounds as shown in Theorem \ref{thm_regret_ub}. However, the computation of the MLE at each round requires access to all feature vectors corresponding to previous assortments. To reduce the time and space complexities of our algorithm, we can instead use a variant of the online Newton step rule from \cite{Hazan_Koren_Levy_2014}. The online version presented as Algorithm~\ref{alg:seq_assortment_online} in Appendix \ref{proof_regret_ub_online} finds an approximate MLE solution only using the context vectors of the last assortment. We show that the modified algorithm enjoys the following regret rate.

\begin{theorem}
    Suppose Assumptions~\ref{assumption_mnl}, \ref{assumption_positive_sens}, and \ref{assumption_stochastic_contexts} hold and we run CAP-ONS (Algorithm~\ref{alg:seq_assortment_online}) with initialization length $T_0$ given in \eqref{eqn_T0_online} and confidence width $\alpha_t$ given in \eqref{eqn_alpha_online}. Then, the expected regret for a sufficiently large time horizon $T$ satisfies
    \begin{equation*}
        \mathcal{R}_T = \wt{\mathcal{O}} (d  K \sqrt{T} / L_0).
    \end{equation*}
    \label{thm_regret_ub_online}
\vspace{-10pt}
\end{theorem}

\vspace{-10pt}
\subsection{Regret Lower Bound}

In this section, we provide a regret lower bound of order $\Omega(d\sqrt{T}/L_0)$ in terms of key problem primitives $N$, $d$, and $T$ for the problem of assortment selection and pricing under the contextual MNL choice model. This result demonstrates that CAP (Algorithm \ref{alg:seq_assortment}) and CAP-ONS (Algorithm \ref{alg:seq_assortment_online}) are optimal, up to logarithmic terms in $d$, $T$, and $L_0$. 

\begin{theorem}
    For any maximum assortment size $K$, any minimum price sensitivity $L_0 > 0$, any context dimension $d$ divisible by $4$, and any policy $\vect{\pi}$, there exists a worst-case problem instance with $n = \Theta (K \cdot 2^d)$ items, bounded context vectors (i.e., $\|\vect{x}_{ti}\| \leq 1$ for all $i \in [n]$), and bounded feature vectors (i.e., $\|(\vect{\theta}^*; \vect{\phi}^*)\| \leq 1$) such that the regret of policy $\vect{\pi}$ is lower bounded as
    \begin{align*}
        \mathcal{R}_T(\vect{\pi}) \geq C_3 \cdot d\sqrt{T} / L_0
    \end{align*}
    for some universal constant $C_3 > 0$.
    \label{thm_regret_lb}
\end{theorem}
\vspace{-10pt}
\begin{proof}
    (Sketch) We reduce the task of lower bounding the worst-case regret to lower bounding the Bayes risk over an adversarial parameter set. Then, we use a counting argument similar to the one used in \citet{Chen_Wang_Zhou_2020} to provide an explicit lower bound on the Bayes risk. See Appendix \ref{lb_proofs} for details.
\end{proof}
\section{NUMERICAL EXPERIMENTS}\label{sect:numerical_exp}
\vspace{-5pt}

\begin{figure*}[t]
\centering
\includegraphics[width= 0.99 \linewidth]{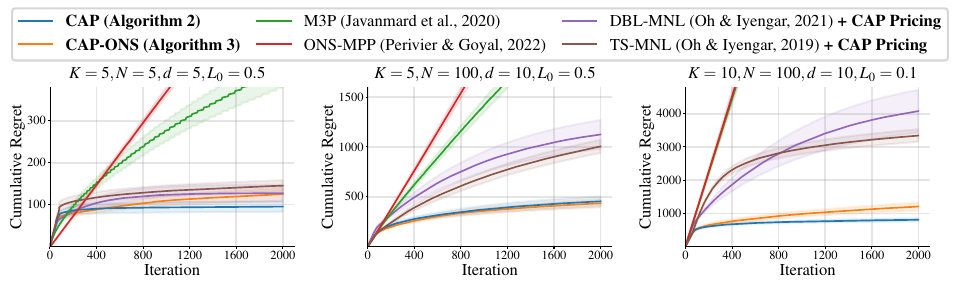}
\caption{Cumulative regret for CAP (Algorithm~\ref{alg:seq_assortment}), CAP-ONS (Algorithm~\ref{alg:seq_assortment_online}), M3P \citep{Javanmard_Nazerzadeh_Shao_2020}, ONS-MPP \citep{Perivier_Goyal_2022}, a version of DBL-MNL \citep{Oh_Iyengar_2021} extended with our dynamic pricing, and a version of TS-MNL \citep{oh2019thompson} extended with our dynamic pricing. The center lines show the mean across the runs while the shaded regions indicate two standard deviations. Results demonstrate the efficacy of our algorithms in achieving diminishing regret per round as our theoretical results predict. Since M3P and ONS-MPP consider only dynamic pricing, they are not able to achieve diminishing regret. DBL-MNL and TS-MNL are designed solely for assortment selection, but their extensions using our pricing approach enable simultaneous assortment selection and pricing. However, even with dynamic pricing, their regret rates quickly deteriorate as $K$ increases or $L_0$ decreases.}
\label{fig:experimental_results}
\if\ARXIV0
    \vspace{-5pt}
\fi
\end{figure*}

We demonstrate the efficacy of our proposed algorithms: CAP presented in Algorithm~\ref{alg:seq_assortment} and its online version CAP-ONS in Algorithm~\ref{alg:seq_assortment_online}. We numerically evaluate our algorithms over independently generated problem instances and provide our results in Figure \ref{fig:experimental_results}. In each instance, we generate problem parameters $(\vect{\psi}^*; \vect{\phi}^*)$ and context vectors $\vect{x}_{ti}$ by sampling their entries from uniform distributions such that we satisfy Assumptions~\ref{assumption_positive_sens} and \ref{assumption_stochastic_contexts}. See Appendix \ref{appendix_exp_details} for further details. The code for our experiments is available at \url{https://github.com/basics-lab/assortment_selection_pricing}.

We compare the performance of our proposed algorithms against state-of-the-art algorithms designed for the MNL choice model. Since the literature primarily focuses on either assortment selection or pricing, our baselines concentrate solely on either assortment selection or pricing. Figure \ref{fig:experimental_results} illustrates that our algorithms, which simultaneously address both assortment selection and pricing, outperform baseline methods.

Our baselines include two MNL pricing algorithms: M3P \citep{Javanmard_Nazerzadeh_Shao_2020} and ONS-MPP \citep{Perivier_Goyal_2022}. These algorithms are designed to optimize prices under the assumption that all $N$ items can be offered without any need for assortment selection. To account for the requirements of our experimental setting, we consider that only top $K$ items (based on their estimated utility value) are offered with chosen prices. These pricing-only algorithms perform comparably when (since there is no assortment decision to be made) but their performance deteriorates as $N \gg K$.

We also consider two MNL assortment selection algorithms as baselines: DBL-MNL \citep{Oh_Iyengar_2021} and TS-MNL \citep{oh2019thompson}. Since both of these algorithms are designed specifically for assortment selection under fixed prices, they cannot achieve diminishing regret in their original form. Therefore, we use our pricing approach to implement heuristic extensions of these algorithms applicable for the joint assortment selection and pricing setting. These extensions utilize the respective frameworks to derive linear estimates for the utility functions and determine the optimal assortments and prices using our Algorithm~\ref{alg:optimization_algo}. In our empirical studies, these heuristic extensions are able to achieve diminishing regret, but the regret gaps between these algorithms and CAP increase as $K$ increases or $L$ decreases.
\section{CONCLUSION}
We study the joint problem of contextual assortment selection and pricing in which a seller aims to maximize cumulative revenue over a horizon. The user’s choice behavior follows a multinomial logit model with unknown parameters, and the seller learns from sequential user feedback. We propose an algorithm that achieves \smash{$\wt{O}(d \sqrt{K T}/L_0)$} regret in $T$ rounds where $d$ is the dimension of the context vectors, $K$ is the assortment size, and $L_0$ is the minimum price sensitivity. We show that this regret rate is optimal up to logarithmic terms in $d$, $T$, $N$, and $L_0$.

\subsection*{Acknowledgements}
This work was supported in part by NSF grants CIF-2007669 and CIF-1937357.

\newpage

\bibliography{ref}

\begin{thebibliography}{}

\bibitem[Abbasi-Yadkori, 2011]{abbasi_2011}
Abbasi-Yadkori, Y. (2011).
\newblock Improved algorithms for linear stochastic bandits.
\newblock In {\em Proceedings of Twenty-Fifth Conference on Neural Information Processing Systems (NeurIPS)}, pages 2312--2320.

\bibitem[Agrawal et~al., 2020]{agrawal2020tractable}
Agrawal, P., Tulabandhula, T., and Avadhanula, V. (2020).
\newblock A tractable online learning algorithm for the multinomial logit contextual bandit.
\newblock {\em arXiv e-prints}, pages arXiv--2011.

\bibitem[Agrawal et~al., 2017]{Agrawal_Avadhanula_Goyal_Zeevi_2017}
Agrawal, S., Avadhanula, V., Goyal, V., and Zeevi, A. (2017).
\newblock {Thompson} sampling for the {MNL}-bandit.
\newblock In {\em Proceedings of the 2017 Conference on learning theory (COLT)}, page 76–78.

\bibitem[Agrawal et~al., 2019]{Agrawal_Avadhanula_Goyal_Zeevi_2019}
Agrawal, S., Avadhanula, V., Goyal, V., and Zeevi, A. (2019).
\newblock {MNL}-bandit: A dynamic learning approach to assortment selection.
\newblock {\em Operations Research}, 67(5):1453–1485.

\bibitem[Amani and Thrampoulidis, 2021]{Amani_Thrampoulidis_2021}
Amani, S. and Thrampoulidis, C. (2021).
\newblock Ucb-based algorithms for multinomial logistic regression bandits.

\bibitem[Aouad et~al., 2018]{Aouad_Levi_Segev_2018}
Aouad, A., Levi, R., and Segev, D. (2018).
\newblock Greedy-like algorithms for dynamic assortment planning under multinomial logit preferences.
\newblock {\em Operations Research}, 66(5):1321–1345.

\bibitem[Ban and Keskin, 2021]{ban2021personalized}
Ban, G.-Y. and Keskin, N.~B. (2021).
\newblock Personalized dynamic pricing with machine learning: High-dimensional features and heterogeneous elasticity.
\newblock {\em Management Science}, 67(9):5549--5568.

\bibitem[Bartlett et~al., 2005]{Bartlett_Bousquet_Mendelson_2005}
Bartlett, P.~L., Bousquet, O., and Mendelson, S. (2005).
\newblock Local rademacher complexities.
\newblock {\em The Annals of Statistics}, 33.
\newblock arXiv:math/0508275.

\bibitem[Besbes and Zeevi, 2009]{Besbes_Zeevi_2009}
Besbes, O. and Zeevi, A. (2009).
\newblock Dynamic pricing without knowing the demand function: Risk bounds and near-optimal algorithms.
\newblock {\em Operations Research}, 57(6):1407--1420.

\bibitem[Boyd and Vandenberghe, 2004]{Boyd_Vandenberghe_2004}
Boyd, S.~P. and Vandenberghe, L. (2004).
\newblock {\em Convex Optimization}.
\newblock Cambridge University Press.

\bibitem[Bubeck et~al., 2019]{Bubeck_Devanur_Huang_Niazadeh_2019}
Bubeck, S., Devanur, N.~R., Huang, Z., and Niazadeh, R. (2019).
\newblock Multi-scale online learning: Theory and applications to online auctions and pricing.
\newblock {\em Journal of Machine Learning Research}, 20(62):1--37.

\bibitem[Caro and Gallien, 2007]{Caro_Gallien_2007}
Caro, F. and Gallien, J. (2007).
\newblock Dynamic assortment with demand learning for seasonal consumer goods.
\newblock {\em Management Science}, 53(2):276–292.

\bibitem[Chen et~al., 1999]{Chen_Hu_Ying_1999}
Chen, K., Hu, I., and Ying, Z. (1999).
\newblock Strong consistency of maximum quasi-likelihood estimators in generalized linear models with fixed and adaptive designs.
\newblock {\em The Annals of Statistics}, 27(4):1155–1163.

\bibitem[Chen et~al., 2013]{chen_2013}
Chen, W., Wang, Y., and Yuan, Y. (2013).
\newblock Combinatorial multi-armed bandit: General framework and applications.
\newblock In {\em Proceedings of the Thirtieth International Conference on Machine Learning (ICML)}, volume~28, pages 151--159.

\bibitem[Chen and Wang, 2018]{Chen_Wang_2018}
Chen, X. and Wang, Y. (2018).
\newblock A note on a tight lower bound for capacitated {MNL}-bandit assortment selection models.
\newblock {\em Operations Research Letters}, 46(5):534–537.

\bibitem[Chen et~al., 2020]{Chen_Wang_Zhou_2020}
Chen, X., Wang, Y., and Zhou, Y. (2020).
\newblock Dynamic assortment optimization with changing contextual information.
\newblock {\em The Journal of Machine Learning Research}, 21(1):216:8918--216:8961.

\bibitem[Cheung and Simchi-Levi, 2017]{Cheung_Simchi_2017}
Cheung, W.~C. and Simchi-Levi, D. (2017).
\newblock Thompson sampling for online personalized assortment optimization problems with multinomial logit choice models.
\newblock {\em Available at SSRN}.

\bibitem[Chu et~al., 2011]{Chu_Li_Reyzin_Schapire_2011}
Chu, W., Li, L., Reyzin, L., and Schapire, R. (2011).
\newblock Contextual bandits with linear payoff functions.
\newblock In {\em Proceedings of the Fourteenth International Conference on Artificial Intelligence and Statistics (AISTATS)}, page 208–214.

\bibitem[Faury et~al., 2020]{Faury_Abeille_Calauzenes_Fercoq_2020}
Faury, L., Abeille, M., Calauzènes, C., and Fercoq, O. (2020).
\newblock Improved optimistic algorithms for logistic bandits.
\newblock In {\em Proceedings of the Thirty-Seventh International Conference on Machine Learning (ICML)}.

\bibitem[Freedman, 1975]{Freedman_1975}
Freedman, D.~A. (1975).
\newblock On tail probabilities for martingales.
\newblock {\em The Annals of Probability}, 3:100–118.

\bibitem[Hazan et~al., 2014]{Hazan_Koren_Levy_2014}
Hazan, E., Koren, T., and Levy, K.~Y. (2014).
\newblock Logistic regression: Tight bounds for stochastic and online optimization.
\newblock In {\em Proceedings of The Twenty-Seventh Conference on Learning Theory (COLT)}, page 197–209. PMLR.

\bibitem[Javanmard et~al., 2020]{Javanmard_Nazerzadeh_Shao_2020}
Javanmard, A., Nazerzadeh, H., and Shao, S. (2020).
\newblock Multi-product dynamic pricing in high-dimensions with heterogeneous price sensitivity.
\newblock In {\em Proceedings of The 2020 IEEE International Symposium on Information Theory (ISIT)}, page 2652–2657.

\bibitem[Kleinberg and Leighton, 2003]{kleinberg2003}
Kleinberg, R. and Leighton, T. (2003).
\newblock The value of knowing a demand curve: bounds on regret for online posted-price auctions.
\newblock In {\em Proceedings of the Forty-Fourth Annual IEEE Symposium on Foundations of Computer Science (FOCS)}, pages 594--605.

\bibitem[Kohler and Lucchi, 2017]{Kohler_Lucchi_2017}
Kohler, J.~M. and Lucchi, A. (2017).
\newblock Sub-sampled cubic regularization for non-convex optimization.
\newblock In {\em Proceedings of the Thirty-Fourth International Conference on Machine Learning (ICML)}.

\bibitem[Kveton et~al., 2015]{kveton_2015}
Kveton, B., Wen, Z., Ashkan, A., and Szepesvari, C. (2015).
\newblock {Tight Regret Bounds for Stochastic Combinatorial Semi-Bandits}.
\newblock In {\em Proceedings of the Eighteenth International Conference on Artificial Intelligence and Statistics (AISTATS)}, volume~38, pages 535--543.

\bibitem[Lattimore and Szepesvári, 2020]{Lattimore_Szepesvari_2020}
Lattimore, T. and Szepesvári, C. (2020).
\newblock {\em Bandit Algorithms}.
\newblock Cambridge University Press, 1 edition.

\bibitem[Li et~al., 2017]{Li_Lu_Zhou_2017}
Li, L., Lu, Y., and Zhou, D. (2017).
\newblock Provably optimal algorithms for generalized linear contextual bandits.
\newblock In {\em Proceedings of the 34th International Conference on Machine Learning}, page 2071–2080. PMLR.

\bibitem[McFadden, 1978]{McFadden_1978}
McFadden, D. (1978).
\newblock Modeling the choice of residential location.
\newblock {\em Transportation Research Record}.

\bibitem[Miao and Chao, 2018]{Miao_Chao_2018}
Miao, S. and Chao, X. (2018).
\newblock Dynamic joint assortment and pricing optimization with demand learning.
\newblock {\em Manufacturing \& Service Operations Management}, 23(2):525--545.

\bibitem[Oh and Iyengar, 2021]{Oh_Iyengar_2021}
Oh, M. and Iyengar, G. (2021).
\newblock Multinomial logit contextual bandits: Provable optimality and practicality.
\newblock In {\em Proceedings of the AAAI Conference on Artificial Intelligence}, pages 9205--9213.

\bibitem[Oh and Iyengar, 2019]{oh2019thompson}
Oh, M.-h. and Iyengar, G. (2019).
\newblock Thompson sampling for multinomial logit contextual bandits.
\newblock {\em Advances in Neural Information Processing Systems}, 32.

\bibitem[Paes~Leme and Schneider, 2018]{Paes_Leme_Schneider_2018}
Paes~Leme, R. and Schneider, J. (2018).
\newblock Contextual search via intrinsic volumes.
\newblock In {\em Proceedings of the Fifty-Ninth Annual IEEE Symposium on Foundations of Computer Science (FOCS)}, pages 268--282.

\bibitem[Perivier and Goyal, 2022]{Perivier_Goyal_2022}
Perivier, N. and Goyal, V. (2022).
\newblock Dynamic pricing and assortment under a contextual mnl demand.
\newblock In {\em Proceedings of Thirty-Fifth Conference on Neural Information Processing Systems (NeurIPS)}.

\bibitem[Qin et~al., 2014]{Qin_Chen_Zhu_2014}
Qin, L., Chen, S., and Zhu, X. (2014).
\newblock Contextual combinatorial bandit and its application on diversified online recommendation.
\newblock In {\em Proceedings of the 2014 SIAM International Conference on Data Mining (SDM)}, Proceedings, page 461–469.

\bibitem[Wang et~al., 2025]{wang2025dynamic}
Wang, H., Talluri, K., and Li, X. (2025).
\newblock On dynamic pricing with covariates.
\newblock {\em Operations Research}.

\bibitem[Wang, 2013]{Wang_2013}
Wang, R. (2013).
\newblock Capacitated assortment and price optimization under the multinomial logit model.
\newblock {\em Operations Research Letters}.

\bibitem[Xu and Wang, 2021]{Xu_Wang_2021}
Xu, J. and Wang, Y.-X. (2021).
\newblock Logarithmic regret in feature-based dynamic pricing.
\newblock In {\em Proceedings of Thirty-Fifth Conference on Neural Information Processing Systems (NeurIPS)}, pages 13898--13910.

\bibitem[Xu and Wang, 2024]{xu2024pricing}
Xu, J. and Wang, Y.-X. (2024).
\newblock Pricing with contextual elasticity and heteroscedastic valuation.
\newblock In {\em International Conference on Machine Learning}, pages 55286--55304. PMLR.

\bibitem[Zhang et~al., 2016]{Zhang_Yang_Jin_Xiao_Zhou_2016}
Zhang, L., Yang, T., Jin, R., Xiao, Y., and Zhou, Z.-h. (2016).
\newblock Online stochastic linear optimization under one-bit feedback.
\newblock In {\em Proceedings of the Thirty-Third International Conference on Machine Learning (ICML)}, pages 392--401.

\bibitem[Zong et~al., 2016]{Zong_Ni_Sung_Ke_Wen_Kveton_2016}
Zong, S., Ni, H., Sung, K., Ke, N.~R., Wen, Z., and Kveton, B. (2016).
\newblock Cascading bandits for large-scale recommendation problems.
\newblock In {\em Proceedings of the Thirty-Second Conference on Uncertainty in Artificial Intelligence (UAI)}, page 835–844.

\end{thebibliography}
\bibliographystyle{apalike}

\newpage
\appendix
\onecolumn
\section{Properties of Maximum Likelihood Estimation}\label{sect:mle}

\begin{proposition}
The maximum likelihood estimator is any parameter $\widehat{\vect{\theta}}_t$ that minimizes the negative log-likelihood function over the parameter space, that is
\begin{align}
\widehat{\vect{\theta}}_t \in \argmin_{\vect{\theta}} \ell_t (\vect{\theta}).
\label{eqn_mle}
\end{align}
The negative log-likelihood function $\ell_t (\vect{\theta})$ is convex over $\vect{\theta} \in \bbR^{2d}$. Furthermore, if the Fisher information matrix $\mathcal{I}_t(\vect{\theta}) = \nabla^2_{\vect{\theta}} \ell_t (\vect{\theta})$ is positive definite, then $\ell_t (\vect{\theta})$ is strongly convex and thus admits a unique minimizer.
\label{prop_mle}
\end{proposition}

For each item $i \in S_t \cup \{0\}$, we define the choice response variables $y_{ti} = \1 \{ i_t = i \} \in \{0,1\}$. Then, the gradient of these probabilities with respect to $\vect{\theta}$ can be written as
\begin{align*}
    \nabla_{\vect{\theta}} q_{t i}(\vect{\theta}) = q_{t i}(\vect{\theta}) \left( \wt{\vect{x}}_{ti} - \sum_{j \in S_t} q_{tj}(\vect{\theta}) \wt{\vect{x}}_{tj} \right).
\end{align*}

On the other hand, we can write the negative log-likelihood function at time $t$ as
\begin{align*}
    \ell_t (\vect{\theta}) := -  \sum_{\tau = 1}^{t-1} \sum_{i \in S_\tau \cup \{0\}} y_{ti} \log q_{\tau i}( \vect{\theta} ).
\end{align*}

Calculating the gradient of this negative log-likelihood with respect to $\vect{\theta}$ we obtain
\begin{align*}
    \nabla_{\vect{\theta}} \ell_t (\vect{\theta}) &= \sum_{\tau = 1}^{t-1} \sum_{i \in S_\tau} (q_{\tau i}( \vect{\theta} ) - y_{\tau i}) \wt{\vect{x}}_{\tau i}
\end{align*}

On the other hand, the Hessian of the negative log-likelihood is given by
\begin{align*}
    \nabla^2_{\vect{\theta}} \ell_t (\vect{\theta}) &= \sum_{\tau = 1}^{t-1} \sum_{i \in S_\tau} q_{\tau i}(\vect{\theta}) \wt{\vect{x}}_{\tau i}  \left( \wt{\vect{x}}_{\tau i} - \sum_{j \in S_t} q_{\tau j}(\vect{\theta}) \wt{\vect{x}}_{\tau j} \right)^\top \\
    &= \sum_{\tau = 1}^{t-1} \left(\sum_{i \in S_\tau} q_{t i}(\vect{\theta}) \wt{\vect{x}}_{\tau i} \wt{\vect{x}}_{\tau i}^\top - \sum_{i \in S_\tau} \sum_{j \in S_t} q_{t i}(\vect{\theta}) q_{\tau j}(\vect{\theta}) \wt{\vect{x}}_{\tau i} \wt{\vect{x}}_{\tau j}^\top \right).
\end{align*}

Since this log-likelihood satisfies the necessary regularity conditions, the Hessian of the negative log-likelihood is also equal to the Fisher information matrix $\mathcal{I}_t (\vect{\theta}) = \nabla^2_{\vect{\theta}} \ell_t (\vect{\theta})$.

Now, let $\vect{q}_t (\vect{\theta})$ denote the vector of probabilities $q_{t i}(\vect{\theta})$ and let $\wt{\vect{X}}_t$ be the matrix with columns $\wt{\vect{x}}_{ti}$ for $i \in S_t$, we can write
\begin{align*}
    \nabla^2_{\vect{\theta}} \ell_t (\vect{\theta}) &= \sum_{\tau = 1}^{t-1} \wt{\vect{X}}_{\tau} \Sigma_\tau(\vect{\theta}) \wt{\vect{X}}_{\tau}^\top.
\end{align*}
where $\Sigma_t(\vect{\theta}) =\mathrm{diag}(\vect{q}_t (\vect{\theta})) - \vect{q}_t (\vect{\theta}) \vect{q}_t (\vect{\theta})^\top$. Since we have $q_{t i}(\vect{\theta}) q_{t 0}(\vect{\theta}) > 0$ for all $\vect{\theta} \in \bbR^{2d}$, we conclude that $\Sigma_t (\vect{\theta}) \succ \vect{0}$ for all $\vect{\theta} \in \bbR^{2d}$. 

Therefore, $\nabla^2_{\vect{\theta}} \ell_t (\vect{\theta}) \succcurlyeq \vect{0}$ for all $\vect{\theta} \in \bbR^{2d}$. Hence, the negative log-likelihood is convex with respect to $\vect{\theta}$. As a result, any $\vect{\theta}$ that satisfies the first-order optimality condition $\nabla_{\vect{\theta}} \ell_t (\vect{\theta}) = 0$ is a minimizer. 

Furthermore, if we are given that the Fisher Information Matrix $\nabla^2_{\vect{\theta}} \ell_t (\vect{\theta})$ is positive definite, i.e. $\nabla^2_{\vect{\theta}} \ell_t (\vect{\theta}) \succ \vect{0}$, the negative log-likelihood function becomes strongly convex with respect to $\vect{\theta}$. Consequently, we have a unique MLE solution.

\section{Solving the Assortment and Price Optimization Problem}\label{sect:appendix_solving_assortment_price_optimization}

As stated in Proposition~\ref{prop_opt_assortments_and_prices}, we make the following regularity assumption for the utility functions.

\begin{assumption}
    For each item $i \in [N]$, the utility function $u_{i}(p_i)$ is differentiable, strictly decreasing in price $p_i$, and satisfies $\lim_{p \to \infty} p e^{u_i(p)} = 0$.
    \label{assumption_utility_reg}
\end{assumption}

We first consider the price optimization for a given assortment $S$. That is,
\begin{align*}
    \max_{\vect{p} \in \bbR^N} \sum_{i \in S} p_{i} \cdot q_{i}(\vect{p})
\end{align*}
where $q_{i}(\vect{p})$ denotes the probability of choosing item $i$ under prices $\vect{p} \in \bbR^N$.

\begin{proposition}
    Fix some assortment $S \subseteq [N]$. Under Assumption \ref{assumption_utility_reg}, the quantity $p_i^* + 1 / u_i'(p_i^*)$ is constant for all $i \in S$ at the optimal price vector $\vect{p}^*$. Moreover, $p_i^* + 1 / u_i'(p_i^*)$ is equal to the total revenue obtained by pricing at $\vect{p}^*$.
\label{prop_price_theta}
\end{proposition}

\begin{proof}
    The first-order condition for optimality is
    \begin{align*}
        \nabla_{\vect{p}} \left\{\sum_{i \in S} p_{i} \cdot q_{i}(\vect{p}) \right\} = \vect{0}.
    \end{align*}

For any $i \in S$, it is straightforward to verify that
\begin{align*}
    \frac{\partial q_i(\vect{p})}{\partial p_{i}} &= u_i'(p_i) \cdot q_i(\vect{p}) ( 1- q_i(\vect{p})) < 0,\\
    \frac{\partial q_j(\vect{p})}{\partial p_{i}} &= - u_i'(p_i) \cdot q_j(\vect{p}) q_i(\vect{p}) > 0, \forall j \in S, j \neq i.
\end{align*}
 
Therefore, for each $i \in S$, we need
\begin{align*}
    q_{i}(\vect{p}) + p_i u_i'(p_i) q_i(\vect{p}) ( 1- q_i(\vect{p})) - \sum_{ j \in S, j \neq i} p_j u_i'(p_i) q_i(\vect{p}) q_j(\vect{p}) = 0.
\end{align*}

Rearranging the above equation results in
\begin{align*}
    q_{i}(\vect{p}) u_i'(p_i) \left [ p_i + \frac{1}{u_i'(p_i)}  - \sum_{ j \in S} p_j q_j(\vect{p}) \right] = 0.
\end{align*}

Since $q_{i}(\vect{p}) > 0$ and $u_i'(p_i) < 0$ for all $i \in S$ and for all $p_i \in \bbR$, the above equation is equivalent to
\begin{align}
    p_i + \frac{1}{u_i'(p_i)} = \sum_{ i \in S} p_i q_i(\vect{p}). 
    \label{eq_kkt_price}
\end{align}

The right-hand side of this equation is independent of the item index $i$, so $p_i^* + 1 / u_i'(p_i^*)$ is constant for all $i \in S$ at an optimal price vector $\vect{p}^*$. Moreover, this equality shows that the $p_i^* + 1 / u_i'(p_i^*)$ quantity is equal to the total revenue at the optimal price vector $\vect{p}^*$.

\end{proof}

Let us denote this constant quantity as $\theta = p_i + 1 / u_i'(p_i)$. 

\begin{remark}
Under the concavity and twice-differentiability assumptions made in \cite{Wang_2013}, it is possible to show that there exists a unique $p_i$ value that satisfies $\theta = p_i + 1 / u_i'(p_i)$ for all different values of $\theta$. In other words, it is possible to show that, $\theta = p_i + 1 / u_i'(p_i)$ is a one-to-one relation. However, it is not possible to show this property only under Assumption \ref{assumption_utility_reg}.
\end{remark}

For some values of $\theta$, there might exist multiple $p_i$ values that satisfy $\theta = p_i + 1 / u_i'(p_i)$. Let $\mathcal{P}_i(\theta)$ denote this set of prices that satisfy $\theta = p_i + 1 / u_i'(p_i)$.

 With this notation, the necessary condition for optimality in \eqref{eq_kkt_price} can be equivalently written as
\begin{align*}
   \theta &= \sum_{ i \in S} \left( \theta - \frac{1}{u_i'(p_i)} \right) q_i(\vect{p}) \quad \text{and}\\
   p_i &\in \mathcal{P}_i(\theta), \quad \forall i \in S.
\end{align*}

Noting that $1 - \sum_{ i \in S} q_i(\vect{p}) = q_0(\vect{p}) $ and $q_i(\vect{p}) / q_0(\vect{p}) = e^{u_i(p_i)}$, we have the equivalent conditions
\begin{align*}
   \theta &= - \sum_{ i \in S} \frac{e^{u_i(p_i)}}{u_i'(p_i)} \quad \text{and}\\
   p_i &\in \mathcal{P}_i(\theta), \quad \forall i \in S.
\end{align*}

Under these conditions, Proposition \ref{prop_price_theta} states that the objective function (revenue) is equal to the value of $\theta$. Therefore, we can cast the following problem to find the optimum values for $\theta$ and $\vect{p}$.
\begin{align}
\begin{split}
    \theta^* = \max_{\substack{\theta \in \bbR\\ p_i \in \bbR, \forall i \in S}} & \quad \theta\\
    \text{s.t.} \quad & \quad \theta = - \sum_{ i \in S} \frac{e^{u_i(p_i)}}{u_i'(p_i)},\\
   &\quad p_i \in \mathcal{P}_i(\theta), \quad \forall i \in S.
\end{split}
\label{problem_theta_star}
\end{align}

Now, we define the following functions
\begin{align*}
    g_i(p_i) &= - \frac{e^{u_i(p_i)}}{u_i'(p_i)}\\
    f_i(\theta) &= \max \left\{ g_i(p_i) : p_i \in \mathcal{P}_i(\theta) \right\}
\end{align*}

In Lemma \ref{lemma_f_decreasing}, we show that $f_i(\theta) > 0$ is a continuous and strictly decreasing function of $\theta$. As a result, we have the following proposition.

\begin{proposition}
Under Assumption \ref{assumption_utility_reg}, the optimum objective value $\theta^*$ of problem \eqref{problem_theta_star} satisfies 
\begin{align*}
    \sum_{i \in S} f_i(\theta^*) = \theta^*.
\end{align*}
\label{optimum_theta_star_condition}
\end{proposition}

\begin{proof}
    Assume $\sum_{i \in S} f_i(\theta) < \theta$ for some $\theta$. Then, for any $\theta' \geq \theta$ and for any $\vect{p}$ that satisfies $p_i \in \mathcal{P}_i(\theta')$ for all $i \in S$, we have $\theta' \geq \theta > \sum_{i \in S} f_i(\theta) \geq \sum_{i \in S} f_i(\theta') \geq \sum_{i \in S} g_i(p_i)$. Therefore, all $\theta' \geq \theta$ is infeasible and hence $\theta^* < \theta$.

    Now, assume $\sum_{i \in S} f_i(\theta) > \theta$ for some $\theta$. Then, since each $f_i(\theta) > 0$ is continuous and strictly decreasing, there exists some $\theta' > \theta$ such that $\sum_{i \in S} f_i(\theta') = \theta'$. Since this $\theta'$ is a feasible solution to the problem together with prices $p_i' = \argmax_{p_i} \left\{ g_i(p_i) : p_i \in \mathcal{P}_i(\theta) \right\}$, we have $\theta^* \geq \theta' > \theta$.

    Therefore, any optimum point $\theta = \theta^*$ must satisfy $\sum_{i \in S} f_i(\theta) = \theta$.
\end{proof}

Using Proposition \ref{optimum_theta_star_condition}, we can reduce the multi-product price optimization problem for any given assortment $S$ to a single-dimensional problem given by
\begin{align*}
    \max_{\theta \in \bbR} & \left\{ \theta : \theta = \sum_{i \in S} f_i(\theta) \right\}.
\end{align*}

Furthermore, since $f_i(\theta) > 0$ is strictly decreasing in $\theta$, there exists a unique solution to the condition $\sum_{i \in S} f_i(\theta) = \theta$ for any given assortment $S$. Let us denote this unique feasible (and hence optimal) solution by $\theta^*_S$.

The next step is to incorporate the assortment selection into our optimization problem. We can achieve this by considering the problem
\begin{align}
    \max \left\{ \theta^*_S : S \in \mathcal{S}_K \right\} = \max_{\theta \in \bbR} & \left\{ \theta : S \in \mathcal{S}_K \text{ and } \theta = \sum_{i \in S} f_i(\theta) \right\}.
\label{eq_assortment_bad}
\end{align}

This assortment selection problem in the given form requires searching all possible assortments of size at most $K$ and there are $\sum_{\ell = 1}^K \binom{N}{\ell} = \sum_{\ell = 1}^K N!/((N - \ell)! \, \ell !)$ assortments to consider. However, this search space can be significantly reduced by noticing that finding the unique fixed point of the equation
\begin{align}
    \theta = \max_{S \in \mathcal{S}_K} \sum_{i \in S} f_i(\theta).
\end{align}
is equivalent to solving \eqref{eq_assortment_bad}. Since each $f_i(\theta)$ is strictly decreasing in $\theta$ and the right-hand side is strictly increasing in $\theta$, there exists a unique solution $\theta^*$ to this equation.

Denote an assortment $S^* \in \mathcal{S}_K$ such that $\theta^* = \sum_{i \in S^*} f_i(\theta^*)$. Then, $S^*$ is an optimal assortment together with prices $p_i^* = \argmax_{p_i} \left\{ g_i(p_i) : p_i \in \mathcal{P}_i(\theta^*) \right\}$ for all $i \in S^*$. 

As a result, we obtain the following proposition given in the main paper.

\optassortment*

\subsection{Assortment Selection and Pricing under Estimated Utility Functions} \label{sect:computation}

In the following section, we describe how we can \textbf{efficiently} run Algorithm \ref{alg:optimization_algo} for the estimated utility functions $h_{ti}(p)$ described in Section \ref{sect:algorithm}. 

In Lemma \ref{lemma_utility_ub}, we show that $h_{ti}(p)$ is differentiable and strictly decreasing with derivatives $h_{ti}'(p) \leq -L_0$. Therefore, $h_{ti}(p)$ satisfies the conditions of Proposition \ref{prop_opt_assortments_and_prices} and we can use Algorithm \ref{sect:algorithm} to find an optimal assortment and pricing.

The main difficulty in running Algorithm \ref{sect:algorithm} is to find the set of points $\mathcal{P}_{ti}(B) = \{p : p + 1 / h_{ti}'(p) = B\}$ for any given $B > 0$. In the following, we will show how we can obtain this set for the specific structure of $h_{ti}(p)$.

In the proof of Lemma \ref{lemma_utility_ub}, we establish that $\wt{h}_{ti}(p)$ is smooth and strictly convex. Then, letting $p_0$ be the unique value such that $\wt{h}_{ti}' (p_0) = - L_0$, we show that
\begin{align*}
    h_{ti}(p) =
    \begin{cases}
        \wt{h}_{ti} (p_0) - L_0 (p - p_0) \quad &\text{if } p \geq p_0,\\
        \wt{h}_{ti} (p) \quad &\text{if } p < p_0.
    \end{cases}
\end{align*}

To find all points that satisfy $p + 1 / h_{ti}'(p) = B$, we search for $p$ values to the left and right of $p_0$ separately.

To find all the points $p \in \mathcal{P}_{ti}(B)$ such that $p \leq p_0$, we will use the structure of $\wt{h}_{ti} (p)$. Since $\wt{h}_{ti} (p)$ is given as a sum of a linear function and a square root of a quadratic function, we can write it as
\begin{align*}
    \wt{h}_{ti} (p) = a_1 - a_2 \cdot p + \sqrt{a_3 - 2 a_4 \cdot p + a_5 \cdot p^2}
\end{align*}
for some $a_1, a_2, a_3, a_4, a_5 \in \bbR$. Since the square root part is a norm, the quadratic inside must have non-positive determinant, i.e. $a_4^2 - a_3 a_5 \leq 0$.

Therefore, for any $p \leq p_0$, we have
\begin{align*}
    h_{ti}'(p) &= \frac{a_5 \cdot p - a_4}{\sqrt{a_3 - 2 a_4 \cdot  p + a_5 \cdot p^2}} - a_2\\
    h_{ti}''(p) &= \frac{a_3 a_5 - a_4^2}{(a_3 - 2 a_4 \cdot p + a_5 \cdot p^2)^{3/2}}
\end{align*}

Our goal is to find the solutions for $p + \frac{1}{h_{ti}'(p)} - B = 0$, or equivalently we want to find the roots of $z(p) := h_{ti}'(p) - \frac{1}{B - p}$. Since $h_{ti}'(p) < 0$, this equation only has root on $p > B$. Since $z(p)$ is continuous on $p > B$, there exists at most one root between each local minima/maxima points of $z(p)$. Furthermore, since $z(p)$ is also differentiable function on $p > B$, we can find all points with $z'(p) = 0$ to identify local minima and maxima. Now, we observe that
\begin{align*}
    z'(p) = h_{ti}''(p) - \frac{1}{(B-p)^2}.
\end{align*}

Hence, at any point satisfying $z'(p) = 0$, we must have 
\begin{align*}
    \frac{a_3 a_5 - a_4^2}{(a_3 - 2 a_4 \cdot p + a_5 \cdot p^2)^{3/2}} = \frac{1}{(B-p)^2}
\end{align*}

Then, raising this equation two the second power, we obtain
\begin{align*}
    \frac{1}{(a_3 a_5 - a_4^2)^2 } (a_3 - 2 a_4 \cdot p + a_5 \cdot p^2)^3 - (B-p)^4 = 0.
\end{align*}

Since the left hand side is a $6^{\text{th}}$ order polynomial in $p$, we can easily find all the roots for this equation. Since any local minima/maxima points of $z(p)$ must be one of these roots, this gives us a necessary condition for local minima/maxima points of $z(p)$. Then, check for each of these points and construct the set $\mathcal{Z}$ that contains local minima/maxima points of $z(p)$.

Then, we can search for a root of $z(p)$ between each pair of consecutive points in $\mathcal{Z}$. Since there is at most one root between every pair of consecutive points, we can find all roots of $z(p)$ efficiently. After finding all the roots, we only add the ones that satisfy $p < p_0$ to the set $\mathcal{P}_{ti}(B)$.

Lastly, we check for any solutions $p + 1/h_{ti}'(p) - B = 0$ over $p > p_0$. Since $h_{ti}(p)$ is a linear function over $p > p_0$, the only possible solution is $p = 1/L_0 + B$. If this solution satisfies $p > p_0$, we also add it to the set $\mathcal{P}_{ti}(B)$.

\subsection{Technical Lemmas for Assortment Selection and Pricing}

\begin{lemma}
Under Assumption \ref{assumption_utility_reg}, $f_i(\theta) > 0$ is a continuous and strictly decreasing function of $\theta \in \bbR$.
\label{lemma_f_decreasing}
\end{lemma}

\begin{proof}

Recall the definitions
\begin{align*}
    g_i(p_i) &= - \frac{e^{u_i(p_i)}}{u_i'(p_i)}\\
    f_i(\theta) &= \max \left\{ g_i(p_i) : p_i \in \mathcal{P}_i(\theta) \right\}
\end{align*}
where $\mathcal{P}_i(\theta) = \{ p : \theta = p + 1 / u_i'(p) \}$. Since $u_i(p)$ is a differentiable function, its derivative $u_i'(p)$ is continuous everywhere. We also have $u_i'(p) < L_0$ since $u_i(p)$ is decreasing. 

Let $z(p) = p + 1 / u_i'(p)$. By continuity of $u_i'(p)$, both $z(p)$ and $g_i(p)$ are continuous functions.

First, we show that any $p$ that is a local minimum for $z(p)$ is a local maximum for $g_i(p)$. Similarly, any $p$ that is a local maximum for $z(p)$ is a local minimum for $g_i(p)$.

Suppose $p$ is a local minimum for $z(p)$. Then, there exists some $\delta > 0$ such that $z(p') \geq z(p)$ for all $|p - p'| \leq \delta$. That is, $p' - p + \frac{1}{u_i'(p')} - \frac{1}{u_i'(p)} \geq 0$.

Now, we use Taylor's expansion for $e^{u_i(p)}$ at $p$ to write
\begin{align*}
    e^{u_i(p')} = e^{u_i(p)} + u_i'(p) e^{u_i(p)} (p' - p) + o(\delta).
\end{align*}
Then, dividing through by $e^{u_i(p)} u_i'(p')$, we obtain
\begin{align*}
    \frac{e^{u_i(p')-u_i(p)}}{u_i'(p')} = \frac{1}{u_i'(p')} + \frac{u_i'(p)}{u_i'(p')} (p' - p) + e^{-u_i(p)} o(\delta).
\end{align*}
Since $u_i'(p) < 0$ for all $p$, we have $\frac{u_i'(p)}{u_i'(p')} > 0$. Then, using $p' - p + \frac{1}{u_i(p')} - \frac{1}{u_i(p)} \geq 0$ for any $|p - p'| \leq \delta$, we can write
\begin{align*}
    \frac{e^{u_i(p')-u_i(p)}}{u_i'(p')} &= \frac{1}{u_i'(p')} + \frac{u_i'(p)}{u_i'(p')} \left( \frac{1}{u_i'(p)} - \frac{1}{u_i'(p')} \right) + e^{-u_i(p)} o(\delta)\\
    &= \frac{2}{u_i'(p')} - \frac{u_i'(p)}{(u_i'(p'))^2} + e^{-u_i(p)} o(\delta).
\end{align*}
It is possible to show that $f(x) = \frac{2}{x} - \frac{a}{x^2}$ has a local minimum at $x = a$ when $a < 0$. Therefore, using continuity of $u_i'(p)$, we can show that there exists $\eta > 0$ such that
\begin{align*}
    \frac{e^{u_i(p')-u_i(p)}}{u_i'(p')} \geq \frac{1}{u_i'(p')}
\end{align*}
This inequality is equivalent to $g_i(p') \leq g_i(p)$. Therefore, $g_i(p)$ has a local maximum at $p$. We can show the symmetric result using similar arguments.

The next step is to show the continuity of $f_i(p)$. Consider any $\theta$ and any $p \in \mathcal{P}_i(\theta)$ that is not a local maximum or minimum for $z(p)$. Then, for any $\eta > 0$, there exists a real $\delta > 0$ such that for any $\theta'$, $0 < | \theta - \theta'| < \delta$ implies that there exists a $p' \in \mathcal{P}_i(\theta')$ such that $|p - p'| \leq \eta$. In other words, $\mathcal{P}_i(\theta)$ is a \emph{continuous} function (that maps a real number to a set of real numbers) unless one of the prices in $\mathcal{P}_i(\theta)$ is a local maximum or minimum for $z(p)$.

As a result, $f_i(\theta)$ is a continuous function over $\theta$ values for which no price in $\mathcal{P}_i(\theta)$ is local maxima or minima for $z(p)$. Next, we show continuity on other $\theta$ values. Consider a $\theta$ value and a price point $p \in \mathcal{P}_i(\theta)$ that is a local maximum or minimum for $z(p)$. 

If $p \in \mathcal{P}_i(\theta)$ is a local minimum for $z(p)$, there exists another $p_1 \in \mathcal{P}_i(\theta)$ such that $p_1 < p$ and $z(p) \geq \theta$ for all $p' \in (p_1, p)$. This is because $z(p)$ is continuous $\lim_{p \to - \infty} z(p) = - \infty$. Then, by Lemma \ref{lemma_two_prices}, we have $g_i(p) \leq g_i(p_1)$. Since $g_i(p) \leq g_i(p_1)$, $\theta$ is not \emph{active} in $f_i$. As a result, the continuity of $f_i(\theta)$ is preserved at $\theta$.

Similarly, if $p \in \mathcal{P}_i(\theta)$ is a local maximum for $z(p)$, there exists another $p_2 \in \mathcal{P}_i(\theta)$ such that $p < p_2$ and $z(p) \leq \theta$ for all $p' \in (p_1, p)$. This is because $z(p)$ is continuous $\lim_{p \to \infty} z(p) = \infty$. 
Then, by Lemma \ref{lemma_two_prices}, we have $g_i(p_1) \leq g_i(p_2)$. As a result, the continuity of $f_i(\theta)$ is preserved at $\theta$.

We have shown that $f_i(\theta)$ is continuous. Next, we show that it is a decreasing function. As we showed in previous parts of this proof, the continuity of $f_i(\theta)$ is not affected at $\theta$ values with some $p \in \mathcal{P}_i(\theta)$ that is a local maximum or minimum for $z(p)$. Therefore, it is sufficient to show that $f_i(\theta)$ is decreasing over $\theta$ values such that no $p \in \mathcal{P}_i(\theta)$ is a local maximum or minimum for $z(p)$.

Let $\theta$ be such a value. Since $g_i(p)$ is strictly increasing on every interval in which $z(p)$ is strictly decreasing and $g_i(p)$ is strictly decreasing on every interval in which $z(p)$ is strictly increasing, there exists a real $\delta > 0$ such that for any $p \in \mathcal{P}_i(\theta)$, there exists a $p' \in \mathcal{P}_i(\theta')$ satisfying $g_i(p) > g_i(p')$ whenever $\theta < \theta' < \theta + \delta$. Therefore, there exists a real $\delta > 0$ such that $f_i(\theta) > f_i(\theta')$ whenever $\theta < \theta' < \theta + \delta$.

Since the function $f_i(\theta)$ is continuous and it is locally strictly decreasing almost everywhere in $\bbR$, it must be strictly decreasing.
    
\end{proof}

\begin{lemma}
    Let $p_1 < p_2$ be two price points such that $p_1, p_2 \in \mathcal{P}_i(\theta)$ for some $\theta$. If $p + 1 / u_i'(p) \leq \theta$ for all $p \in (p_1, p_2)$, then $g_i(p_2) \geq g_i(p_1)$. If $p + 1 / u_i'(p) \geq \theta$ for all $p \in (p_1, p_2)$, then $g_i(p_2) \leq g_i(p_1)$.
    \label{lemma_two_prices}
\end{lemma}

\begin{proof}
    Using $p_1 + 1 / u_i'(p_1) = p_2 + 1 / u_i'(p_2) = \theta$, we have 
    \begin{align*}
        g_i(p_1) &= - e^{u_i(p_1)}/u_i'(p_1) = e^{u_i(p_1)} (p_1 - \theta)\\
        g_i(p_2) &= - e^{u_i(p_2)}/u_i'(p_2) = e^{u_i(p_2)} (p_2 - \theta).
    \end{align*}

    We let $w(p) = e^{u_i(p)} (p - \theta)$ and notice that $w(p_1) = g_i(p_1)$ and $w(p_2) = g_i(p_2)$. Now, we compute the derivative of $w(p)$ as
    \begin{align*}
        w'(p) = e^{u_i(p)} (1 + (p - \theta) u_i'(p) ).
    \end{align*}

    Since $u_i'(p) < 0$ for all $p$, we have $w'(p) \geq 0$ if and only if $p + 1 / u_i'(p) \leq \theta$. Hence, if $p + 1 / u_i'(p) \leq \theta$ for all $p \in (p_1, p_2)$, then $w'(p) \geq 0$ for all $p \in (p_1, p_2)$. Since $w(p)$ is continuous and differentiable, we conclude $g_i(p_2) \geq g_i(p_1)$. The symmetric result also follows similarly.
    
\end{proof}

\section{Proof of Theorem~\ref{thm_regret_ub} (Regret Upper Bound for Algorithm~\ref{alg:seq_assortment})}
\label{proof_regret_ub}

In the following section, we present our proof for Theorem~\ref{thm_regret_ub}. For better readability, we first present the overall proof using a series of technical lemmas. We provide the proofs for these technical lemmas later in Appendix~\ref{sect:proof_technical_lemmas}.

We start by recalling Proposition~\ref{prop_opt_assortments_and_prices} which defines $B_t$ as the unique solution of
\begin{align}
    B = \max_{S \in \mathcal{S}_K} \sum_{i \in S} v_{ti}(B)
\end{align}
where $v_{ti}(B) = \max_{p \in \bbR} \left\{ - e^{h_{ti}(p)} / h_{ti}'(p) : p + 1 / h_{ti}'(p) = B \right\}$. This proposition also asserts that the optimum prices are $p_{ti}^* = \argmax_{p \in \bbR} \left\{ - e^{h_{ti}(p)} / h_{ti}'(p) : p + 1 / h_{ti}'(p) = B_t \right\}$. Our first lemma shows that this fixed point $B_t$ lies within $[0, P_0]$ for some $P_0$ under our assumptions, allowing us to constrain our search for the fixed point into a bounded interval. This result also implies that the optimum prices $p_{ti}^*$ are bounded within $[0, P]$ for some $P$.

\begin{restatable}[Bounded optimum prices]{lemma}{lemmaoptprices}
Consider that the utility function for each item $i \in [N]$ is given by a differentiable function $h_{ti}(p)$ such that $h_{ti}(0) \leq 1 + \mu$ and its first order derivative satisfies $h_{ti}'(p) \leq - L_0$ for all $p \in \bbR$. Then, the fixed point satisfies $B_t \in [0, P_0(\mu)]$ and the optimum prices satisfy $p_{ti}^* \in [0, P(\mu)]$ for constants
\begin{align*}
    P_0(\mu) = \frac{e^\mu \cdot (0.6 + \log K)}{L_0} \quad \text{and} \quad P(\mu) = P_0(\mu) + \frac{1}{L_0}.
\end{align*}
For ease of notation, define $P_0 := P_0(1)$, $P := P(1)$, and $\widebar{P} := 1 + P$. 
\label{lemma_opt_prices}
\end{restatable}

Based on Assumption~\ref{assumption_stochastic_contexts}, the true utility functions satisfy $u_{ti}(0) \leq 1$ and $u_{ti}'(p) \leq - L_0$. Furthermore, as we show in the following proof, the estimated utility functions satisfy $h_{ti}(0) \leq 2$ and $h_{ti}'(p) \leq - L_0$. Therefore, both the true optimum prices under $u_{ti}(p)$ and the estimated optimum prices calculated under $h_{ti}(p)$ are bounded by $P$.



Recall that $T_0$ is the length of random initialization. At each round $t < T_0$, the algorithm chooses a subset $S_t$ uniformly at random from $\{S \subseteq [N] : |S| \leq K\}$ and sets $p_{ti} \in [1, 2]$ uniformly at random for all $i \in S_t$. Then, we use the assumption that there exists a constant $\sigma_0 > 0$ such that $\bbE[\vect{x}_{ti} \vect{x}_{ti}^\top] \succcurlyeq \sigma_0 \vect{I}$ and show how many rounds of initialization are required to achieve a target minimum eigenvalue.

\begin{restatable}[Initialization]{lemma}{lemmalambdamin}
Define our target minimum eigenvalue for $\vect{V}_{T_0} = \frac{1}{K^2} \sum_{t = 1}^{T_0-1} \sum_{i \in S_t} \widetilde{\vect{x}}_{ti} \widetilde{\vect{x}}_{ti}^\top$ as
\begin{align*}
\lambda_{\mathrm{min}}^0 = C_1 \frac{d \widebar{P}^2 \log^3(T)}{\sigma_0^2} 
\end{align*}
for some universal constant $C_1$. Then, there exist some positive, universal constant $C_2$ such that if the length of random initialization satisfies
\begin{equation*}
    T_0 \geq C_2 \frac{\lambda_{\mathrm{min}}^0 K}{\sigma_0},
\end{equation*}
then $\lambda_{\mathrm{min}}(\vect{V}_{T_0}) \geq \lambda_{\mathrm{min}}^0$ with probability at least $1 - \frac{1}{T}$.

\label{lemma_initalization_min_eigenvalue}
\end{restatable}

This condition is central in showing that the maximum likelihood estimator is consistent (Lemma~\ref{lemma_consistency}) and satisfies a finite-sample normality-type estimation error bound (Lemma~\ref{lemma_normality}). Similar to \cite{Li_Lu_Zhou_2017} and \cite{Oh_Iyengar_2021}, the independence assumption (Assumption \ref{assumption_stochastic_contexts}) on the feature vectors $\vect{x}_{ti}$ is only needed to ensure that $\wh{\vect{\theta}}_t$ sufficiently close to $\vect{\theta}^*$ at the end of the initialization phase. We do not require this stochasticity assumption in the rest of the regret analysis. Therefore, after the random initialization period of the first $T_0$ rounds, the context vectors $\vect{x}_{ti}$ can even be chosen adversarially as long as their norms $\|\vect{x}_{ti}\|$ are bounded and they satisfy the minimum price sensitivity condition $\langle \vect{\phi}^*, \vect{x}_{ti} \rangle \geq L_0$.

Next, we show that the probability of selection for any item $i \in S_t$ for the assortments $S_t$ and prices $\vect{p}_t$ offered by Algorithm \ref{alg:seq_assortment} can be estimated well enough using $\wh{\vect{\theta}}_t$ sufficiently close to $\vect{\theta}^*$. Namely, we let $\gamma = \log 2 / (8 \widebar{P}) < 1$ and define
\begin{align*}
    \mathcal{B}_\gamma := \{ \vect{\theta} : \|\vect{\theta} - \vect{\theta}^*\| \leq \gamma\}.
\end{align*}
Then, as we show in Lemma \ref{lemma_qti_ratio}, for any $\vect{\theta}_1, \vect{\theta}_2 \in \mathcal{B}_{\gamma}$, we have
\begin{align*}
    \frac{1}{\sqrt{2}} \leq \frac{q_{ti}(\vect{\theta}_1)}{q_{ti}(\vect{\theta}_2)} \leq \sqrt{2}.
\end{align*}

As a result of this relation, we obtain the following estimation results. These results show that $\vect{V}_{t}$ can estimate the Fisher Information Matrix $\mathcal{H}_t(\vect{\theta})$ within a constant factor in a small enough neighborhood around $\vect{\theta}^*$.  

\begin{restatable}[Regularity of Fisher Information]{lemma}{lemmahineq}
    Let $\vect{H}_t : \bbR^{2d} \to \bbR^{2d \times 2d}$ denote the function defined as
    \begin{align*}
        \vect{H}_t(\vect{\theta}) = \sum_{i \in S_{t}} q_{t i}(\vect{\theta}) \vect{\widetilde{x}}_{t i} \vect{\widetilde{x}}_{t i}^\top - \sum_{i \in S_{t}} \sum_{j \in S_{t}} q_{t i}(\vect{\theta}) q_{t j}(\vect{\theta}) \vect{\widetilde{x}}_{t i} \vect{\widetilde{x}}_{t j}^\top.
    \end{align*}
    Then, for any $\vect{\theta}_1, \vect{\theta}_2 \in \mathcal{B}_\gamma$, and any $t \geq T_0$, we have
    \begin{align*}
        \frac{1}{4} \vect{H}_t(\vect{\theta}_1) \preccurlyeq  \vect{H}_t (\vect{\theta}_2) \preccurlyeq 4 \vect{H}_t(\vect{\theta}_1).
    \end{align*}
    \label{lemma_H_ineq}
\end{restatable}

\begin{restatable}[Fisher Information Estimation]{lemma}{lemmavub}
    Let $t \geq T_0$ and assume that $\widehat{\vect{\theta}}_\tau \in \mathcal{B}_\gamma$ for all $T_0 \leq \tau < t$. Then, for any $\vect{\theta} \in \mathcal{B}_\gamma$, we have
    \begin{align*}
        \mathcal{H}_t(\vect{\theta}) := \sum_{\tau = 1}^{t-1} \vect{H}_\tau(\vect{\theta}) \succcurlyeq C_3 \vect{V}_{t}
    \end{align*}
    for some universal constant $C_3 > 0$.
    \label{lemma_V_ub}
\end{restatable}

The next result shows that our MLE estimates can reach and stay within the $\gamma$-neighborhood of the true parameter $\vect{\theta}^*$ with high probability as long as the initialization is successful.

\begin{restatable}[Consistency of MLE]{lemma}{lemmaconsistency}
Let $T_0$ be any round such that $\lambda_{\mathrm{min}}(\vect{V}_{T_0}) \geq \lambda_{\mathrm{min}}^0$. Then, we have 
\begin{align*}
    \bbP(\exists t \geq T_0, \widehat{\vect{\theta}}_t \notin \mathcal{B}_\gamma ) \leq \frac{1}{T}.
\end{align*}
\label{lemma_consistency}
\end{restatable}

Combining the results of Lemma~\ref{lemma_initalization_min_eigenvalue} and Lemma~\ref{lemma_consistency}, we can show that the conditions $\lambda_{\mathrm{min}}(\vect{V}_{T_0}) \geq \lambda_{\mathrm{min}}^0$ and $\widehat{\vect{\theta}}_t \in \mathcal{B}_\gamma$ for all $t \geq T_0$ are satisfied with probability $1 - \mathcal{O}(T^{-1})$ if we select
\begin{align}
T_0 = \Theta\left( \frac{\lambda_{\mathrm{min}}^0 K}{\sigma_0} \right) = \Theta\left(\frac{d \widebar{P}^2 K \log^3(T)}{\sigma_0^3} \right).
\label{eqn_T0}
\end{align}

Thus, we can define a \emph{good} event
\begin{align*}
    \mathcal{E}_0 = \left\{ \lambda_{\mathrm{min}}(\vect{V}_{T_0}) \geq \lambda_{\mathrm{min}}^0 \right\} \cap \left\{ \widehat{\vect{\theta}}_t \in \mathcal{B}_\gamma, \forall t \geq T_0  \right\}
\end{align*}
that holds with probability $1 - \mathcal{O}(T^{-1})$.

Now, given that the initialization successfully identifies a point $\gamma$-neighborhood of the true parameter $\vect{\theta}^*$, the next step is to construct tight confidence regions that contain the true parameter with high probability. The next Lemma establishes that we can construct a confidence region using the estimated Fisher information matrix $\vect{V}_t$.

\begin{restatable}[Normality of MLE]{lemma}{lemmanormality}
    Suppose the event $\mathcal{E}_0$ holds. Then, for any $t \geq T_0$,
    \begin{align}
        \|\widehat{\vect{\theta}}_t - \vect{\theta}^*\|_{\vect{V}_t} \leq C_5 \left( \sqrt{d \log \left( 1 + \frac{2t^3}{d} \right)} + \frac{1}{\sigma_0} \log(T) \right)
    \end{align}
    with probability at least $1 - \mathcal{O}(t^{-2})$.
    \label{lemma_normality}
\end{restatable}

For the selection of $T_0$ given in \eqref{eqn_T0_online}, we already showed that $\mathcal{E}_0$ holds with probability $1 - \mathcal{O}(T^{-1})$. Therefore, conditioned on $\mathcal{E}_0$ happens, we can further ensure that $\|\widehat{\vect{\theta}}_t - \vect{\theta}^*\|_{\vect{V}_t} \leq \alpha_t$ holds with probability at least $1 - t^{-2}$ if we choose the confidence radius as
\begin{align}
    \alpha_t = C_5 \left( \sqrt{d \log \left( 1 + \frac{2t^3}{d} \right)} + \frac{1}{\sigma_0} \log(T) \right).
    \label{eqn_alpha_t}
\end{align}

Consequently, for each $t \geq T_0$, we can define another \emph{good} event $\mathcal{E}_t = \{ \|\widehat{\vect{\theta}}_t - \vect{\theta}^*\|_{\vect{V}_t} \leq \alpha_t \}$ that holds with probability at least $1 - \mathcal{O}(t^{-2})$ conditioned on $\mathcal{E}_0$.

Now, we are ready to construct our optimistic utility functions using the confidence regions established in Lemma \ref{lemma_normality}. The following lemma establishes important properties for the optimistic utility functions constructed in Algorithm \ref{alg:seq_assortment}.

\begin{restatable}{lemma}{lemmautilityub}
Suppose $\mathcal{E}_t$ holds for all $t \geq T_0$. Let $\wt{h}_{ti} : \bbR \to \bbR$ be the function defined as
\begin{align*}
    \wt{h}_{ti}(p) := \langle \widehat{\vect{\psi}}_t, \vect{x}_{ti} \rangle - \langle \widehat{\vect{\phi}}_t, \vect{x}_{ti} \rangle \cdot p + g_{ti}(p),
\end{align*}
where $\widehat{\vect{\theta}}_t = (\widehat{\vect{\psi}}_t, \widehat{\vect{\phi}}_t)$ and $g_{ti}(p) := \alpha_t \|(\vect{x}_{ti}, -p\vect{x}_{ti})\|_{\vect{V}_{t}^{-1}}$. Furthermore, let $h_{ti} : \bbR \to \bbR$ be the function defined as
\begin{align*}
    h_{ti}(p) := \min_{p' \leq p} \left\{ \wt{h}_{ti} (p') - L_0 (p - p') \right\}.
\end{align*}

Then, the function $h_{ti}(p)$ is differentiable and satisfies
\begin{align}
h_{ti}(p) &\geq u_{ti}(p), \label{eq_ub_condition_1}\\
h_{ti}(p) - u_{ti}(p) &\leq 2 g_{ti}(p), \label{eq_ub_condition_2}\\
h'_{ti}(p) &\leq -L_0, \label{eq_ub_condition_3}\\
h_{ti}(0) &\leq 2 \label{eq_ub_condition_4}
\end{align}
for all $p \in \bbR$. 
\label{lemma_utility_ub}
\end{restatable}

We recall that Algorithm \ref{alg:seq_assortment} chooses the assortment $S_t$ and prices $\vect{p}_t$ by solving
\begin{align*}
    (S_t, \vect{p}_t) \in \argmax_{\substack{S \in \mathcal{S}_K \\ \vect{p} \in \bbR_{+}^{n}}} \wt{R}_t(S, \vect{p})
\end{align*}
where $\wt{R}_t(S, \vect{p})$ denotes the optimistic estimate of the revenue function as defined in \eqref{eqn_rev_ub}. Then, using the properties of the optimistic estimate of the utility functions $h_{ti}(p)$, we can show the following lemma.

\begin{restatable}{lemma}{lemmarevenueub}
Assume good event $\mathcal{E}_t$ holds for some $t \geq T_0$. Then,
\begin{align*}
    \text{(a)}& \quad R_t(S_t^*, \vect{p}_t^*) \leq \widetilde{R}_t(S_t, \vect{p}_t), \quad \text{and}\\
    \text{(b)}& \quad \widetilde{R}_t(S_t, \vect{p}_t) - R_t(S_t, \vect{p}_t) \leq 4 P \cdot \sum_{i \in S_t} q_{ti}(\vect{\theta}^*) g_{ti}(p_{ti}).
\end{align*}
\label{lemma_revenue_ub}
\end{restatable}

Now, we break the regret $\mathcal{R}_T$ into the initialization phase and the learning phase:
\begin{align*}
    \mathcal{R}_T &= \bbE \left[ \sum_{t=1}^{T_0-1} \left( R_t(S_t^*, \vect{p}_t^*) - R_t(S_t, \vect{p}_t) \right)\right] + \bbE \left[ \sum_{t=T_0}^{T} \left( R_t(S_t^*, \vect{p}_t^*) - R_t(S_t, \vect{p}_t) \right)\right] \\
    &\leq P T_0 + \bbE \left[ \sum_{t=T_0}^{T} \left( R_t(S_t^*, \vect{p}_t^*) - R_t(S_t, \vect{p}_t) \right)\right] \\
    &\leq P T_0 + \bbE \left[ \sum_{t=T_0}^{T} \left( \widetilde{R}_t(S_t, \vect{p}_t) - R_t(S_t, \vect{p}_t) \right) \right]
\end{align*}
where the first inequality follows from $R_t(S_t^*, \vect{p}_t^*) \leq P$ and the second inequality follows from property (a) in Lemma~\ref{lemma_revenue_ub}. Now, we decompose the remaining expectation term into two parts where the high probability event $\mathcal{E}_0$ holds and it does not.
\begin{align*}
    \mathcal{R}_T &\leq P T_0 + \bbE \left[ \sum_{t=T_0}^{T} \left( \widetilde{R}_t(S_t, \vect{p}_t) - R_t(S_t, \vect{p}_t) \right) \1(\mathcal{E}_0) \right] + \bbE \left[ \sum_{t=T_0}^{T} \left( \widetilde{R}_t(S_t, \vect{p}_t) - R_t(S_t, \vect{p}_t) \right) \1(\neg \mathcal{E}_0) \right] \\
    &\leq P T_0 + \sum_{t=T_0}^{T} \bbE \left[ \left( \widetilde{R}_t(S_t, \vect{p}_t) - R_t(S_t, \vect{p}_t) \right) \1(\mathcal{E}_0) \right] + \mathcal{O}(P).
\end{align*}
where the last inequality uses $\widetilde{R}_t(S_t, \vect{p}_t) \leq P$ and $\Pr(\neg \mathcal{E}_0) \leq \mathcal{O}(T^{-1})$. For each expectation term in the remaining summation, we can split it into two parts where the high probability event $\mathcal{E}_t$ holds and it does not:
\begin{align*}
    \bbE &\left[ \left( \widetilde{R}_t(S_t, \vect{p}_t) - R_t(S_t, \vect{p}_t) \right) \1(\mathcal{E}_0) \right] \\
    &= \bbE \left[ \left( \widetilde{R}_t(S_t, \vect{p}_t) - R_t(S_t, \vect{p}_t) \right) \1(\mathcal{E}_0) \1(\mathcal{E}_t) \right] + \bbE \left[ \left( \widetilde{R}_t(S_t, \vect{p}_t) - R_t(S_t, \vect{p}_t) \right) \1(\mathcal{E}_0) \1(\neg \mathcal{E}_t) \right]\\
    &\leq 4 P \alpha_t \sum_{i \in S_t} q_{ti}(\vect{\theta}^*) \| \wt{\vect{x}}_{ti} \|_{\vect{V}_{t}^{-1}} + \mathcal{O}(P \cdot t^{-2}).
\end{align*}
where the last inequality follows from property (b) in Lemma~\ref{lemma_revenue_ub} as well as $\Pr(\neg \mathcal{E}_t) \leq \mathcal{O}(t^{-2})$ and $\widetilde{R}_t(S_t, \vect{p}_t) \leq P$. As a result,
\begin{align*}
    \mathcal{R}_T &\leq P T_0 + 4 P \sum_{t=T_0}^{T} \alpha_t \sum_{i \in S_t} q_{ti}(\vect{\theta}^*) \| \wt{\vect{x}}_{ti} \|_{\vect{V}_{t}^{-1}} + \sum_{t=1}^{T} \mathcal{O}(P \cdot t^{-2}) +\mathcal{O}(P)\\
    &\leq P T_0 + 4 P \sum_{t=T_0}^{T} \alpha_t \sum_{i \in S_t} q_{ti}(\vect{\theta}^*) \| \wt{\vect{x}}_{ti} \|_{\vect{V}_{t}^{-1}} +\mathcal{O}(P)
\end{align*}
 Applying Cauchy-Schwarz and Jensen's inequalities in the second term, it follows that
\begin{align*}
\mathcal{R}_T &\leq P T_0 + 4 P \alpha_T \sqrt{ T  \sum_{t=T_0}^{T} \left( \sum_{i \in S_t} q_{ti}(\vect{\theta}^*) \| \wt{\vect{x}}_{ti} \|_{\vect{V}_{t}^{-1}} \right)^2 } + \mathcal{O}(P)\\
&\leq P T_0 + 4 P \alpha_T \sqrt{ T  \sum_{t=T_0}^{T} \sum_{i \in S_t} q_{ti}(\vect{\theta}^*) \| \wt{\vect{x}}_{ti} \|_{\vect{V}_{t}^{-1}}^2 } + \mathcal{O}(P).
\end{align*}

Applying Lemma~\ref{lemma_cumulative_max_uncertainty} and Lemma~\ref{lemma_qti_ratio}, we obtain 
\begin{align*}
    \mathcal{R}_T \leq P T_0 + 29 P \alpha_T \sqrt{ d K T \log ( T / d ) } + \mathcal{O}(P).
\end{align*}

\subsection{Proofs for Technical Lemmas}\label{sect:proof_technical_lemmas}

\lemmaoptprices*

\begin{proof}

Given the conditions on $h_{ti}(p)$, we have $h_{ti}(p) \leq 1 + \mu - L_0 p$ for all $p \geq 0$. Therefore, for any $B \geq 0$, we have
\begin{align*}
    v_{ti}(B) &= \max_{p \in \bbR} \left\{ - \exp\{h_{ti}(p)\} / h_{ti}'(p) : p + 1 / h_{ti}'(p) = B \right\} \\
    &\leq \max_{p \in \bbR} \left\{ \frac{\exp\{h_{ti}(B - 1 / h_{ti}'(p))\}}{L_0} \right\} \\
    &\leq \frac{\exp\{h_{ti}(B + 1 / L_0)\}}{L_0}\\
    &\leq \frac{1}{L_0} e^{\mu - L_0 B}.
\end{align*}

The first equality is the definition of $v_{ti}(B)$ and the second inequality uses the condition $h_{ti}'(p) \leq -L_0$ for all $p$, and the last inequality uses the result $h_{ti}(p) \leq 1 + \mu - L_0 p$ for all $p \geq 0$.

As a result, for any $S \in \mathcal{S}_K$, we have $\sum_{i \in S} v_{ti}(B) \leq \frac{K}{L_0} e^{\mu - L_0 B}$ for all $B \geq 0$.

Now, we let $B_u$ be the unique solution of the fixed point equation
\begin{align}
    B = \frac{K}{L_0} e^{\mu - L_0 B}.
\label{eq_fixed_point_ub}
\end{align}

Since the right-hand sides of \eqref{eq_fixed_point_ub} and \eqref{eq_fixed_point_assortment} are both positive for all $B \in \bbR$, the fixed points $B_u$ and $B_t$ are both positive. Furthermore, since the right-hand side of \eqref{eq_fixed_point_ub} is an upper bound for the right-hand side of \eqref{eq_fixed_point_assortment} for all $B \geq 0$, we must have $B_t \leq B_u$. 

In \eqref{eq_fixed_point_ub}, the left-hand side is increasing and the right-hand side is decreasing in $B$. Additionally, for $B = e^\mu (0.6 + \log(K)) / L_0$, the left-hand side of \eqref{eq_fixed_point_ub} is greater than the right-hand side. Hence, the fixed point satisfies $0 \leq B_t \leq B_u \leq P_0 = e^\mu (0.6 + \log(K)) / L_0$.

Furthermore, the optimum prices satisfy $p_{ti}^* + 1 / h_{ti}'(p_{ti}^*) = B_t$. Hence, $0 \leq p_{ti}^* + 1 / h_{ti}'(p_{ti}^*) \leq B_u$. Using that $h_{ti}'(p) \leq - L_0$, we have $0 \leq p_{ti}^* \leq P_0 + 1/L_0$.

\end{proof}

\lemmalambdamin*

\begin{proof}
    Let $\vect{\Sigma} = \bbE[\vect{x}_{ti} \vect{x}_{ti}^\top]$ and $\widetilde{\vect{\Sigma}} = \bbE[\widetilde{\vect{x}}_{ti} \widetilde{\vect{x}}_{ti}^\top]$. Then, noting that $p_{ti}$ is uniformly and independently distributed over $[1,2]$ for all $t \leq T_0$, we can write
    \begin{align*}
        \widetilde{\vect{\Sigma}} = \begin{bmatrix}
            \vect{\Sigma} & -\frac{3}{2} \vect{\Sigma}\\
            -\frac{3}{2} \vect{\Sigma} & \frac{7}{3} \vect{\Sigma}
        \end{bmatrix}.
    \end{align*}
    Then, using Schur's formula, each eigenvalue $\widetilde{\lambda}$ of $\widetilde{\vect{\Sigma}}$ are given by solutions of the equation
    \begin{align*}
        0 &= \det(\widetilde{\vect{\Sigma}} - \widetilde{\lambda} \vect{I})\\
        &= \det(\vect{\Sigma} - \widetilde{\lambda} \vect{I}) \det \left(\frac{7}{3} \vect{\Sigma} - \widetilde{\lambda} \vect{I} - \frac{9}{4} \vect{\Sigma} (\vect{\Sigma} - \widetilde{\lambda} \vect{I})^{-1} \vect{\Sigma} \right).
    \end{align*}

    Since the inverse of the matrix $\vect{\Sigma} - \widetilde{\lambda} \vect{I}$ appears on the right-hand side, we must have $\det(\vect{\Sigma} - \widetilde{\lambda} \vect{I}) \neq 0$. Hence, all eigenvalues must satisfy 
    \begin{align*}
        \det \left(\frac{7}{3} \vect{\Sigma} - \widetilde{\lambda} \vect{I} - \frac{9}{4} \vect{\Sigma} (\vect{\Sigma} - \widetilde{\lambda} \vect{I})^{-1} \vect{\Sigma} \right) = 0.
    \end{align*}

    Letting $\vect{\Sigma} = \vect{V} \vect{\Lambda} \vect{V}^\top$ be the eigen-decomposition of $\vect{\Sigma}$ with $\{\lambda_j\}_{j = 1}^{d}$ denoting the eigenvalues. Then, we can write
    \begin{align*}
       0 &= \det \left(\frac{7}{3} \vect{V} \vect{\Lambda} \vect{V}^\top - \widetilde{\lambda} \vect{I} - \frac{9}{4} \vect{V} \vect{\Lambda} \vect{V}^\top (\vect{V} \vect{\Lambda} \vect{V}^\top - \widetilde{\lambda} \vect{I})^{-1} \vect{V} \vect{\Lambda} \vect{V}^\top \right) \\
       &= \det (\vect{V})^2\det \left(\frac{7}{3}\vect{\Lambda} - \widetilde{\lambda} \vect{I} - \frac{9}{4} \vect{\Lambda} (\vect{\Lambda} - \widetilde{\lambda} \vect{I})^{-1} \vect{\Lambda} \right)\\
       &= \prod_{j = 1}^{d} \left(\frac{7}{3} \lambda_j - \widetilde{\lambda} - \frac{9}{4} \frac{\lambda_j^2}{\lambda_j - \widetilde{\lambda}} \right).
    \end{align*}
     Consequently, the eigenvalues of $\widetilde{\vect{\Sigma}}$ are given by
    \begin{align*}
        \widetilde{\lambda}_{j,1} = (20 + 2\sqrt{97}) \lambda_j \text{ and } \widetilde{\lambda}_{j,2} =(20 - 2\sqrt{97}) \lambda_j, \; \forall j \in [d].
    \end{align*}
     Since $\lambda_j \geq \sigma_0$ for all $j$  by Assumption \ref{assumption_stochastic_contexts}, $\lambda_{\mathrm{min}}(\widetilde{\vect{\Sigma}}) \geq C \sigma_0$ for some positive, universal constant $C$. Then, using Proposition 1 from \cite{Li_Lu_Zhou_2017}, we establish that there exist some positive, universal constants $C_1$ and $C_2$ such that if the length of random initialization satisfies
    \begin{equation*}
        T_0 \geq \left( \frac{C_2 \sqrt{d} + C_3 \sqrt{\log T}}{\sigma_0} \right)^2 + \frac{2 B K}{\sigma_0},
    \end{equation*}
    then $\lambda_{\mathrm{min}}(\vect{V}_{T_0}) \geq B$ with probability at least $1 - \frac{1}{T}$. Lastly, we set $B = \lambda_{\mathrm{min}}^0$ and observe that 
    \begin{align*}
        \lambda_{\mathrm{min}}^0 = C_1 \left( \frac{\widebar{P} d \log(T)}{\sigma_0} \right)^2 \geq C_4 \left( \frac{C_1 \sqrt{d} + C_2 \sqrt{\log T}}{\sigma_0} \right)^2
    \end{align*}
    for some universal constant $C_4 > 0$ since $\widebar{P} \geq 1$.

\end{proof}

\lemmahineq*

\begin{proof}

We start with some definitions that will be useful in the following proof. 

First, we let $\vect{\widetilde{x}}_{t 0} = \vect{0}_{2d}$ and $\widebar{S}_t = S_t \cup \{0\}$ denote the extended assortment that includes the null item. Then, we can write 
\begin{align*}
    \vect{H}_t(\vect{\theta}) = \sum_{i \in \widebar{S}_t} q_{t i}(\vect{\theta}) \vect{\widetilde{x}}_{t i} \vect{\widetilde{x}}_{t i}^\top - \sum_{i \in \widebar{S}_t} \sum_{j \in \widebar{S}_t} q_{t i}(\vect{\theta}) q_{t j}(\vect{\theta}) \vect{\widetilde{x}}_{t i} \vect{\widetilde{x}}_{t j}^\top.
\end{align*}
    
Then, we let $\vect{H}_1 = \vect{H}_t(\vect{\theta}_1)$ and $\vect{H}_2 = \vect{H}_t(\vect{\theta}_2)$ denote the matrices of interest.

Define the differences in the probabilities as $\delta_{t i} = q_{t i}(\vect{\theta}_1) - q_{t i}(\vect{\theta}_2)$.

Define the expected value of the context selection as $\widebar{\vect{x}}_1 = \sum_{i \in \widebar{S}_t} q_{t i}(\vect{\theta}_1) \vect{\widetilde{x}}_{t i}$ and $\widebar{\vect{x}}_2 = \sum_{i \in \widebar{S}_t} q_{t i}(\vect{\theta}_2) \vect{\widetilde{x}}_{t i}$.

Lastly, we define the mean-centered contexts as $\vect{w}_{ti} = \wt{\vect{x}}_{ti} - \widebar{\vect{x}}_1$ and $\wt{\vect{w}}_{ti} = \wt{\vect{x}}_{ti} - \widebar{\vect{x}}_2$.

Using these definitions and noting that $\sum_{i \in \widebar{S}_t} q_{t i}(\vect{\theta}_1) \vect{w}_{ti} = \vect{0}$ and $\sum_{i \in \widebar{S}_t} q_{t i}(\vect{\theta}_2) \wt{\vect{w}}_{ti} = \vect{0}$, we can write
\begin{align*}
    \vect{H}_1 &= \sum_{i \in \widebar{S}_t} q_{t i}(\vect{\theta}_1) \vect{w}_{ti} \vect{w}_{ti}^\top\\
    \vect{H}_2 &= \sum_{i \in \widebar{S}_t} q_{t i}(\vect{\theta}_2) \wt{\vect{w}}_{ti} \wt{\vect{w}}_{ti}^\top.
\end{align*}

Our initial goal is to show $\frac{1}{2} \wt{\vect{H}} \preccurlyeq \vect{H}_2 \preccurlyeq 2 \wt{\vect{H}}$ for an intermediate matrix defined as
\begin{align*}
    \wt{\vect{H}} = \sum_{i \in \widebar{S}_t} q_{t i}(\vect{\theta}_1) \wt{\vect{w}}_{ti} \wt{\vect{w}}_{ti}^\top.
\end{align*}

To achieve this goal, it is sufficient to show that $- \frac{1}{2} \wt{\vect{H}} \preccurlyeq \wt{\vect{H}} - \vect{H}_2 \preccurlyeq \frac{1}{2} \wt{\vect{H}}$. We notice that this difference can be written as
\begin{align*}
    \wt{\vect{H}} - \vect{H}_2 = \sum_{i \in \widebar{S}_t} \delta_{t i} \wt{\vect{w}}_{ti} \wt{\vect{w}}_{ti}^\top.
\end{align*}

To prove that this inequality holds, it is sufficient to show that $|\delta_{ti}| \leq \frac{1}{2} q_{t i}(\vect{\theta}_1)$. Using Lemma \ref{lemma_qti_ratio}, we can show that
\begin{align*}
    \frac{\delta_{ti}}{q_{t i}(\vect{\theta}_1)} = \frac{q_{t i}(\vect{\theta}_1) - q_{t i}(\vect{\theta}_2)}{q_{t i}(\vect{\theta}_1)} = 1 - \frac{q_{t i}(\vect{\theta}_2)}{q_{t i}(\vect{\theta}_1)} \leq 1 - \frac{1}{\sqrt{2}} < \frac{1}{2}
\end{align*}
and similarly
\begin{align*}
    \frac{-\delta_{ti}}{q_{t i}(\vect{\theta}_1)} = \frac{q_{t i}(\vect{\theta}_2) - q_{t i}(\vect{\theta}_1)}{q_{t i}(\vect{\theta}_1)} = \frac{q_{t i}(\vect{\theta}_2)}{q_{t i}(\vect{\theta}_1)} - 1 \leq \sqrt{2} - 1 < \frac{1}{2}.
\end{align*}

With this, we showed that $\frac{1}{2} \wt{\vect{H}} \preccurlyeq \vect{H}_2 \preccurlyeq 2 \wt{\vect{H}}$. To show the final intended result, the next step is to show $\frac{1}{2} \vect{H}_1 \preccurlyeq \wt{\vect{H}} \preccurlyeq 2 \vect{H}_1$. Similar to the previous part, it is sufficient to show that $- \frac{1}{2} \vect{H}_1 \preccurlyeq \wt{\vect{H}} - \vect{H}_1 \preccurlyeq \frac{1}{2} \vect{H}_1$. We can write this difference as
\begin{align*}
\wt{\vect{H}} - \vect{H}_1 &= \sum_{i \in \widebar{S}_t} q_{t i}(\vect{\theta}_1) \left[  \wt{\vect{w}}_{ti} \wt{\vect{w}}_{ti}^\top - \vect{w}_{ti} \vect{w}_{ti}^\top \right]\\
&= \sum_{i \in \widebar{S}_t} q_{t i}(\vect{\theta}_1) \left[ (\vect{w}_{ti} + \widebar{\vect{x}}_1 - \widebar{\vect{x}}_2) (\vect{w}_{ti} + \widebar{\vect{x}}_1 - \widebar{\vect{x}}_2)^\top - \vect{w}_{ti} \vect{w}_{ti}^\top \right]\\
&= \sum_{i \in \widebar{S}_t} q_{t i}(\vect{\theta}_1) \left[ (\widebar{\vect{x}}_1 - \widebar{\vect{x}}_2) (\widebar{\vect{x}}_1 - \widebar{\vect{x}}_2)^\top + (\widebar{\vect{x}}_1 - \widebar{\vect{x}}_2) \vect{w}_{ti}^\top + \vect{w}_{ti} (\widebar{\vect{x}}_1 - \widebar{\vect{x}}_2)^\top \right]\\
&= (\widebar{\vect{x}}_2 - \widebar{\vect{x}}_1) (\widebar{\vect{x}}_2 - \widebar{\vect{x}}_1)^\top.
\end{align*}

Now, we note that $\widebar{\vect{x}}_2 - \widebar{\vect{x}}_1 = - \sum_{i \in \widebar{S}_t} \delta_{ti} \vect{\widetilde{x}}_{t i}$. On the other hand, we also have
\begin{align*}
    \sum_{i \in \widebar{S}_t} q_{t i}(\vect{\theta}_2) \vect{w}_{ti} &= \sum_{i \in \widebar{S}_t} q_{t i}(\vect{\theta}_2) ( \wt{\vect{x}}_{ti} - \widebar{\vect{x}}_1 )\\
    &=\sum_{i \in \widebar{S}_t} q_{t i}(\vect{\theta}_2) \wt{\vect{x}}_{ti} - \sum_{i \in \widebar{S}_t} q_{t i}(\vect{\theta}_1) \vect{\widetilde{x}}_{t i}\\
    &=  - \sum_{i \in \widebar{S}_t} \delta_{ti} \vect{\widetilde{x}}_{t i}.
\end{align*}

Hence, we can write $\widebar{\vect{x}}_2 - \widebar{\vect{x}}_1 = \sum_{i \in \widebar{S}_t} q_{t i}(\vect{\theta}_2) \vect{w}_{ti} = \sum_{i \in \widebar{S}_t} q_{t i}(\vect{\theta}_2) \vect{w}_{ti} - \sum_{i \in \widebar{S}_t} q_{t i}(\vect{\theta}_1) \vect{w}_{ti} = - \sum_{i \in \widebar{S}_t} \delta_{ti} \vect{w}_{ti}$. Putting our results together, we have
\begin{align*}
    \wt{\vect{H}} - \vect{H}_1 &= \left( \sum_{i \in \widebar{S}_t} \delta_{ti} \vect{w}_{ti} \right) \left( \sum_{i \in \widebar{S}_t} \delta_{ti} \vect{w}_{ti} \right)^\top\\
     &=  \sum_{i \in \widebar{S}_t} \sum_{j \in \widebar{S}_t} \delta_{ti} \delta_{tj} \vect{w}_{ti} \vect{w}_{tj}^\top\\
     &= \frac{1}{2} \sum_{i \in \widebar{S}_t} \sum_{j \in \widebar{S}_t} \delta_{ti} \delta_{tj} (\vect{w}_{ti} \vect{w}_{tj}^\top + \vect{w}_{tj} \vect{w}_{ti}^\top)
\end{align*}

Using the inequality $- 2 \vect{w}_{ti} \vect{w}_{ti}^\top \preccurlyeq (\vect{w}_{ti} \vect{w}_{tj}^\top + \vect{w}_{tj} \vect{w}_{ti}^\top) \preccurlyeq 2 \vect{w}_{ti} \vect{w}_{ti}^\top$ and the fact that $\sum_{j \in \widebar{S}_t} |\delta_{tj}| \leq 2$, we can show
\begin{align*}
    - 2 \sum_{i \in \widebar{S}_t} |\delta_{ti}| \vect{w}_{ti} \vect{w}_{ti}^\top \preccurlyeq \wt{\vect{H}} - \vect{H}_1 \preccurlyeq 2 \sum_{i \in \widebar{S}_t} |\delta_{ti}| \vect{w}_{ti} \vect{w}_{ti}^\top
\end{align*}

Now, we use our result $|\delta_{ti}| \leq \frac{1}{2} q_{t i}(\vect{\theta}_1)$ to conclude $- \frac{1}{2} \vect{H}_1 \preccurlyeq \wt{\vect{H}} - \vect{H}_1 \preccurlyeq \frac{1}{2} \vect{H}_1$. This inequality implies $\frac{1}{2} \vect{H}_1 \preccurlyeq \wt{\vect{H}} \preccurlyeq 2 \vect{H}_1$. Combining with the previous result $\frac{1}{2} \wt{\vect{H}} \preccurlyeq \vect{H}_2 \preccurlyeq 2 \wt{\vect{H}}$, we show the final result.
    
\end{proof}

\lemmavub*

\begin{proof}

For any $t \geq T_0$, $\vect{V}_{t}$ is given by
\begin{align*}
    \vect{V}_{t} &= \frac{1}{K^2} \sum_{\tau = 1}^{T_0 - 1} \sum_{i \in S_\tau} \wt{\vect{x}}_{\tau i} \wt{\vect{x}}_{\tau i}^\top + \sum_{\tau = T_0}^{t - 1} \left[ \sum_{i \in S_\tau} q_{\tau i}(\vect{\wh{\theta}}_\tau) \wt{\vect{x}}_{\tau i} \wt{\vect{x}}_{\tau i}^\top - \sum_{i \in S_t} \sum_{j \in S_\tau} q_{\tau i}(\vect{\wh{\theta}}_\tau) q_{\tau j}(\vect{\wh{\theta}}_\tau) \wt{\vect{x}}_{\tau i} \wt{\vect{x}}_{\tau j}^\top \right] \\
    &= \frac{1}{K^2} \sum_{\tau = 1}^{T_0 - 1} \sum_{i \in S_\tau} \wt{\vect{x}}_{\tau i} \wt{\vect{x}}_{\tau i}^\top + \sum_{\tau = T_0}^{t - 1} \vect{H}_\tau(\vect{\wh{\theta}}_\tau). 
\end{align*}

Now, we will upper bound these two terms separately. 

To upper bound the terms for $\tau \geq T_0$, we use Lemma \ref{lemma_H_ineq} which states that $\vect{H}_\tau(\vect{\theta}) \succcurlyeq \frac{1}{4} \vect{H}_\tau(\vect{\wh{\theta}}_\tau)$.

For $\tau < T_0$, we use
\begin{align*}
    \vect{H}_\tau(\vect{\theta}) &= \sum_{i \in S_{t}} q_{t i}(\vect{\theta}) \vect{\widetilde{x}}_{t i} \vect{\widetilde{x}}_{t i}^\top - \sum_{i \in S_{t}} \sum_{j \in S_{t}} q_{t i}(\vect{\theta}) q_{t j}(\vect{\theta}) \vect{\widetilde{x}}_{t i} \vect{\widetilde{x}}_{t j}^\top \\
    &= \sum_{i \in S_{t}} q_{t i}(\vect{\theta}) \vect{\widetilde{x}}_{t i} \vect{\widetilde{x}}_{t i}^\top - \frac{1}{2} \sum_{i \in S_{t}} \sum_{j \in S_{t}} q_{t i}(\vect{\theta}) q_{t j}(\vect{\theta}) (\vect{\widetilde{x}}_{t i} \vect{\widetilde{x}}_{t j}^\top + \vect{\widetilde{x}}_{t j} \vect{\widetilde{x}}_{t i}^\top) \\
    &\succcurlyeq \sum_{i \in S_{t}} q_{t i}(\vect{\theta}) \vect{\widetilde{x}}_{t i} \vect{\widetilde{x}}_{t i}^\top - \frac{1}{2} \sum_{i \in S_{t}} \sum_{j \in S_{t}} q_{t i}(\vect{\theta}) q_{t j}(\vect{\theta}) (\vect{\widetilde{x}}_{t i} \vect{\widetilde{x}}_{t i}^\top + \vect{\widetilde{x}}_{t j} \vect{\widetilde{x}}_{t j}^\top) \\
    &= \sum_{i \in S_{t}} q_{t i}(\vect{\theta}) \vect{\widetilde{x}}_{t i} \vect{\widetilde{x}}_{t i}^\top - \sum_{i \in S_{t}} \sum_{j \in S_{t}} q_{t i}(\vect{\theta}) q_{t j}(\vect{\theta}) \vect{\widetilde{x}}_{t i} \vect{\widetilde{x}}_{t i}^\top \\
    &= \sum_{i \in S_{t}} q_{t i}(\vect{\theta}) q_{t 0}(\vect{\theta}) \vect{\widetilde{x}}_{t i} \vect{\widetilde{x}}_{t i}^\top \\
    &\succcurlyeq \frac{\nu}{K^2} \sum_{i \in S_\tau} \wt{\vect{x}}_{\tau i} \wt{\vect{x}}_{\tau i}^\top 
\end{align*}
where we define the constant
\begin{align*}
    \nu &:= K^2 \cdot \min_{t < T_0} \inf_{\vect{\theta} \in \mathcal{B}_\gamma} q_{t i}(\vect{\theta}) q_{t 0}(\vect{\theta}) > 0.
\end{align*}

Next, we will show that $\nu > C_3'$ for some universal constant $C_3' > 0$. For any $\vect{\theta}$, let $\vect{\theta} = (\vect{\psi}, \vect{\phi})$. If $\vect{\theta} \in \mathcal{B}_{\gamma}$, we can show that $\|\vect{\psi} - \vect{\psi}^*\| \leq \gamma < 1$ and $\|\vect{\phi} - \vect{\phi}^*\| \leq \gamma < 1$, 
\begin{align*}
    |\langle \vect{\psi}, \vect{x}_{ti} \rangle| &\leq |\langle \vect{\psi}^*, \vect{x}_{ti} \rangle| + |\langle \vect{\psi} - \vect{\psi}^*, \vect{x}_{ti} \rangle| \leq 1 + \gamma < 2\quad \text{and}\\
    |\langle \vect{\phi}, \vect{x}_{ti} \rangle| &\leq |\langle \vect{\phi}^*, \vect{x}_{ti} \rangle| + |\langle \vect{\phi} - \vect{\phi}^*, \vect{x}_{ti} \rangle| \leq 1 + \gamma < 2.
\end{align*}

Note that for all $t < T_0$, we have $1 \leq p_{ti} \leq 2$ for all $i \in S_t$. Therefore, for any $\vect{\theta} \in \mathcal{B}_{\gamma}$ and $t < T_0$, we have
\begin{align*}
    q_{ti}(\vect{\theta}) q_{t0}(\vect{\theta}) &= \frac{\exp (\langle \vect{\psi}, \vect{x}_{ti} \rangle - \langle \vect{\phi}, \vect{x}_{ti} \rangle p_{ti})}{ \left( 1 + \sum_{j \in S_t} \exp (\langle \vect{\psi}, \vect{x}_{tj} \rangle - \langle \vect{\phi}, \vect{x}_{tj} \rangle p_{tj}) \right)^2 }\\
    &> \frac{e^{- 6}}{\left(1 + K e^6 \right)^2},
\end{align*}
showing that $\nu > C_3'$ for some constant $C_3' > 0$.

Letting $C_3 = \min \{C_3', 1/4\}$, we can show that
\begin{align*}
    C_3 \vect{V}_{t} &= \frac{C_3}{K^2} \sum_{\tau = 1}^{T_0 - 1} \sum_{i \in S_\tau} \wt{\vect{x}}_{\tau i} \wt{\vect{x}}_{\tau i}^\top + C_3 \sum_{\tau = T_0}^{t - 1} \vect{H}_\tau(\vect{\wh{\theta}}_\tau) \\
    &\preccurlyeq \sum_{\tau = 1}^{t - 1} \vect{H}_\tau(\vect{\theta}).
\end{align*}

\end{proof}

\lemmaconsistency*

\begin{proof}

Recall that the gradient of the negative log-likelihood is given by
\begin{align*}
    \nabla_{\vect{\theta}} \ell_t (\vect{\theta}) &= \sum_{\tau = 1}^{t-1} \sum_{i \in S_\tau} (q_{\tau i}( \vect{\theta} ) - y_{\tau i}) \wt{\vect{x}}_{\tau i}
\end{align*}
and we have $\nabla_{\vect{\theta}} \ell_t (\widehat{\vect{\theta}}_t) = 0$ by definition of $\widehat{\vect{\theta}}_t$.

We can write the expectation of $\nabla_{\vect{\theta}} \ell_t (\vect{\theta})$ over the user choices $y_{\tau i}$ as
\begin{align*}
    G_{t}(\vect{\theta}) := \bbE[\nabla_{\vect{\theta}} \ell_t (\vect{\theta})] = \sum_{\tau = 1}^{t-1} \sum_{i \in S_{\tau}} \left( q_{\tau i}(\vect{\theta} ) - q_{\tau i}( \vect{\theta}^* ) \right) \vect{\widetilde{x}}_{\tau i}.
\end{align*}
We can show that
\begin{align*}
    G_{t}(\vect{\theta^*}) = 0 \text{ and } G_{t}(\widehat{\vect{\theta}}_t) = \sum_{\tau = 1}^{t-1} \sum_{i \in S_{\tau}} \epsilon_{\tau i} \vect{\widetilde{x}}_{\tau i},
\end{align*}
where $\epsilon_{t i} = y_{ti} - q_{t i}(\vect{\theta}^*)$ are sub-Gaussian random variables with parameter $1$. Note that collections of variables $\{\epsilon_{ti}\}_{i \in S_t}$ are independent over $t$, but the variables within each collection are not independent.

    For any $\vect{\theta}_1, \vect{\theta}_2 \in \mathbb{R}^{2d}$ and any $\vect{z} \in \mathbb{R}^{2d}$, the mean value theorem implies that there exists some $\overbar{\vect{\theta}} = \lambda \vect{\theta}_1 + (1 - \lambda) \vect{\theta}_2$ with $0 < \lambda < 1$, such that 
    \begin{align*}
        \vect{z}^\top \left( G_t(\vect{\theta}_1) - G_t(\vect{\theta}_2) \right) = \vect{z}^\top \mathcal{H}_t(\overbar{\vect{\theta}}) (\vect{\theta}_1 - \vect{\theta}_2)
    \end{align*}
    where we defined 
    \begin{align*}
        \mathcal{H}_t(\vect{\theta}) &:= \nabla_{\vect{\theta}}G_t(\vect{\theta}) \\
        &= \sum_{\tau = 1}^{t-1} \sum_{i \in S_{\tau}} \vect{\widetilde{x}}_{\tau i} \nabla_{\vect{\theta}}q_{\tau i}(\vect{\theta}).
    \end{align*}
    
    Recalling the definition
    \begin{align*}
        \vect{H}_\tau(\vect{\theta}) &= \sum_{i \in S_{\tau}} \vect{\widetilde{x}}_{\tau i} \nabla_{\vect{\theta}}q_{\tau i}(\vect{\theta}) \\
         &= \sum_{i \in S_{\tau}} q_{\tau i}(\vect{\theta}) \vect{\widetilde{x}}_{\tau i} \vect{\widetilde{x}}_{\tau i}^\top - \sum_{i \in S_{\tau}} \sum_{j \in S_{\tau}} q_{\tau i}(\vect{\theta}) q_{\tau j}(\vect{\theta}) \vect{\widetilde{x}}_{\tau i} \vect{\widetilde{x}}_{\tau j}^\top,
    \end{align*}
    we also see that $\mathcal{H}_t(\vect{\theta}) = \sum_{\tau = 1}^{t-1} \vect{H}_\tau(\vect{\theta})$.

    Now, we're ready to complete the proof with strong induction. The base case is $t = T_0$ and we proceed with inductive steps for each $t \in \{T_0+1, T_0+2, \dots, T\}$.

    We start with proving the inductive steps. Assuming that $\|\widehat{\vect{\theta}}_t - \vect{\theta}^*\| \leq \gamma$ for all $T_0 \leq \tau < t$, we have $\mathcal{H}_t(\vect{\theta}) \succcurlyeq C_3 \vect{V}_t$ for any $\vect{\theta} \in \mathcal{B}_\gamma$ using Lemma~\ref{lemma_V_ub}. Therefore, we can write
    \begin{align*}
        (\vect{\theta}_1 - \vect{\theta}_2)^\top (G_t(\vect{\theta}_1) - G_t(\vect{\theta}_2)) \geq C_3 (\vect{\theta}_1 - \vect{\theta}_2)^\top  \vect{V}_t (\vect{\theta}_1 - \vect{\theta}_2) > 0
    \end{align*}
    for any $\vect{\theta}_1 \neq \vect{\theta}_2$ and therefore $G_t(\vect{\theta})$ is an injection from $\bbR^{2d}$ to $\bbR^{2d}$. This allows us to use Lemma~A of~\cite{Chen_Hu_Ying_1999} which implies that
    \begin{align*}
        \left \{ \vect{\theta} : \|G_t(\vect{\theta})\|_{\vect{V}_t^{-1}} \leq C_3 \gamma \sqrt{\lambda_{\mathrm{min}}(\vect{V}_t)} \right\} \subseteq \mathcal{B}_\gamma.
    \end{align*}

    In addition, Lemma~15 of \cite{Oh_Iyengar_2021} shows that the event
    \begin{align*}
        \mathcal{E}_G := \left\{ \|G_t(\widehat{\vect{\theta}}_t)\|_{\vect{V}_t^{-1}} \leq 4 \sqrt{4d + \log(1 / \delta)} \right\}
    \end{align*}
    holds with probability at least $1 - \delta$. Thus, $\|\widehat{\vect{\theta}}_t - \vect{\theta}^* \| \leq \gamma$ holds with probability at least $1 - \delta$ when $\lambda_{\mathrm{min}}(\vect{V}_t) \geq \lambda_{\mathrm{min}}(\vect{V}_{T_0}) \geq \frac{16}{C_3^2 \gamma^2} (4d + \log(1 / \delta))$. Since we have $\lambda_{\mathrm{min}}^0 = C_1 \left( \frac{\widebar{P} d \log(T)}{\sigma_0} \right)^2 \geq \frac{16}{C_3^2 \gamma^2} (4d + 2 \log(T))$ for some constant $C_1$, the minimum eigenvalue condition is satisfied when $\lambda_{\mathrm{min}}(\vect{V}_t) \geq \lambda_{\mathrm{min}}^0$. 

    For the base case $t = T_0$, we similarly have $\mathcal{H}_{T_0}(\vect{\theta}) \succcurlyeq C_3 \vect{V}_{T_0}$ for any $\vect{\theta} \in \mathcal{B}_\gamma$ by Lemma~\ref{lemma_V_ub}. Therefore, we can follow similar steps for $t = T_0$ to argue that $\|\widehat{\vect{\theta}}_{T_0} - \vect{\theta}^* \| \leq \gamma$ holds true with probability at least $1 - T^{-2}$ when $\lambda_{\mathrm{min}}(\vect{V}_{T_0}) \geq \frac{16}{C_3^2 \gamma^2} (4d + 2 \log(T))$. Similarly, the minimum eigenvalue condition is satisfied when $\lambda_{\mathrm{min}}(\vect{V}_t) \geq \lambda_{\mathrm{min}}^0$.
    
    Taking a union bound over the base case and the inductive steps of the proof, we complete the proof of the theorem.
    
\end{proof}

\lemmanormality*

\begin{proof}

Following the proof of Lemma~\ref{lemma_consistency}, we use $\mathcal{H}_t(\overbar{\vect{\theta}}) \succcurlyeq C_3 \vect{V}_t$ to obtain
\begin{align}
\|G(\widehat{\vect{\theta}}_t)\|_{\vect{V}_t^{-1}}^2 &= \|G(\widehat{\vect{\theta}}_t) - G(\vect{\theta}^*)\|_{\vect{V}_t^{-1}}^2\nonumber\\
&\geq (\widehat{\vect{\theta}}_t - \vect{\theta}^*)^\top \mathcal{H}_t(\overbar{\vect{\theta}}) 
 \vect{V}_t^{-1} \mathcal{H}_t(\overbar{\vect{\theta}}) (\widehat{\vect{\theta}}_t - \vect{\theta}^*)\nonumber\\
 &\geq C_3^2 \|\widehat{\vect{\theta}}_t - \vect{\theta}^*\|_{\vect{V}_t}^2
\label{eqn_g_theta_hat}
\end{align}
for any $\widehat{\vect{\theta}}_t \in \{ \vect{\theta} : \|\vect{\theta} - \vect{\theta}^*\| \leq \gamma\}$. 

The next step is to upper bound $\|G(\widehat{\vect{\theta}}_t)\|_{\vect{V}_t^{-1}}$. We first separate it into two terms that correspond to the initialization rounds and the remaining rounds respectively. That is, we write
\begin{align*}
    \|G(\widehat{\vect{\theta}}_t)\|_{\vect{V}_t^{-1}} &= \left\| G(\widehat{\vect{\theta}}_{T_0}) + G(\widehat{\vect{\theta}}_t) - G(\widehat{\vect{\theta}}_{T_0}) \right\|_{\vect{V}_t^{-1}} \\
    &\leq \left\| G(\widehat{\vect{\theta}}_{T_0}) \right\|_{\vect{V}_t^{-1}} + \left\|G(\widehat{\vect{\theta}}_t) - G(\widehat{\vect{\theta}}_{T_0}) \right\|_{\vect{V}_t^{-1}}\\
    &\leq \|G(\widehat{\vect{\theta}}_{T_0})\|_{\vect{V}_{T_0}^{-1}} + \left\|G(\widehat{\vect{\theta}}_t) - G(\widehat{\vect{\theta}}_{T_0}) \right\|_{\vect{V}_t^{-1}}
\end{align*}
where the last inequality follows from $\vect{V}_t \succcurlyeq \vect{V}_{T_0}$ for any $t \geq T_0$.

We upper bound the first term using Lemma \ref{lemma_G_T0_ub} which states that
\begin{align}
    \|G(\widehat{\vect{\theta}}_{T_0})\|_{\vect{V}_{T_0}^{-1}} \leq \frac{C_5}{\sigma_0} \log(T)
\label{eqn_init_grad}
\end{align}
with probability $1 - \mathcal{O}(T^{-2})$.
    
To upper bound the second term, we use an improved self-normalized bound for vector-valued martingales as given in Theorem \ref{theorem_martingale}. In using this result, we let $\vect{\epsilon}_t$ denote the random vector with entries $\epsilon_{t i} = y_{ti} - q_{ti}(\vect{\theta}^*)$ and we let $\wt{\vect{X}}_t \in R^{2d \times K}$ denote the matrix with columns $\vect{\widetilde{x}}_{t i}$. We note that we have $\|\vect{\epsilon}_t\|_1 \leq 2$ and 
\begin{align*}
    \wt{\vect{X}}_t \Sigma_t \wt{\vect{X}}_t^\top = H_t(\vect{\theta}^*) = \sum_{i \in S_{t}} q_{t i}(\vect{\theta}^*) \vect{\widetilde{x}}_{t i} \vect{\widetilde{x}}_{t i}^\top - \sum_{i \in S_{t}} \sum_{j \in S_{t}} q_{t i}(\vect{\theta}^*) q_{t j}(\vect{\theta}^*) \vect{\widetilde{x}}_{t i} \vect{\widetilde{x}}_{t j}^\top
\end{align*}
where $\Sigma_t$ is the covariance matrix $\bbE[\vect{\epsilon}_t \vect{\epsilon}_t^\top]$. As a result, Theorem \ref{theorem_martingale} shows that
    \begin{align*}
        \|G(\widehat{\vect{\theta}}_t) - G(\widehat{\vect{\theta}}_{T_0})\|_{\vect{V}_t^{-1}} \leq \frac{\sqrt{\lambda}}{4} + \frac{4}{\sqrt{\lambda}} \log \left( 
        \frac{\det (\vect{V}_t)^{1/2}}{\delta \lambda^{d}}\right) +  \frac{8}{\sqrt{\lambda}} d \log(2)
    \end{align*}
    with probability at least $1 - \delta$ for any $0 < \lambda < \lambda_{\mathrm{min}}(\vect{V}_{T_0})$. 
    
    Then we combine with Lemma~\ref{lemma_det_ub} to obtain
    \begin{align*}
        \|G(\widehat{\vect{\theta}}_t) - G(\widehat{\vect{\theta}}_{T_0})\|_{\vect{V}_t^{-1}} &\leq \frac{\sqrt{\lambda}}{4} + \frac{4d }{\sqrt{\lambda}} \log \left( \frac{t \widebar{P}^2}{d \lambda} \right) + \frac{4}{\sqrt{\lambda}} \log \left( 
        \frac{1}{\delta}\right) +  \frac{8}{\sqrt{\lambda}} d \log(2)\\
        &\leq \frac{\sqrt{\lambda}}{4} + \frac{4d }{\sqrt{\lambda}} \left(\log \left( \frac{t }{d } \right) + \log \left( 
        \frac{1}{\delta}\right) + 2 \log(2) \right)\\
        &\leq \frac{\sqrt{\lambda}}{4} + \frac{4d }{\sqrt{\lambda}} \log \left( 1 + \frac{2t}{d \delta} \right)
    \end{align*}
    for any $\lambda \geq \widebar{P}^2$. Accordingly, we set $\lambda = \max \left\{ \widebar{P}^2, 16 d \log \left(1 + \frac{2t}{d \delta} \right) \right\}$ to obtain
    \begin{align}
        \|G(\widehat{\vect{\theta}}_t) - G(\widehat{\vect{\theta}}_{T_0})\|_{\vect{V}_t^{-1}} \leq 2 \sqrt{d \log \left( 1 + \frac{2t}{d \delta} \right)}.
    \label{eqn_g_V_ub}
    \end{align}
    Now, we set $\delta = t^{-2}$ and obtain $\lambda = 16 d \log \left( 1 + \frac{2t}{d \delta} \right) \geq 16 d \log \left( 1 + \frac{2}{d} \right) \geq 1$. Lastly, we confirm that $1 < \lambda < \lambda_{\mathrm{min}}(\vect{V}_{T_0})$ is satisfied for our selection of $\lambda$. On the other hand, we can verify that $\lambda < \lambda_{\mathrm{min}}(\vect{V}_{T_0})$ is satisfied under good event $\mathcal{E}_0$ because $\lambda_{\mathrm{min}}(\vect{V}_{T_0}) \geq \lambda_{\mathrm{min}}^0 = C_1 \frac{d \widebar{P}^2 \log^3(T)}{\sigma_0^2} > \lambda$ for some constant $C_1$. 

    Combining \eqref{eqn_init_grad} and \eqref{eqn_g_V_ub} gives the stated result in the lemma.

\end{proof}

\lemmautilityub*

\begin{proof}

Recall the definition of the utility function
\begin{align*}
    u_{ti}(p) = \langle \vect{\psi}^*, \vect{x}_{ti} \rangle - \langle \vect{\phi}^*, \vect{x}_{ti} \rangle \cdot p.
\end{align*}
and recall the definition $\widetilde{\vect{x}}_{ti} = (\vect{x}_{ti}, -p\vect{x}_{ti})$ to write
\begin{align*}
|\langle \widehat{\vect{\theta}}_t, \widetilde{\vect{x}}_{ti} \rangle - \langle \vect{\theta}^*, \widetilde{\vect{x}}_{ti} \rangle| &= \left| \langle \vect{V}_{t}^{1/2} (\widehat{\vect{\theta}}_t - \vect{\theta}^*), \vect{V}_{t}^{-1/2} \widetilde{\vect{x}}_{ti} \rangle \right|\\
&\leq \| \vect{V}_{t}^{1/2} (\widehat{\vect{\theta}}_t - \vect{\theta}^*)\| \| \vect{V}_{t}^{-1/2} \widetilde{\vect{x}}_{ti}\|\\
&\leq  \|\widehat{\vect{\theta}}_t - \vect{\theta}^*\|_{\vect{V}_{t}} \| \widetilde{\vect{x}}_{ti}\|_{\vect{V}_{t}^{-1}}\\
&\leq  \|\widehat{\vect{\theta}}_t - \vect{\theta}^*\|_{\vect{V}_{t}} \|(\vect{x}_{ti}, -p\vect{x}_{ti})\|_{\vect{V}_{t}^{-1}}\\
&\leq \alpha_t \|(\vect{x}_{ti}, -p\vect{x}_{ti})\|_{\vect{V}_{t}^{-1}}.
\end{align*}
where $\|\widehat{\vect{\theta}}_t - \vect{\theta}^*\|_{\vect{V}_{t}} \leq \alpha_t$ follows from the assumption that $\mathcal{E}_t$ holds. Hence, we obtain
\begin{align*}
    \langle \widehat{\vect{\psi}}_t, \vect{x}_{ti} \rangle - \langle \widehat{\vect{\phi}}_t, \vect{x}_{ti} \rangle \cdot p - g_{ti}(p) \leq u_{ti}(p) \leq \langle \widehat{\vect{\psi}}_t, \vect{x}_{ti} \rangle - \langle \widehat{\vect{\phi}}_t, \vect{x}_{ti} \rangle \cdot p + g_{ti}(p),
\end{align*}
showing that $u_{ti}(p) \leq \wt{h}_{ti}(p) \leq u_{ti}(p) + 2 g_{ti}(p)$ for all $p \in \bbR$.

Since $u'_{ti}(p) \leq -L_0$ for all $p \in \bbR$, we also have $u_{ti}(p) \leq u_{ti}(p') - L_0 (p - p') \leq \wt{h}_{ti} (p') - L_0 (p - p')$ for any $p' \leq p$. Therefore, $u_{ti}(p) \leq h_{ti}(p)$ for all $p \in \bbR$ proving condition \eqref{eq_ub_condition_1}.

On the other hand, we have $h_{ti}(p) \leq \wt{h}_{ti} (p) \leq u_{ti}(p) + 2 g_{ti}(p)$ for all $p \in \bbR$ proving \eqref{eq_ub_condition_2}.

Furthermore, $h_{ti}(0) \leq \wt{h}_{ti} (0) \leq u_{ti}(0) + 2 g_{ti}(0) \leq 1 + \alpha_t \|(\vect{x}_{ti}, \vect{0})\|_{\vect{V}_{t}^{-1}} \leq 1 + \alpha_t / \sqrt{\lambda_{\mathrm{min}}^0} \leq 1 + \mathcal{O} \left(\frac{1}{\bar{P} \sqrt{\log T}} \right)$. As a result, $h_{ti}(0) \leq 2$ for sufficiently large $T$.

Next, we show that $h_{ti}(p)$ is a differentiable function and its derivative is at most $- L_0$ for all $p \in \bbR$. Notice that the function $\wt{h}_{ti} (p)$ can be written as $\wt{h}_{ti} (p) = y(p) + c \sqrt{z(p)}$ for a linear function $y : \bbR \to \bbR$ and a positive quadratic function $z : \bbR \to \bbR_{+}$ of the form $z(p) = a + b p + p^2$ satisfying $4a - b^2 > 0$. With this notation, the second derivative of $\wt{h}_{ti} (p)$ is given as
\begin{align*}
\wt{h}_{ti}'' (p) = \frac{4 a - b^2}{4 (a + p (b + p))^{3/2}} > 0.
\end{align*}

Therefore, $\wt{h}_{ti} (p)$ is smooth and strictly convex. Let $p_0$ be the unique value such that $\wt{h}_{ti}' (p_0) = - L_0$.

We let $p^\dagger$ denote the value of $p'$ that minimizes the function $\wt{h}_{ti} (p') - L_0 (p - p')$ over $(-\infty, p]$. As a result, we obtain $h_{ti}(p) = \wt{h}_{ti} (p^\dagger) - L_0 (p - p^\dagger)$. Using that the function $\wt{h}_{ti} (p') - L_0 (p - p')$ is convex, we can write
\begin{align*}
    p^\dagger =
    \begin{cases}
        p_0 \quad &\text{if } p_0 \leq p,\\
        p \quad &\text{if } p < p_0.
    \end{cases}
\end{align*}

Consequently, we obtain
\begin{align*}
    h_{ti}(p) =
    \begin{cases}
        \wt{h}_{ti} (p_0) - L_0 (p - p_0) \quad &\text{if } p \geq p_0,\\
        \wt{h}_{ti} (p) \quad &\text{if } p < p_0.
    \end{cases}
\end{align*}

The function $h_{ti}(p)$ is differentiable everywhere including $p = p_0$ since $\wt{h}_{ti}' (p_0) = -L_0$. Furthermore, $h_{ti}'(p) \leq -L_0$ for all $p \geq 0$ since $\wt{h}_{ti}' (p) \leq - L_0$ for $p < p_0$. Consequently, we prove property \eqref{eq_ub_condition_3}.

\end{proof}

\lemmarevenueub*

\begin{proof}

    \emph{Inequality (a): }Fix some $t$ and define revenue functions $R^A : 2^{[N]} \times \bbR^N_+ \to \bbR$ given by
    \begin{align*}
        R^{A}(S, \vect{p}) = \frac{\sum_{i \in S \setminus A} p_{i} \exp(u_{ti}(p_{i})) + \sum_{i \in S \cap A} p_{i} \exp(h_{ti}(p_{i}))}{1 + \sum_{i \in S \setminus A} \exp(u_{ti}(p_{i})) + \sum_{i \in S \cap A } \exp(h_{ti}(p_{i}))}
    \end{align*}
    for any $A \subseteq [N]$. Note that this definition leads to $R^{\emptyset}(S, \vect{p}) = R_t(S, \vect{p})$ and $R^{S}(S, \vect{p}) = \widetilde{R}_t(S, \vect{p})$. We also define 
    \begin{align*}
        (S^{A}, \vect{p}^{A}) = \argmax_{\substack{S \subseteq [N] : |S| \leq K \\ \vect{p} \in \bbR^N_+}} R^{A}(S, \vect{p}).
    \end{align*}
    which satisfies $(S^{\emptyset}, \vect{p}^{\emptyset}) = (S^*_t, \vect{p}^*_t)$ and $(S^{[N]}, \vect{p}^{[N]}) = (S_t, \vect{p}_t)$.

    By the optimality of $(S^{A}, \vect{p}^{A})$ for any revenue function $R^A$, we have $p^{A}_j \geq R^{A}(S^{A}, \vect{p}^{A})$ for all $j \in S^{A}$. We can write this inequality as $p^{A}_j \geq a / b$ where
    \begin{align*}
        a &= \sum_{i \in S^A \setminus A} p^A_{i} \exp(u_{ti}(p^A_{i})) + \sum_{i \in S^A \cap A} p^A_{i} \exp(h_{ti}(p^A_{i})) \quad \text{and}\\
        b &= 1 + \sum_{i \in S^A \setminus A} \exp(u_{ti}(p^A_{i})) + \sum_{i \in S^A \cap A } \exp(h_{ti}(p^A_{i})).
    \end{align*}
    Letting $\delta =  \exp(h_{tj}(p^A_{j})) - \exp(u_{tj}(p^A_{j}))$, we have $a b + b  \delta p^{A}_j \geq  a b + a \delta$ which implies
    \begin{align*}
        \frac{a + p^{A}_j \delta}{b + \delta} \geq \frac{a}{b}.
    \end{align*}
    Hence, we have $R^{A \cup \{j\}}(S^{A}, \vect{p}^{A}) \geq R^{A}(S^{A}, \vect{p}^{A})$ for all $j \in S^{A}$. 
    
    We also have $R^{A \cup \{j\}}(S^{A}, \vect{p}^{A}) = R^{A}(S^{A}, \vect{p}^{A})$ for any $j \notin S^A$. Therefore, $R^{A \cup \{j\}}(S^{A}, \vect{p}^{A}) \geq R^{A}(S^{A}, \vect{p}^{A})$ for any $j \in [N]$. Using the optimality of $(S^{A \cup \{j\}}, \vect{p}^{A \cup \{j\}})$ for function $R^{A \cup \{j\}}$, we can write
    \begin{align*}
        R^{A \cup \{j\}}(S^{A \cup \{j\}}, \vect{p}^{A \cup \{j\}}) \geq R^{A}(S^{A}, \vect{p}^{A})
    \end{align*}
    for any $j \in [N]$. Therefore, by induction, we can show that
    \begin{align*}
        \widetilde{R}_t(S_t, \vect{p}_t) = R^{[N]}(S^{[N]}, \vect{p}^{[N]}) \geq R^{\emptyset}(S^\emptyset, \vect{p}^\emptyset) = R_t(S_t^*, \vect{p}_t^*).
    \end{align*}    
    
\emph{Inequality (b): }Let $u_{ti} := u_{ti}(p_{ti})$ and $h_{ti} := h_{ti}(p_{ti})$ with $2 g_{ti}(p_{ti}) \geq h_{ti} - u_{ti} \geq 0$. By the mean value theorem, for any $i$, there exists $z_{ti} := (1-c) u_{ti} + c h_{ti}$ for some $c \in (0,1)$ such that
\begin{align*}
    \widetilde{R}_t(S_t, \vect{p}_t) - R_t(S_t, \vect{p}_t) &= \frac{\sum_{i \in S_t} p_{ti} \exp(h_{ti})}{1 + \sum_{j \in S_t} \exp(h_{tj})} - \frac{\sum_{i \in S_t} p_{ti} \exp(u_{ti})}{1 + \sum_{j \in S_t} \exp(u_{tj})} \\
    &= \frac{(\sum_{i \in S_t} p_{ti} \exp(z_{ti}) (h_{ti} - u_{ti}))(1 + \sum_{i \in S_t} \exp(z_{ti}))}{(1 + \sum_{i \in S_t} \exp(z_{ti}))^2} \\
    & \qquad - \frac{(\sum_{i \in S_t} p_{ti} \exp(z_{ti}))(\sum_{i \in S_t} \exp(z_{ti}) (h_{ti} - u_{ti}))}{(1 + \sum_{i \in S_t} \exp(z_{ti}))^2}\\
    &= \sum_{i \in S_t} p_{ti} q_{t}(i | \vect{z}_{t}) (h_{ti} - u_{ti}) \\
    &\qquad - \left(\sum_{i \in S_t} p_{ti} q_{t}(i | \vect{z}_{t}) \right) \left(\sum_{i \in S_t} q_{t}(i | \vect{z}_{t}) (h_{ti} - u_{ti})\right)\\
    &= \sum_{i \in S_t} \left( p_{ti} - \sum_{i \in S_t} p_{ti} q_{t}(i | \vect{z}_{t}) \right) q_{t}(i | \vect{z}_{t}) (h_{ti} - u_{ti})\\
    &\leq P \cdot \sum_{i \in S_t} q_{t}(i | \vect{z}_{t}) (h_{ti} - u_{ti})\\
    &\leq 2 P \cdot \sum_{i \in S_t} q_{t}(i | \vect{z}_{t}) g_{ti}(p_{ti})
\end{align*}
    where the first inequality follows from $|p_{ti}| \leq P$ and $q_{t}(i | \vect{z}_{t})$ is a categorical distribution given by
    \begin{align*}
        q_{t}(i | \vect{z}_{t}) = \frac{\exp(z_{ti})}{1 + \sum_{j \in S_t} \exp(z_{tj})} 
    \end{align*}
    for $i \in S_t$. Then, noting that
    \begin{align*}
        g_{ti}(p_{ti}) &= \alpha_t \|(\vect{x}_{ti}, -p\vect{x}_{ti})\|_{\vect{V}_{t}^{-1}} \\
        &\leq \alpha_t \frac{1}{\sqrt{\lambda_{\mathrm{min}}(\vect{V}_{T_0})}} \|(\vect{x}_{ti}, -p\vect{x}_{ti})\|_2\\
        &\leq \alpha_t \frac{1}{\sqrt{\lambda_{\mathrm{min}}(\vect{V}_{T_0})}} (1+P)\\
        &\leq \frac{\log 2}{2},
    \end{align*} 
    we have $u_{ti} \leq z_{ti} \leq u_{ti} + \log 2$. Hence,
    \begin{align*}
        \frac{1}{2} \leq \frac{q_{t}(i | \vect{z}_{t})}{q_{ti}(\vect{\theta}^*)} \leq 2
    \end{align*}
    for all $i \in S_t$. Consequently, we obtain
    \begin{align*}
    \widetilde{R}_t(S_t, \vect{p}_t) - R_t(S_t, \vect{p}_t) \leq 4 P \cdot \sum_{i \in S_t} q_{ti}(\vect{\theta}^*) g_{ti}(p_{ti}),
    \end{align*}
    completing the proof.
\end{proof}

\begin{lemma}
    For $t > T_0$, $\det (\vect{V}_t)$ is increasing with respect to $t$ and $\det (\vect{V}_t) \leq \left( t \widebar{P}^2 / d \right)^{2d}$.
    \label{lemma_det_ub}
\end{lemma}

\begin{proof}
Let $\lambda_1, \dots, \lambda_{2d}$ be the eigenvalues of $\vect{V}_t$. Then, using the AM-GM inequality we can write
\begin{align*}
    \det (\vect{V}_t) &= \prod_{i = 1}^{2d} \lambda_i\\
    &\leq \left( \frac{\sum_{i = 1}^{2d} \lambda_i}{2d} \right)^{2d}\\
    &= \left( \frac{\mathrm{trace}(\vect{V}_t)}{2d} \right)^{2d}\\
    &\leq \left( \frac{\sum_{s = 1}^{T_0} \sum_{i \in S_s} \frac{1}{K^2} \|\wt{\vect{x}}_{s i}\|_2^2 + 2 \sum_{s = T_0 + 1}^{t} \sum_{i \in S_s} q_{si}(\widehat{\vect{\theta}}_s) \|\wt{\vect{x}}_{s i}\|_2^2 }{2d} \right)^{2d}\\
    &\leq \left( \frac{t \widebar{P}^2}{d} \right)^{2d}.
\end{align*}

\end{proof}

\begin{lemma}
If good events $\mathcal{E}_0$ and $\mathcal{E}_t$ hold for all $t \geq T_0$, then 
\begin{align*}
\sum_{t = T_0}^{T} \sum_{i \in S_t} q_{ti}(\wh{\vect{\theta}}_t)\|\widetilde{\vect{x}}_{ti}\|^2_{\vect{V}_t^{-1}} \leq 18 K \log \left( \frac{\det(\vect{V}_{T+1})}{\det(\vect{V}_{T_0})} \right).
\end{align*}
\label{lemma_det_lb}
\end{lemma}

\begin{proof}
Let $\lambda_1, \dots, \lambda_{2d}$ be the eigenvalues of $H_t(\wh{\vect{\theta}}_t) = \sum_{i \in S_t} q_{ti}(\wh{\vect{\theta}}_t) \widetilde{\vect{x}}_{ti} \widetilde{\vect{x}}_{ti}^\top - \sum_{i \in S_t} \sum_{j \in S_t} q_{ti}(\wh{\vect{\theta}}_t) q_{tj}(\wh{\vect{\theta}}_t) \widetilde{\vect{x}}_{ti} \widetilde{\vect{x}}_{tj}^\top$. Since $H_t(\wh{\vect{\theta}}_t)$ is positive semi-definite, $\lambda_j \geq 0$ for all $j$. Then, we have
\begin{align*}
    \det \left( \vect{I} + \vect{V}_{t}^{-1/2}  H_t(\wh{\vect{\theta}}_t) \vect{V}_{t}^{-1/2} \right) &= \prod_{i = 1}^{2d} (1 + \lambda_j)\\
    &\geq 1 + \sum_{i = 1}^{2d} \lambda_j \\
    &= 1 - 2d  + \sum_{i = 1}^{2d} (1 + \lambda_j)\\
    &= 1 - 2d + \tr \left( \vect{I} + \vect{V}_{t}^{-1/2}  H_t(\wh{\vect{\theta}}_t) \vect{V}_{t}^{-1/2} \right)\\
    &= 1 +  \sum_{i \in S_t} q_{ti}(\wh{\vect{\theta}}_t) \|\widetilde{\vect{x}}_{ti} \|_{\vect{V}_{t}^{-1}}^2 - \sum_{i \in S_t} \sum_{j \in S_t} q_{ti}(\wh{\vect{\theta}}_t) q_{tj}(\wh{\vect{\theta}}_t) \widetilde{\vect{x}}_{tj}^\top \vect{V}_{t}^{-1} \widetilde{\vect{x}}_{ti}\\
    &\geq 1 + \sum_{i \in S_t} q_{ti}(\wh{\vect{\theta}}_t) q_{t0}(\wh{\vect{\theta}}_t) \|\widetilde{\vect{x}}_{ti} \|_{\vect{V}_{t}^{-1}}^2
\end{align*}
using the inequality $\widetilde{\vect{x}}_{tj}^\top \vect{V}_{t}^{-1} \widetilde{\vect{x}}_{ti} + \widetilde{\vect{x}}_{ti}^\top \vect{V}_{t}^{-1} \widetilde{\vect{x}}_{tj} \leq \widetilde{\vect{x}}_{ti}^\top \vect{V}_{t}^{-1} \widetilde{\vect{x}}_{ti} + \widetilde{\vect{x}}_{tj}^\top \vect{V}_{t}^{-1} \widetilde{\vect{x}}_{tj} = \|\widetilde{\vect{x}}_{ti} \|_{\vect{V}_{t}^{-1}}^2 + \|\widetilde{\vect{x}}_{tj} \|_{\vect{V}_{t}^{-1}}^2$.

Now, to lower bound $\det(\vect{V}_{T+1})$, we write
\begin{align*}
    \det (\vect{V}_{T+1}) &= \det \left( \vect{V}_{T} + H_{T}(\wh{\vect{\theta}}_{T}) \right)\\
    &= \det (\vect{V}_{T}) \det \left( \vect{I} + \vect{V}_{T}^{-1/2}  H_{T}(\wh{\vect{\theta}}_{T}) \vect{V}_{T}^{-1/2} \right)\\
    &\geq \det (\vect{V}_{T_0}) \prod_{t = T_0}^{T} \left( 1 +  \sum_{i \in S_t} q_{ti}(\wh{\vect{\theta}}_t) q_{t0}(\wh{\vect{\theta}}_t) \|\widetilde{\vect{x}}_{ti}\|^2_{\vect{V}_{t}^{-1}} \right).
\end{align*}

Now, using that $\lambda_{\mathrm{min}}(\vect{V}_{T_0}) \geq \widebar{P}^2$ which is satisfied under event $\mathcal{E}_0$, we have
\begin{align*}
    \|\widetilde{\vect{x}}_{ti}\|^2_{\vect{V}_{t}^{-1}} \leq \frac{\|\widetilde{\vect{x}}_{ti}\|^2}{\lambda_{\mathrm{min}}(\vect{V}_{t})} \leq \frac{(1+P^2)}{\lambda_{\mathrm{min}}(\vect{V}_{t})} \leq \frac{\widebar{P}^2}{\widebar{P}^2} = 1.
\end{align*}

Hence, $\sum_{i \in S_t} q_{ti}(\wh{\vect{\theta}}_t) q_{t0}(\wh{\vect{\theta}}_t) \|\widetilde{\vect{x}}_{ti}\|^2_{V_{t}^{-1}} \leq 1$ for all $t \geq T_0$. Then, using the fact that $z \leq 2 \log(1+z)$ for any $z \in [0,1]$,
\begin{align*}
    \sum_{t = T_0}^{T} \sum_{i \in S_t} q_{ti}(\wh{\vect{\theta}}_t) q_{t0}(\wh{\vect{\theta}}_t) \|\widetilde{\vect{x}}_{ti}\|^2_{\vect{V}_{t}^{-1}} &\leq 2 \sum_{t = T_0}^{T} \log \left( 1 + \sum_{i \in S_t} q_{ti}(\wh{\vect{\theta}}_t) q_{t0}(\wh{\vect{\theta}}_t) \|\widetilde{\vect{x}}_{ti}\|^2_{\vect{V}_{t}^{-1}} \right)\\
    &= 2 \log \prod_{t = T_0}^{T} \left( 1 + \sum_{i \in S_t} q_{ti}(\wh{\vect{\theta}}_t) q_{t0}(\wh{\vect{\theta}}_t) \|\widetilde{\vect{x}}_{ti}\|^2_{\vect{V}_{t}^{-1}} \right)\\
    &\leq 2 \log \left( \frac{\det(\vect{V}_{T+1})}{\det(\vect{V}_{T_0})} \right)
\end{align*}

Note that we have $p_{ti} \geq 0$ for all $i \in S_t$ and all $t \geq T_0$. Furthermore, for any $\vect{\wh{\theta}}_t \in \mathcal{B}_\gamma$, we have $\langle \wh{\vect{\phi}}_t, \vect{x}_{tj} \rangle \geq L_0 - \gamma \geq 0$ and $\langle \wh{\vect{\psi}}_t, \vect{x}_{tj} \rangle \leq 1 + \gamma \leq 2$. Hence, we can lower bound $q_{t0}(\wh{\vect{\theta}}_t)$ as
\begin{align*}
    q_{t0}(\wh{\vect{\theta}}_t) = \frac{1}{ 1 + \sum_{j \in S_t} \exp (\langle \wh{\vect{\psi}}_t, \vect{x}_{tj} \rangle - \langle \wh{\vect{\phi}}_t, \vect{x}_{tj} \rangle p_{tj}) } > \frac{1}{(1 + K e^2)} \geq \frac{1}{9K}.
\end{align*}

Combining this result with the previous inequality, we show the intended result.

\end{proof}

\begin{lemma}
If good events $\mathcal{E}_0$ and $\mathcal{E}_t$ hold for all $t \geq T_0$, then
    \begin{align*}
        \sum_{t = T_0}^{T} \sum_{i \in S_t} q_{ti}(\wh{\vect{\theta}}_t) \|\widetilde{\vect{x}}_{ti}\|^2_{\vect{V}_t^{-1}} \leq 36 d K \log ( T / d ).
    \end{align*}
\label{lemma_cumulative_max_uncertainty}
\end{lemma}

\begin{proof}
    Combining Lemma~\ref{lemma_det_ub} and Lemma~\ref{lemma_det_lb}, we obtain
\begin{align*}
    \sum_{t = T_0}^{T} \sum_{i \in S_t} q_{ti}(\wh{\vect{\theta}}_t) \|\widetilde{\vect{x}}_{ti}\|^2_{\vect{V}_t^{-1}} \leq 18 K \log \left( \frac{\det(\vect{V}_{T})}{\det(\vect{V}_{T_0})} \right) &\leq  18 K \log \left( \frac{T \widebar{P}^2}{d \lambda_{\mathrm{min}}(\vect{V}_{T_0})} \right)^{2d} \\
    &\leq 36 d K \log ( T / d )
\end{align*}
where the last inequality is by $\lambda_{\mathrm{min}}(\vect{V}_{T_0}) \geq \widebar{P}^2$ which is satisfied under event $\mathcal{E}_0$.
\end{proof}

\begin{lemma}
    For any $\vect{\theta}_1, \vect{\theta}_2 \in \mathcal{B}_{\gamma}$, we have
    \begin{align*}
        \frac{1}{\sqrt{2}} \leq \frac{q_{ti}(\vect{\theta}_1)}{q_{ti}(\vect{\theta}_2)} \leq \sqrt{2}
    \end{align*}
    for all $i \in S_t$.
\label{lemma_qti_ratio}
\end{lemma}

\begin{proof}
Let $z_{ti}^1 = \exp (\langle \vect{\psi}_1, \vect{x}_{ti} \rangle - \langle \vect{\phi}_1, \vect{x}_{ti} \rangle p_{ti})$ and $z_{ti}^2 = \exp (\langle \vect{\psi}_2, \vect{x}_{ti} \rangle - \langle \vect{\phi}_2, \vect{x}_{ti} \rangle p_{ti})$ for all $i \in S_t$. Then, we have
\begin{align*}
    \frac{z_{ti}^1}{z_{ti}^2} &= \exp (\langle \vect{\psi}_1 - \vect{\psi}_2, \vect{x}_{ti} \rangle - \langle \vect{\phi}_1 - \vect{\phi}_2, \vect{x}_{ti} \rangle p_{ti})\\
    &= \exp (\langle \vect{\theta}_1 - \vect{\theta}_2, \wt{\vect{x}}_{ti} \rangle).
\end{align*}

Therefore, we have $e^{- 2 \gamma \widebar{P}} = 1/\sqrt[4]{2} \leq z_{ti}^1 / z_{ti}^2 \leq \sqrt[4]{2} = e^{2 \gamma \widebar{P}}$. On the other hand, we have
\begin{align*}
    \frac{q_{ti}(\vect{\theta}_1)}{q_{ti}(\vect{\theta}_2)} &= \frac{z_{ti}^1}{ 1 + \sum_{j \in S_t} z_{tj}^1 } \cdot \frac{ 1 + \sum_{j \in S_t} z_{tj}^2 }{z_{ti}^2}
\end{align*}

Now, we note that for any two sets of positive numbers $\{a_i\}_{i \in S}$ and $\{b_i\}_{i \in S}$ such that $1 / c \leq a_i / b_i \leq c$ for some $c > 1$, we have $1 / c \leq (\sum_{i \in S} a_i)/(\sum_{i \in S} b_i) \leq c$. Using this result, we complete the proof.




\end{proof}

\begin{lemma} 
    If the number of initialization rounds satisfies 
    \begin{align*}
        T_0 \geq \left( \frac{C_3 \sqrt{d} + C_4 \sqrt{\log T}}{\sigma_0} \right)^2
    \end{align*}
    for some universal constants $C_3 = \sqrt{2} C_1$ and $C_4 = \max\{C_2, 10\}$, then
    \begin{align*}
         \|G_{T_0} (\widehat{\vect{\theta}}_{T_0})\|_{\vect{V}_{T_0}^{-1}} = \left\| \sum_{t = 1}^{T_0} \sum_{i \in S_{t}} \epsilon_{t i} \vect{\widetilde{x}}_{t i} \right\|_{\vect{V}_{T_0}^{-1}} \leq \frac{C_5}{\sigma_0} \log(T).
    \end{align*}
    with probability at least $1 - \mathcal{O} (T^{-2})$ for some universal constant $C_5 = 192$.
    \label{lemma_G_T0_ub}
\end{lemma}

\begin{proof}
    We have
    \begin{align*}
        \|G_{T_0} (\widehat{\vect{\theta}}_{T_0})\|_{\vect{V}_{T_0}^{-1}}^2 \leq \frac{1}{\lambda_{\mathrm{min}}(\vect{V}_{T_0})} \|G_{T_0} (\widehat{\vect{\theta}}_{T_0})\|_2^2.
    \end{align*}

    Let $\vect{z}_t := \sum_{i \in S_{t}} \epsilon_{t i} \vect{\widetilde{x}}_{t i}$ and recall the definition $\epsilon_{t i} = y_{ti} - q_{ti}(\vect{\theta}^*)$, to write
    \begin{align*}
        \vect{z}_t = \sum_{i \in S_{t}} \epsilon_{t i} \vect{\widetilde{x}}_{t i} = \sum_{i \in S_{t}} y_{ti} \vect{\widetilde{x}}_{t i} - \sum_{i \in S_{t}} q_{ti}(\vect{\theta}^*) \vect{\widetilde{x}}_{t i}.
    \end{align*}

    Furthermore, note that  $\| \vect{\widetilde{x}}_{t i} \| \leq 3$ for all $t \leq T_0$ since $1 \leq p_{ti} \leq 2$ for all $t \leq T_0$. Therefore, $\|\vect{z}_t\| \leq \| \sum_{i \in S_{t}} y_{ti} \vect{\widetilde{x}}_{t i} \| +  \|\sum_{i \in S_{t}} q_{ti}(\vect{\theta}^*) \vect{\widetilde{x}}_{t i}\| \leq 6$.

    Then, using the vector Bernstein inequality from \cite{Kohler_Lucchi_2017}, we have
    \begin{align*}
        \|G_{T_0} (\widehat{\vect{\theta}}_{T_0})\|_2^2 \leq 48 T_0 \log(2 T^2)
    \end{align*}
    with probability at least $1 - T^{-2}$ given that $T_0 > 24 \log(2 T^2)$. Note that we satisfy $T_0 > 24 \log(2 T^2)$ with the condition given for $T_0$ in the statement of the lemma because $\sigma_0 \leq 1$.

    On the other hand, by Lemma \ref{lemma_initalization_min_eigenvalue}, we have
    \begin{align*}
        \lambda_{\mathrm{min}}(\vect{V}_{T_0}) &\geq \frac{\sigma_0}{2} \left(T_0 - \left( \frac{C_1 \sqrt{d} + C_2 \sqrt{\log T}}{\sigma_0} \right)^2 \right)\\
        &\geq \frac{\sigma_0 T_0}{4} .
    \end{align*}
    with probability at least $1 - T^{-2}$. After combining these inequalities, we obtain the intended result.
\end{proof}

\subsection{Estimating Minimum Price Sensitivity }\label{sect:estimate_L0}

Assume the time horizon is large enough so that $\frac{1}{L_0} < \frac{1}{4} T^{1/4} \sqrt{\frac{\sigma_0}{K}}$. Then, instead of setting the parameters for the initialization round using the true value of $\widebar{P}$ (which requires knowing $L_0$), we set the target minimum eigenvalue as $\lambda_{\mathrm{min}}^0 = \Theta(T^{1/2})$. As a result, we can use the initialization result established in Lemma \ref{lemma_initalization_min_eigenvalue} to translate this target to an initialization period of length
\begin{align*}
T_0 = \Theta \left( \frac{\lambda_{\mathrm{min}}^0 K}{\sigma_0} \right) = \Theta\left( \sqrt{T} \right).
\end{align*}

As the consistency results established in Lemma \ref{lemma_consistency} shows, the MLE estimate at time $T_0$ satisfies
\begin{align*}
    \| \wh{\vect{\theta}}_{T_0} - \vect{\theta}^* \|_2 \leq T^{-1/4} \sqrt{\frac{K}{\sigma_0}} < \frac{L_0}{4}.
\end{align*}
with probability $1 - \mathcal{O}(T^{-1})$.

Consequently, we have $\| \wh{\vect{\phi}}_{T_0} - \vect{\phi}^* \| \leq L_0/4$. Then, we can write
\begin{align*}
    \langle \wh{\vect{\phi}}_{T_0}, \vect{x}_{ti} \rangle &= \langle \vect{\phi}^*, \vect{x}_{ti} \rangle + \langle \wh{\vect{\phi}}_{T_0} -\vect{\phi}^*, \vect{x}_{ti} \rangle\\
    &\geq L_0 - \| \wh{\vect{\phi}}_{T_0} - \vect{\phi}^* \| \|\vect{x}_{ti}\|\\
    &\geq L_0 - \frac{L_0}{4}\\
    &\geq \frac{3 L_0}{4}.
\end{align*}

Taking the minimum over all previous iterations, we have
\begin{align*}
    \min_{t \in [T_0], i \in [N]} \langle \vect{\phi}^*, \vect{x}_{ti} \rangle &= \min_{t \in [T_0], i \in [N]} \left\{ \langle \wh{\vect{\phi}}_{T_0} -\vect{\phi}^*, \vect{x}_{ti} \rangle + \langle \wh{\vect{\phi}}_{T_0}, \vect{x}_{ti} \rangle \right\}\\
    &= \min_{t \in [T_0], i \in [N]} \langle \wh{\vect{\phi}}_{T_0} -\vect{\phi}^*, \vect{x}_{ti} \rangle  + \min_{t \in [T_0], i \in [N]} \langle \wh{\vect{\phi}}_{T_0}, \vect{x}_{ti} \rangle \\
    &\geq - T^{-1/4} \sqrt{\frac{K}{\sigma_0}} + \min_{t \in [T_0], i \in [N]} \langle \wh{\vect{\phi}}_{T_0}, \vect{x}_{ti} \rangle. 
\end{align*}

Then, we estimate the minimum sensitivity parameter as
\begin{align*}
    \wh{L_0} = \min_{t \in [T_0], i \in [N]} \langle \wh{\vect{\phi}}_{T_0}, \vect{x}_{ti} \rangle - T^{-1/4} \sqrt{\frac{K}{\sigma_0}}.
\end{align*}

Note that $\wh{L_0}$ satisfies $ \min_{t \in [T_0], i \in [N]} \langle \vect{\phi}^*, \vect{x}_{ti} \rangle \geq \wh{L_0}$. Furthermore, $\wh{L_0} \geq - \frac{L_0}{4} + \frac{3 L_0}{4} = \frac{L_0}{2} > 0$.

The next step is to upper bound the expected number of rounds in which there is a context vector $\vect{x}_{ti}$ such that $\langle \vect{\phi}^*, \vect{x}_{ti} \rangle \leq \wh{L_0}$. Note that the contexts are sampled independently from an identical distribution by our assumption. Therefore, each ordering among $\langle \vect{\phi}^*, \vect{x}_{ti} \rangle$ values for $t \in [T]$ and $i \in [N]$ is equally likely. As a result, the expected number of time-item index pairs $(t, i) \geq T_0 \times [N]$ for which $\langle \vect{\phi}^*, \vect{x}_{ti} \rangle \leq \wh{L_0}$ can be upper bounded as
\begin{align*}
    \sum_{k = 0}^{\infty} k \left( 1 - \frac{N \sqrt{T}}{N T} \right)^k \left( \frac{N \sqrt{T}}{N T} \right) = \sum_{k = 0}^{\infty} k \left( 1 - \frac{1}{\sqrt{T}} \right)^k \left( \frac{1}{\sqrt{T}} \right) \leq \sqrt{T}.
\end{align*}

Consequently, the minimum sensitivity parameter estimate $\wh{L_0}$ fails only in $\sqrt{T}$ rounds and causes additional $\widebar{P} \sqrt{T} = \mathcal{O}(\log K \sqrt{T} / L_0)$ regret. In total, this algorithm still manages to achieve an asymptotic regret rate of $\mathcal{O}(d \sqrt{K T} / L_0)$.

\section{Self-Normalized Bounds for Vector-Valued Martingales}

\begin{theorem}
    Let $\{\mathcal{F}_t\}_{t = 1}^{\infty}$ be a filtration. Let $\{\vect{X}_t\}_{t = 1}^{\infty}$ be a stochastic process such that $\vect{X}_t \in \bbR^{d \times K}$ is $\mathcal{F}_t$ measurable and the columns of $\vect{X}_t$ denoted by $\vect{x}_{ti}$ satisfy $\|\vect{x}_{ti}\| \leq B$ almost surely for some $B > 0$. Let $\{\vect{\epsilon}_t\}_{t = 1}^{\infty}$ be a martingale difference process such that $\vect{\epsilon}_t \in \bbR^K$ is $\mathcal{F}_{t+1}$-measurable. Let $\vect{H}_0 \in \bbR^{d \times d}$ such that $\lambda_{\mathrm{min}}(\vect{H}_0) > \lambda$ for some $\lambda > 0$. Furthermore, assume that we have $\|\vect{\epsilon}_t\|_1 \leq 2$ almost surely conditional on $\mathcal{F}_t$ and the conditional covariance is given by $\Sigma_t := \bbE[\vect{\epsilon}_t \vect{\epsilon}_t^\top | \mathcal{F}_t]$. For any $t \geq 1$ define
    \begin{align*}
        \vect{S}_t = \sum_{s = 1}^{t-1} \vect{X}_{s} \vect{\epsilon}_{s} \qquad \text{and} \qquad \vect{H}_t = \vect{H}_0 + \sum_{s = 1}^{t-1} \vect{X}_{s} \Sigma_s \vect{X}_{s}^\top.
    \end{align*}
    Then, for any $\delta \in (0, 1]$, with probability at least $1 - \delta$, for all $t \geq 1$, we have
    \begin{align*}
        \|\vect{S}_t\|_{\vect{H}_t^{-1}} \leq \frac{\sqrt{\lambda}}{4} + \frac{4}{\sqrt{\lambda}} \log \left( 
        \frac{\det (\vect{H}_t)^{1/2}}{\delta \lambda^{d/2}}\right) +  \frac{4}{\sqrt{\lambda}} d \log(2).
    \end{align*}
\label{theorem_martingale}
\end{theorem}
    \begin{proof}
        Let $\wt{\vect{H}}_t = \sum_{s = 1}^{t-1} \vect{X}_s \Sigma_s \vect{X}_s^\top$ and define the function
        \begin{align*}
            M_t(\vect{\xi}) = \exp(\vect{\xi}^\top \vect{S}_t - \|\vect{\xi}\|^2_{\wt{\vect{H}}_t} ),
        \end{align*}
        for any $t \geq 1$ and $\vect{\xi} \in \bbR^{d}$. For $t = 0$, let $M_0(\vect{\xi}) = 0$.

        By Lemma \ref{lemma_supermartingale}, we can show that $\{M_t(\vect{\xi})\}_{t = 1}^{\infty}$ is a non-negative super-martingale for any $\|\vect{\xi}\|_2 \in \frac{1}{2B} \mathcal{B}_2(d)$. Then, we let $h(\vect{\xi})$ be a probability density with support on $\frac{1}{2B} \mathcal{B}_2(d)$ and define
        \begin{align*}
            \wt{M}_t = \int_{\vect{\xi}} M_t(\vect{\xi}) dh(\vect{\xi})
        \end{align*}
        for all $t \geq 1$. Lemma 20.3 of \cite{Lattimore_Szepesvari_2020} shows that $\wt{M}_t$ is also a non-negative super-martingale and $\bbE[\wt{M}_0] = 1$. Then, the maximal inequality (Theorem 3.9 of \cite{Lattimore_Szepesvari_2020}) shows that
        \begin{align}
            \Pr \left[ \sup_{t \geq 0} \log(\wt{M}_t) \geq \log (1/\delta) \right] \leq \delta.
        \label{logmt_ub}
        \end{align}

        Next, let $h(\vect{\xi})$ be the density of a normal distribution with the precision matrix $2 \vect{H}_0$ truncated on $\frac{1}{2B} \mathcal{B}_2(d)$ and $N(h)$ be its normalization constant. Then, we can show that
        \begin{align*}
            \wt{M}_t = \frac{1}{N(h)} \int_{\frac{1}{2B} \mathcal{B}_2(d)} \exp(\vect{\xi}^\top \vect{S}_t - \|\vect{\xi}\|^2_{\vect{H}_t} ) d \vect{\xi}.
        \end{align*}

        Additionally, let $g(\vect{\xi})$ be the density of a normal distribution with the precision matrix $2 \vect{H}_t$ truncated on $\frac{1}{4B} \mathcal{B}_2(d)$ and $N(g)$ be its normalization constant. Following the arguments in the proof of Theorem 1 in \cite{Faury_Abeille_Calauzenes_Fercoq_2020}, for any $t \geq 1$, one can show that
        \begin{align*}
            \wt{M}_t \geq \exp(\vect{\xi}^\top \vect{S}_t - \|\vect{\xi}\|^2_{\vect{H}_t}) \cdot \frac{N(g)}{N(h)}
        \end{align*}
        for any $\vect{\xi} \in \frac{1}{4B} \mathcal{B}_2(d)$. Let $\vect{\xi}_0 = \frac{\vect{H}_t^{-1} \vect{S}_t}{\|\vect{S}_t\|_{\vect{H}_t^{-1}} } \cdot \frac{\sqrt{\lambda}}{4}$ which satisfies $\|\vect{\xi}_0\| \leq 1/4$. Then, we can write
        \begin{align}
            \log(\wt{M}_t) \geq \vect{\xi}_0^\top \vect{S}_t - \|\vect{\xi}_0\|^2_{\vect{H}_t} + \log \left( \frac{N(g)}{N(h)} \right) = \frac{\sqrt{\lambda}}{4} \|\vect{S}_t\|_{\vect{H}_t^{-1}} - \frac{\lambda}{16} + \log \left( \frac{N(g)}{N(h)} \right).
        \label{logmt_lb}
        \end{align}

        Combining \eqref{logmt_ub} and \eqref{logmt_lb}, for any $t \geq 1$, we have
        \begin{align*}
            \Pr \left[ \|\vect{S}_t\|_{\vect{H}_t^{-1}} \leq \frac{\sqrt{\lambda}}{4} + \frac{4}{\sqrt{\lambda}}  \log \left( \frac{N(h)}{\delta N(g)} \right) \right] \geq 1 - \delta.
        \end{align*}

        Using Lemma 6 of \cite{Faury_Abeille_Calauzenes_Fercoq_2020}, we can write
        \begin{equation*}
            \log \left( \frac{N(h)}{ N(g)} \right) \leq \log \left( 
        \frac{\det (\vect{H}_t)^{1/2}}{\lambda^{d/2}}\right) + d \log(2)
        \end{equation*}

    \end{proof}

    \begin{lemma}
        For all $\vect{\xi} \in \frac{1}{2B} \mathcal{B}(d)$, the process $\{M_t(\vect{\xi})\}_{t = 1}^{\infty}$ is a non-negative super-martingale.
        \label{lemma_supermartingale}
    \end{lemma}

    \begin{proof}
        To show that $\{M_t(\vect{\xi})\}_{t = 1}^{\infty}$ is a non-negative super-martingale, it is sufficient to show that $\bbE[M_{t+1}(\vect{\xi}) | \mathcal{F}_t ] \leq M_{t}(\vect{\xi})$ for all $t \geq 1$ and $\vect{\xi} \in \frac{1}{2B} \mathcal{B}(d)$. We have
        \begin{align*}
            \bbE[M_{t+1}(\vect{\xi}) | \mathcal{F}_t ] &= \bbE[\exp(\vect{\xi}^\top \vect{S}_{t+1} - \|\vect{\xi}\|^2_{\wt{\vect{H}}_{t+1}} ) | \mathcal{F}_t ]\\
            &= \bbE[\exp(\vect{\xi}^\top \vect{X}_{t} \vect{\epsilon}_{t} - \vect{\xi}^\top \vect{X}_{t} \Sigma_t \vect{X}_{t}^\top \vect{\xi} ) | \mathcal{F}_t ] M_{t}(\vect{\xi})\\
            &= \bbE[\exp(\vect{\xi}^\top \vect{X}_{t} \vect{\epsilon}_{t} ) | \mathcal{F}_t ] \exp(- \vect{\xi}^\top \vect{X}_{t} \Sigma_t \vect{X}_{t}^\top \vect{\xi}) M_{t}(\vect{\xi})
        \end{align*}
        Using Hölder's inequality, we can check that 
        \begin{align*}
            |\vect{\xi}^\top \vect{X}_{t} \vect{\epsilon}_{t}| \leq \|\vect{\epsilon}_{t}\|_1 \|\vect{X}_{t}^\top \vect{\xi}\|_{\infty} = \|\vect{\epsilon}_{t}\|_1 \max_{i \in [K]} |\vect{x}_{ti}^\top \vect{\xi}| \leq 1.
        \end{align*}
        Therefore, we can use Lemma 6 from \cite{Amani_Thrampoulidis_2021}, to write
        \begin{equation*}
            \bbE[\exp(\vect{\xi}^\top \vect{X}_{t} \vect{\epsilon}_{t}) | \mathcal{F}_t ] \leq \exp(\vect{\xi}^\top \vect{X}_{t} \Sigma_t \vect{X}_{t}^\top \vect{\xi}).
        \end{equation*}
        Consequently, we can show that $\bbE[M_{t+1}(\vect{\xi}) | \mathcal{F}_t ] \leq M_{t}(\vect{\xi})$ and complete the proof.
    \end{proof}

\section{Importance of Estimating Fisher Information Matrix}\label{sect:estimation_importance}

\cite{Oh_Iyengar_2021} has the best regret rate among efficient contextual MNL bandit algorithms in the literature. Their analysis shows that their algorithm has a regret rate of $\wt{\mathcal{O}}( \kappa d \sqrt{T} )$ where
\begin{align*}
    \kappa := \left( \min_{t, i} \inf_{\vect{\theta} : \|\vect{\theta} - \vect{\theta}^*\| \leq 1} q_{ti}(\vect{\theta}) q_{t0}(\vect{\theta}) \right)^{-1}.
\end{align*}

\textbf{As we will show in the following proof, computation of this parameter for our setting results in $\kappa = \mathcal{O}(K^{2 + 1/L_0})$ which translates into a $\wt{\mathcal{O}}(K^{2 + 1 / L_0} d \sqrt{T})$ regret bound.}

\begin{proof}
    We can write $q_{ti}(\vect{\theta}) q_{t0}(\vect{\theta})$ as
   \begin{align*}
       q_{ti}(\vect{\theta}) q_{t0}(\vect{\theta}) = \frac{\exp (\langle \vect{\psi}_t, \vect{x}_{ti} \rangle - \langle \vect{\phi}_t, \vect{x}_{ti} \rangle p_{ti})}{ \left( 1 + \sum_{j \in S_t} \exp (\langle \vect{\psi}_t, \vect{x}_{tj} \rangle - \langle \vect{\phi}_t, \vect{x}_{tj} \rangle p_{tj}) \right)^2}
   \end{align*}
   
   Since optimum prices lie in the interval $[0, P]$, and we have $\|\vect{\theta}_t\| \leq 2$ and $\|\vect{x}_{ti}\| \leq 1$, we can show that
   \begin{align*}
       q_{ti}(\vect{\theta}) q_{t0}(\vect{\theta}) > \frac{e^{-6-6P}}{\left(1+ K e^{2+2P}\right)^2} > \frac{e^{-2-2P}}{(K+1)^2}.
   \end{align*}
   
   Therefore, we can conclude that $\kappa = \mathcal{O}(K^2 e^P) = \mathcal{O}(K^{2 + \frac{1}{L_0}})$.
   
\end{proof}

\textbf{Furthermore, we can show that $\kappa > {(K-1)}^{\frac{1}{(1 + \epsilon) L_0}}$ for any $\epsilon > 0$. Therefore, the $K^{1 / L_0}$ dependency cannot be avoided for the regret of the algorithm provided in~\cite{Oh_Iyengar_2021}.}

\begin{proof}

Assume that $N = K \geq 2$. Consider $u_{t1}(p) = 1 - p$ and $u_{ti}(p) = 1 - L_0 p$ for $i \in [K] \setminus \{1\}$. Under this construction, we will show a lower bound for $\kappa$.

For $i \in [K] \setminus \{1\}$, we have
\begin{align*}
    v_{ti}(B) = \frac{1}{L_0} e^{- L_0 B}.
\end{align*}
and for $i = 1$, we have
\begin{align*}
    v_{ti}(B) = e^{- B}.
\end{align*}

Then, following Proposition~\ref{prop_opt_assortments_and_prices}, if we let $B_t$ be the unique solution of the fixed point equation
\begin{align}
    B = \frac{K - 1}{L_0} e^{- L_0 B} + e^{- B},
    \label{eq_Bt_estimation_importance}
\end{align}
the optimum prices are given by $p_{t 1}^* = B_t + 1$ and $p_{t i}^* = B_t + 1/L_0$ for $i \in [K] \setminus \{1\}$.

From Lemma~\ref{lemma_opt_prices}, we have the upper bound $B_t \leq P_0$. Next, we'll show a lower bound for $B_t$.

The right hand side (RHS) of~\eqref{eq_Bt_estimation_importance} is decreasing and its left hand side (LHS) is increasing in $B$. Therefore, if we let $B_\ell$ be the solution of the fixed point equation
\begin{align}
    B = \frac{K - 1}{L_0} e^{- L_0 B},
    \label{eq_Bt_estimation_importance_ub}
\end{align}
then we have $B_\ell \leq B_t$. In~\eqref{eq_Bt_estimation_importance_ub}, the LHS is increasing and the RHS is decreasing in $B$. Additionally, for $B = \log(K - 1) / ((1 + \epsilon) L_0)$ with any $\epsilon > 0$ and large enough $K$, the LHS of~\eqref{eq_Bt_estimation_importance_ub} is smaller than its RHS. Hence, the fixed point satisfies 
\begin{align*}
    P_\ell := \frac{\log(K - 1)}{(1 + \epsilon) L_0} \leq B_\ell \leq B_t.
\end{align*}

Now, we can write $q_{t 1}(\vect{\theta}^*) q_{t0}(\vect{\theta}^*)$ as
\begin{align*}
    q_{t 1}(\vect{\theta}^*) q_{t0}(\vect{\theta}^*) &= \frac{\exp (1 - (B_t + 1))}{ \left( 1 + (K - 1) \exp \left(1 - L_0 (B_t + \frac{1}{L_0}) \right) + \exp (1 - (B_t + 1)) \right)^2}\\
    &= \frac{e^{- B_t}}{ \left( 1 + (K - 1) e^{- L_0 B_t} + e^{- B_t} \right)^2}.
\end{align*}

From this expression, we can show that
\begin{align*}
    q_{t 1}(\vect{\theta}^*) q_{t0}(\vect{\theta}^*) < e^{-P_\ell}
\end{align*}

Therefore, we can conclude that $\kappa > {(K-1)}^{\frac{1}{(1 + \epsilon) L_0}}$.

\end{proof}

\section{Proof of Theorem~\ref{thm_regret_ub_online} (Regret Upper Bound for Algorithm~\ref{alg:seq_assortment_online})}
\label{proof_regret_ub_online}

Similar to Algorithm~\ref{alg:seq_assortment}, we run $T_0$ initialization rounds with random assortment and price selections to obtain an initial pilot estimate $\vect{\theta}_0 := \widehat{\vect{\theta}}_{T_0}$. Using the results of Lemma~\ref{lemma_initalization_min_eigenvalue} and Lemma~\ref{lemma_consistency}, we can show that the conditions $\lambda_{\mathrm{min}}(\vect{V}_{T_0}) \geq \lambda_{\mathrm{min}}^0$ and $\widehat{\vect{\theta}}_t \in \mathcal{B}_{\gamma/2}$ for all $t \geq T_0$ are satisfied with probability $1 - \mathcal{O}(T^{-1})$ if we select
\begin{align}
T_0 = \Theta\left( \frac{\lambda_{\mathrm{min}}^0 K}{\sigma_0} \right) = \Theta\left(\frac{d \widebar{P}^2 K \log^3(T)}{\sigma_0^3} \right).
\label{eqn_T0_online}
\end{align}

Then, we apply the following parameter update at each time step $t$:
\begin{align}
    \wh{\vect{\theta}}_t = \argmin_{\vect{\theta} :\|\vect{\theta} - \vect{\theta}_0\| \leq \gamma/2} \left\{  \frac{1}{2} \|\vect{\theta} - \wh{\vect{\theta}}_{t-1}\|_{\vect{V}_t}^2 + 4 (\vect{\theta} - \wh{\vect{\theta}}_{t-1})^\top g_t(\wh{\vect{\theta}}_{t-1}) \right\}
\label{eqn_parameter_update_online}
\end{align}
which directly ensures that $\|\widehat{\vect{\theta}}_{t} - \vect{\theta}^*\| \leq \gamma$ for all $t \geq T_0$ with probability $1 - \mathcal{O}(T^{-2})$. 

Then, using this update rule, we modify our algorithm and present it in the algorithm block for Algorithm \ref{alg:seq_assortment_online}. Algorithm \ref{alg:seq_assortment} requires $\Theta(tK)$ computational complexity to compute the MLE estimate in each round $t$. Since this cost grows linearly with each round $t$, the overall amortized computational cost turns out to be $\Theta(TK)$. On the other hand, the parameter update in Algorithm \ref{alg:seq_assortment_online}, only the $\Theta(K)$ context vector in the last offered assortment is needed per each round.

\begin{algorithm}[h]
\caption{CAP-ONS: CAP with online Newton steps}
\begin{algorithmic}[1]
\State \textbf{Input:} initialization rounds $T_0$, confidence parameters $\{\alpha_t\}_{t \in [T]}$, minimum price sensitivity $L_0$
\State $\vect{V}_0 \gets \vect{0} \in \bbR^{2d \times 2d}$
\For{$t = 1, 2, \dots, T_0$} \Comment{initialization rounds}
    \State Choose $S_t$ uniformly at random from $\{S \subseteq [N] : |S| \leq K\}$
    \State Choose $p_{ti}$ independently and uniformly at random from $[1,2]$ for all $i \in S_t$
    \State Offer assortment $S_t$ at price $\vect{p}_t$ and observe $i_t$
    \State $\vect{V}_{t+1} \gets \vect{V}_{t} + \frac{1}{K^2} \sum_{i \in S_t} \wt{\vect{x}}_{ti} \wt{\vect{x}}_{ti}^\top$
\EndFor
\State Compute MLE $\widehat{\vect{\theta}}_{T_0}$ by solving \eqref{eqn_mle} and set $\vect{\theta}_0 = \widehat{\vect{\theta}}_{T_0}$.
\For{$t = T_0 + 1, T_0 + 2, \dots, T$}
    \State Compute $\vect{\wh{\theta}}_t$ by solving \eqref{eqn_parameter_update_online}     \State Let $g_{ti}(p) := \alpha_t^{\mathrm{OL}} \|(\vect{x}_{ti}, -p\vect{x}_{ti})\|_{\vect{V}_{t}^{-1}}$ for all $i \in [n]$ \Comment{Price-dependent confidence function}
    \State Let $\wt{h}_{ti}^{\mathrm{OL}}(p) := \langle \widehat{\vect{\psi}}_t, \vect{x}_{ti} \rangle - \langle \widehat{\vect{\phi}}_t, \vect{x}_{ti} \rangle \cdot p + g_{ti}(p)$ for all $i \in [n]$
    \vspace{2pt}
    \State Let $ h_{ti}^{\mathrm{OL}}(p) := \min_{p' \leq p} \left\{ \wt{h}_{ti}^{\mathrm{OL}} (p') - L_0 (p - p') \right\}$ for all $i \in [n]$ \Comment{Utility function estimate}
    \vspace{2pt}
    \State Choose $(S_t, \vect{p}_t)$ using Algorithm \ref{alg:optimization_algo} with estimated utility functions $h_{ti}^{\mathrm{OL}}(p)$ 
    \State Offer assortment $S_t$ at price $\vect{p}_t$ and observe $i_t$
    \vspace{2pt}
    \State $\vect{V}_{t+1} \gets \vect{V}_{t} + \sum_{i \in S_t} q_{ti}(\vect{\wh{\theta}}_t) \wt{\vect{x}}_{ti} \wt{\vect{x}}_{ti}^\top - \sum_{i \in S_t} \sum_{j \in S_t} q_{ti}(\vect{\wh{\theta}}_t) q_{tj}(\vect{\wh{\theta}}_t) \wt{\vect{x}}_{ti} \wt{\vect{x}}_{tj}^\top$ \Comment{Information estimate}
\EndFor
\end{algorithmic}
\label{alg:seq_assortment_online}
\end{algorithm}

To analyze the regret of Algorithm~\ref{alg:seq_assortment_online}, we first define a per-round negative log-likelihood function $f_t(\vect{\theta})$ and its gradient $\nabla_{\vect{\theta}} f_t(\vect{\theta})$ as
\begin{align*}
    f_{t}(\vect{\theta}) &= - q_{t {i_t}}(\vect{\theta} )\\
    g_t(\vect{\theta}) = \nabla_{\vect{\theta}} f_t(\vect{\theta}) &= \sum_{i \in S_{t}} q_{ti}(\vect{\theta} ) \vect{\widetilde{x}}_{t i} - \vect{\widetilde{x}}_{t i_{t}}.
\end{align*}

We note that negative log-likelihood $f_t(\vect{\theta})$ for MNL model at each round $t$ is a strongly convex function over a bounded domain, which enables us to apply a variant of online Newton updates \citep{Hazan_Koren_Levy_2014} that was also used in \cite{Hazan_Koren_Levy_2014, Zhang_Yang_Jin_Xiao_Zhou_2016, Oh_Iyengar_2021} which proposed online algorithms for logistic models. 

To prove the regret rate for our algorithm with online parameter updates, we construct a new confidence region using a new confidence radius $\alpha^{\mathrm{OL}}_t$ specified in the following lemma. Then, the utility function upper-bound estimate $h^{\mathrm{OL}}_{ti}(p)$ is also modified accordingly.

\begin{lemma}
Let $T_0$ be any round such that $\lambda_{\mathrm{min}}(\vect{V}_{T_0}) \geq KP^2$. Then, for any $t > T_0$, we have $\|\widehat{\vect{\theta}}_t - \vect{\theta}^*\|_{\vect{V}_t} \leq \alpha^{\mathrm{OL}}_t$ with probability at least $1 - t^{-2}$ for confidence radius
\begin{align}
     \alpha^{\mathrm{OL}}_t = \sqrt{ \gamma^2 T_0 + 576 d K \log ( T / d ) + 16 \log \left( \frac{\lceil \log (K \gamma t^2 /  \omega) \rceil t^2}{\delta} \right) + 8 }.
\label{eqn_alpha_online}
\end{align}
where $\omega := \min_{\vect{\theta} \in \mathcal{B}_{\gamma}} q_{t i}(\wh{\vect{\theta}}_t) q_{t 0}(\wh{\vect{\theta}}_t)$ and satisfies $1/\omega = \mathcal{O}(K^{2+1/L_0})$.
\label{lemma_normality_online}
\end{lemma}

Then, similar to the proof of Theorem~\ref{thm_regret_ub}, we define a \emph{good} event $\wt{\mathcal{E}}_t = \{ \|\widehat{\vect{\theta}}_t - \vect{\theta}^*\|_{\vect{V}_t} \leq \alpha^{\mathrm{OL}}_t \}$ for $t \geq T_0$ that holds with probability at least $1 - t^{-2}$. Consequently, following steps similar to the proof of Theorem~\ref{thm_regret_ub}, we can write the regret as
\begin{align*}
\mathcal{R}_T \leq P T_0 + 4 P \alpha^{\mathrm{OL}}_T \sqrt{ T  \sum_{t=T_0}^{T} \sum_{i \in S_t} q_{ti}(\vect{\theta}^*) \| \wt{\vect{x}}_{ti} \|_{\vect{V}_{t}^{-1}}^2 } + \mathcal{O}(P).
\end{align*}
for $\alpha^{\mathrm{OL}}_T$ given in Lemma~\ref{lemma_normality_online}. Finally, using Lemma~\ref{lemma_cumulative_max_uncertainty} and Lemma~\ref{lemma_qti_ratio}, we show that
\begin{align*}
    \mathcal{R}_T \leq P T_0 + 29 P \alpha_T^{\mathrm{OL}} \sqrt{ d K T \log ( T / d )} + \mathcal{O}(P).
\end{align*}

Note that $\alpha_T^{\mathrm{OL}} = \wt{\mathcal{O}}(\sqrt{d K})$ for the selection of $T_0$ given in \eqref{eqn_T0_online}.

\subsection{Proof of Lemma~\ref{lemma_normality_online}}

The proof of Lemma~\ref{lemma_normality_online} depends on a few technical results we present next. First, we define the matrix 
\begin{align*}
    \vect{W}_t = \vect{H}_t(\wh{\vect{\theta}}_t) = \sum_{i \in S_{t}} q_{t i}(\wh{\vect{\theta}}_t) \vect{\widetilde{x}}_{t i} \vect{\widetilde{x}}_{t i}^\top - \sum_{i \in S_{t}} \sum_{j \in S_{t}} q_{t i}(\wh{\vect{\theta}}_t) q_{t j}(\wh{\vect{\theta}}_t) \vect{\widetilde{x}}_{t i} \vect{\widetilde{x}}_{t j}^\top.
\end{align*}
    
We start by showing that following bound holds true over $\mathcal{B}_{\gamma} := \{ \vect{\theta} : \|\vect{\theta} - \vect{\theta}^*\| \leq \gamma\}$.
\begin{lemma}
For any $\vect{\theta}_1, \vect{\theta}_2 \in \mathcal{B}_{\gamma}$, we have
\begin{align*}
    f_t(\vect{\theta}_2) \geq f_t(\vect{\theta}_1) + g_t(\vect{\theta}_1)^\top (\vect{\theta}_2 - \vect{\theta}_1) + \frac{1}{4} (\vect{\theta}_2 - \vect{\theta}_1)^\top \vect{W}_t (\vect{\theta}_2 - \vect{\theta}_1).
\end{align*}
\label{lemma_f_strong_convex}
\end{lemma}

\begin{proof}
    Using the Taylor's expansion, there exists some $c \in (0, 1)$ such that 
    \begin{align*}
        f_t(\vect{\theta}_2) = f_t(\vect{\theta}_1) + g_t(\vect{\theta}_1)^\top (\vect{\theta}_2 - \vect{\theta}_1) + (\vect{\theta}_2 - \vect{\theta}_1)^\top \vect{H}_t(\vect{\widebar{\theta}}) (\vect{\theta}_2 - \vect{\theta}_1)
    \end{align*}
    where $\vect{\widebar{\theta}} = c \, \vect{\theta}_2 + (1-c) \vect{\theta}_1$ and $\vect{H}_t(\vect{\widebar{\theta}})$ is the Hessian of $f_t$ at $\vect{\widebar{\theta}}$. Furthermore, by Lemma \ref{lemma_H_ineq}, we have $\vect{H}_t(\vect{\widebar{\theta}}) \succcurlyeq \frac{1}{4} \vect{H}_t(\wh{\vect{\theta}}_t) = \frac{1}{4} \vect{W}_t$. Consequently, the result follows.
\end{proof}

Next, we prove the following lemma that shows the dependency between the error $(\wh{\vect{\theta}}_t - \vect{\theta}^*)$ at time $t$ and the error $(\wh{\vect{\theta}}_{t+1} - \vect{\theta}^*)$ at time $t + 1$.

\begin{lemma} For any $t$,
\begin{align*}
    2 g_t(\wh{\vect{\theta}}_t)^\top (\wh{\vect{\theta}}_t - \vect{\theta}^*) \leq 4 \|g_t(\wh{\vect{\theta}}_t)\|_{\vect{V}_{t+1}^{-1}}^2 + \frac{1}{4} \|\wh{\vect{\theta}}_t - \vect{\theta}^*\|_{\vect{V}_{t+1}}^2 - \frac{1}{4} \|\wh{\vect{\theta}}_{t+1} - \vect{\theta}^*\|_{\vect{V}_{t+1}}^2.
\end{align*}
\label{lemma_g_relation}
\end{lemma}

\begin{proof}
Note that 
\begin{align*}
    \wh{\vect{\theta}}_{t+1} = \argmin_{\vect{\theta} : \|\vect{\theta} - \vect{\theta}_0\| \leq \gamma/2} \left\{  \frac{1}{2} \|\vect{\theta} - \wh{\vect{\theta}}_{t}\|_{\vect{V}_{t+1}}^2 + 4 (\vect{\theta} - \wh{\vect{\theta}}_{t})^\top g_t(\wh{\vect{\theta}}_{t}) \right\}. 
\end{align*}

From the first-order optimality condition, we have
\begin{align*}
    \left( 4 g_t(\wh{\vect{\theta}}_{t}) + \vect{V}_{t+1} (\wh{\vect{\theta}}_{t+1} - \wh{\vect{\theta}}_{t}) \right)^\top (\vect{\theta} - \wh{\vect{\theta}}_{t+1}) \geq 0
\end{align*}
for any $\vect{\theta}$ such that $\|\vect{\theta} - \vect{\theta}_0\| \leq \gamma/2$. We can rewrite this inequality as
\begin{align*}
    \vect{\theta}^\top \vect{V}_{t+1} (\wh{\vect{\theta}}_{t+1} - \wh{\vect{\theta}}_{t}) \geq \wh{\vect{\theta}}_{t+1}^{\,\top} \vect{V}_{t+1} (\wh{\vect{\theta}}_{t+1} - \wh{\vect{\theta}}_{t}) - 4 g_t(\wh{\vect{\theta}}_{t})^\top (\vect{\theta} - \wh{\vect{\theta}}_{t+1}).
\end{align*}
Then, we can write
\begin{align*}
    &\|\wh{\vect{\theta}}_{t} - \vect{\theta}^*\|_{\vect{V}_{t+1}}^2 - \|\wh{\vect{\theta}}_{t+1} - \vect{\theta}^*\|_{\vect{V}_{t+1}}^2 \\
    &= \wh{\vect{\theta}}_{t}^\top \vect{V}_{t+1} \wh{\vect{\theta}}_{t} - \wh{\vect{\theta}}_{t+1}^\top \vect{V}_{t+1} \wh{\vect{\theta}}_{t+1} + 2 {\vect{\theta}^*}^\top \vect{V}_{t+1} (\wh{\vect{\theta}}_{t+1} - \wh{\vect{\theta}}_{t})\\
    &\geq \wh{\vect{\theta}}_{t}^\top \vect{V}_{t+1} \wh{\vect{\theta}}_{t} - \wh{\vect{\theta}}_{t+1}^\top \vect{V}_{t+1} \wh{\vect{\theta}}_{t+1} + 2 \wh{\vect{\theta}}_{t+1}^\top \vect{V}_{t+1} (\wh{\vect{\theta}}_{t+1} - \wh{\vect{\theta}}_{t}) - 8 g_t(\wh{\vect{\theta}}_{t})^\top (\vect{\theta}^* - \wh{\vect{\theta}}_{t+1})\\
    &= \wh{\vect{\theta}}_{t}^\top \vect{V}_{t+1} \wh{\vect{\theta}}_{t} + \wh{\vect{\theta}}_{t+1}^\top \vect{V}_{t+1} \wh{\vect{\theta}}_{t+1} - 2 \wh{\vect{\theta}}_{t+1}^\top \vect{V}_{t+1} \wh{\vect{\theta}}_{t} - 8 g_t(\wh{\vect{\theta}}_{t})^\top (\vect{\theta}^* - \wh{\vect{\theta}}_{t+1})\\
    &= \|\wh{\vect{\theta}}_{t+1} - \wh{\vect{\theta}}_{t}\|_{\vect{V}_{t+1}}^2 + 8 g_t(\wh{\vect{\theta}}_{t})^\top (\wh{\vect{\theta}}_{t+1} - \wh{\vect{\theta}}_{t}) + 8 g_t(\wh{\vect{\theta}}_{t})^\top (\wh{\vect{\theta}}_{t} - \vect{\theta}^*)\\
    &\geq - 16 \| g_t(\wh{\vect{\theta}}_{t}) \|_{\vect{V}_{t+1}^{-1}}^2 + 8 g_t(\wh{\vect{\theta}}_{t})^\top (\wh{\vect{\theta}}_{t} - \vect{\theta}^*)
\end{align*}
where the last inequality follows from
\begin{align*}
\|\wh{\vect{\theta}}_{t+1} - \wh{\vect{\theta}}_{t}\|_{\vect{V}_{t+1}}^2 + 8 g_t(\wh{\vect{\theta}}_{t})^\top (\wh{\vect{\theta}}_{t+1} - \wh{\vect{\theta}}_{t}) &\geq \min_{\vect{\theta} : \|\vect{\theta} - \vect{\theta}_0\| \leq \gamma/2} \left\{ \|\vect{\theta} - \wh{\vect{\theta}}_{t}\|_{\vect{V}_{t+1}}^2 + 8 (\vect{\theta} - \wh{\vect{\theta}}_{t})^\top g_t(\wh{\vect{\theta}}_{t}) \right\}\\
&\geq \min_{\vect{\theta}} \left\{ \|\vect{\theta} - \wh{\vect{\theta}}_{t}\|_{\vect{V}_{t+1}}^2 + 8 (\vect{\theta} - \wh{\vect{\theta}}_{t})^\top g_t(\wh{\vect{\theta}}_{t}) \right\}\\
&= \min_{\vect{\theta}} \left\{ \|\vect{\theta}\|_{\vect{V}_{t+1}}^2 + 8 \vect{\theta}^\top g_t(\wh{\vect{\theta}}_{t}) \right\}\\
&= - 16 \| g_t(\wh{\vect{\theta}}_{t}) \|_{\vect{V}_{t+1}^{-1}}^2
\end{align*}

\end{proof}

Next, we let $\mathcal{F}_t$ denote the filtration up to time $t$ and define the conditional expected values for the per-round negative log-likelihood $f_t(\vect{\theta})$ and its gradient $g_t(\vect{\theta})$ as follows.
\begin{align*}
\widebar{f}_{t}(\vect{\theta}) &= \bbE_{i_t} [ f_{t}(\vect{\theta}) | \mathcal{F}_t ]\\
\widebar{g}_t(\vect{\theta}) &= \bbE_{i_t} [ g_{t}(\vect{\theta}) | \mathcal{F}_t ].
\end{align*}

\begin{lemma}
For any positive definite matrix $\vect{V}$,
\begin{align*}
    \|g_t(\vect{\theta})\|_{\vect{V}}^2 \leq 4 \max_{i \in S_t} \|\wt{\vect{x}}_{ti}\|_{\vect{V}}^2.
\end{align*}
\label{lemma_g_to_context}
\end{lemma}

\begin{proof}
Recall that $y_{ti}$ is a binary variable such that $y_{ti} = 1$ if $i_t = i$ and $y_{ti} = 0$ otherwise. For convenience also denote $q_{ti} = q_{t}(i | S_{t}, \vect{p}_{t} ; \vect{\theta})$. Then, we note that $\sum_{i \in S_t} q_{ti} \leq 1$ and $\sum_{i \in S_t} y_{ti} \leq 1$. Consequently, we can write
\begin{align*}
\|g_t(\vect{\theta})\|_{\vect{V}}^2 &= \sum_{i \in S_t} \sum_{j \in S_t} (q_{ti} - y_{ti}) (q_{tj} - y_{tj}) \, \wt{\vect{x}}_{ti}^\top \vect{V} \wt{\vect{x}}_{tj}\\
&\leq \sum_{i \in S_t} \sum_{j \in S_t} (q_{ti} q_{tj} + y_{ti} y_{tj} + q_{ti} y_{tj} + q_{tj} y_{ti}) \, |\wt{\vect{x}}_{ti}^\top \vect{V} \wt{\vect{x}}_{tj}|\\
&\leq 4 \max_{i,j \in S_t} |\wt{\vect{x}}_{ti}^\top \vect{V} \wt{\vect{x}}_{tj}|\\
&\leq 4 \max_{i \in S_t} |\wt{\vect{x}}_{ti}^\top \vect{V} \wt{\vect{x}}_{ti}|\\
&= 4 \max_{i \in S_t} \|\wt{\vect{x}}_{ti}\|_{\vect{V}}^2.
\end{align*}

\end{proof}

Then, we show that $\widebar{f}_{t}(\vect{\theta})$ is minimized at $\vect{\theta}^*$. Formally, we prove the following lemma.
\begin{lemma}
    For any $\vect{\theta} \in \bbR^{2d}$, we have $\widebar{f}_{t}(\vect{\theta}) \geq \widebar{f}_{t}(\vect{\theta}^*)$.
    \label{lemma_f_order}
\end{lemma}

\begin{proof}
For any $\vect{\theta} \in \bbR^{2d}$,
\begin{align*}
\widebar{f}_{t}(\vect{\theta}) - \widebar{f}_{t}(\vect{\theta}^*) &= \sum_{i \in S_t} q_{t}(i | S_{t}, \vect{p}_{t} ; \vect{\theta}^* ) [ \log q_{t}(i | S_{t}, \vect{p}_{t} ; \vect{\theta}^* ) - \log q_{t}(i | S_{t}, \vect{p}_{t} ; \vect{\theta} )]\\
&\geq 0
\end{align*}
since it is equal to the Kullback-Leibler (KL) divergence between distributions $q_{t}(i | S_{t}, \vect{p}_{t} ; \vect{\theta}^* )$ and $q_{t}(i | S_{t}, \vect{p}_{t} ; \vect{\theta} )$.
\end{proof}

\begin{lemma}
Suppose $\wh{\vect{\theta}}_t \in \mathbb{B}_\gamma$ for all $t \geq T_0$. Then, with probability at least $1 - \delta$,
\begin{align*}
    \sum_{\tau = T_0}^{t} &\left(\widebar{g}_\tau(\widehat{\vect{\theta}}_\tau) - g_\tau(\widehat{\vect{\theta}}_\tau) \right)^\top (\widehat{\vect{\theta}}_\tau - \vect{\theta}^*) \leq 2 \log \left( \frac{\lceil \log (K \gamma t^2 / \omega) \rceil t^2}{\delta} \right) + \frac{1}{8} \sum_{\tau  = T_0}^{t} \|\wh{\vect{\theta}}_\tau - \vect{\theta}^*\|_{\vect{W}_\tau}^2 + 1.
\end{align*}
where $\omega := \min_{\vect{\theta} \in \mathcal{B}_{\gamma}} q_{t i}(\wh{\vect{\theta}}_t) q_{t 0}(\wh{\vect{\theta}}_t)$ and satisfies $1/\omega = \mathcal{O}(K^{2+1/L_0})$.
\label{lemma_sum_gradients}
\end{lemma}

\begin{proof}
    The following proof is adapted from Lemma~14 of \cite{Oh_Iyengar_2021}. Note that $\xi_t = (\widebar{g}_t(\wh{\vect{\theta}}_t) - g_t(\wh{\vect{\theta}}_t))^\top (\wh{\vect{\theta}}_t - \vect{\theta}^*)$ is a martingale difference sequence and it satisfies
    \begin{align*}
        | (\widebar{g}_t(\wh{\vect{\theta}}_t) - g_t(\wh{\vect{\theta}}_t))^\top (\wh{\vect{\theta}}_t - \vect{\theta}^*) | &\leq | \widebar{g}_t(\wh{\vect{\theta}}_t)^\top (\wh{\vect{\theta}}_t - \vect{\theta}^*) | + | g_t(\wh{\vect{\theta}}_t)^\top (\wh{\vect{\theta}}_t - \vect{\theta}^*) | \\
        &\leq 2 \sqrt{2} \gamma \widebar{P} \\
        &\leq 2 \sqrt{2}.
    \end{align*}
    using the fact that $\|g_t(\wh{\vect{\theta}}_t)\| = \|\sum_{i \in S_t} (q_{\tau i}( \vect{\theta} ) - y_{t i}) \wt{\vect{x}}_{t i}\| \leq \sqrt{2} \widebar{P}$ for any $\vect{\theta}$. Therefore, 
    \begin{align*}
        M_t := \sum_{\tau = T_0}^{t} &\left(\widebar{g}_\tau(\widehat{\vect{\theta}}_\tau) - g_\tau(\widehat{\vect{\theta}}_\tau) \right)^\top (\widehat{\vect{\theta}}_\tau - \vect{\theta}^*)
    \end{align*}
    is a martingale. Now, we notice that $\bbE_{i_t} [\xi_t^2 | \mathcal{F}_t] = (\wh{\vect{\theta}}_t - \vect{\theta}^*)^\top \vect{H}_t(\wh{\vect{\theta}}_t) (\wh{\vect{\theta}}_t - \vect{\theta}^*) = \|\wh{\vect{\theta}}_t - \vect{\theta}^*\|_{\vect{W}_t}^2$ and define the random variable
    \begin{align*}
        B_t :=  \sum_{\tau  = T_0}^{t} \bbE_{i_\tau} [\xi_t^2 | \mathcal{F}_\tau] = \sum_{\tau  = T_0}^{t} \|\wh{\vect{\theta}}_\tau - \vect{\theta}^*\|_{\vect{W}_\tau}^2.
    \end{align*}
    
    In the following, we will show how we can use this quantity to upper-bound $M_t$. Since $B_t$ is a random variable, it is not possible to apply Freedman's inequality \cite{Freedman_1975} directly to $M_t$. Instead, we consider two cases with (i) $B_t \leq \frac{\omega}{t K}$ and (ii) $B_t > \frac{\omega}{t K}$ where $\omega := \min_{\vect{\theta} \in \mathcal{B}_{\gamma}} q_{t i}(\wh{\vect{\theta}}_t) q_{t 0}(\wh{\vect{\theta}}_t)$ as     introduced in Sect \ref{sect:estimation_importance}.

    \textbf{Case (i): } When $B_t \leq \frac{\omega}{t K}$, we have
    \begin{align*}
        M_t &= \sum_{\tau = T_0}^{t} \left(\widebar{g}_\tau(\widehat{\vect{\theta}}_\tau) - g_\tau(\widehat{\vect{\theta}}_\tau) \right)^\top (\widehat{\vect{\theta}}_\tau - \vect{\theta}^*) \\
        &= \sum_{\tau = T_0}^{t} \sum_{i \in S_\tau} \left( y_{\tau i} - q_{ti}(\vect{\theta}^*) \right) \wt{\vect{x}}_{\tau i}^\top (\widehat{\vect{\theta}}_\tau - \vect{\theta}^*)\\
        &\leq \sum_{\tau = T_0}^{t} \sum_{i \in S_\tau} \left| y_{\tau i} - q_{ti}(\vect{\theta}^*) \right| |\wt{\vect{x}}_{\tau i}^\top (\widehat{\vect{\theta}}_\tau - \vect{\theta}^*)|\\
        &\leq \sum_{\tau = T_0}^{t} \sum_{i \in S_\tau} |\wt{\vect{x}}_{\tau i}^\top (\widehat{\vect{\theta}}_\tau - \vect{\theta}^*)|\\
        &\leq \sqrt{t K \sum_{\tau = T_0}^{t} \sum_{i \in S_\tau} (\wt{\vect{x}}_{\tau i}^\top (\widehat{\vect{\theta}}_\tau - \vect{\theta}^*))^2}\\
        &= \sqrt{t K \sum_{\tau = T_0}^{t} (\widehat{\vect{\theta}}_\tau - \vect{\theta}^*)^\top \left( \sum_{i \in S_\tau} \wt{\vect{x}}_{\tau i} \wt{\vect{x}}_{\tau i}^\top \right)(\widehat{\vect{\theta}}_\tau - \vect{\theta}^*)}\\
        &\leq \sqrt{ \frac{t K}{\omega} \sum_{\tau = T_0}^{t} (\widehat{\vect{\theta}}_\tau - \vect{\theta}^*)^\top \vect{W}_\tau (\widehat{\vect{\theta}}_\tau - \vect{\theta}^*)}\\
        &= \sqrt{\frac{t K}{\omega} B_t}\\
        &\leq 1
    \end{align*}

    where we defined and used the result that 
    \begin{align*}
        \vect{W}_t &= \sum_{i \in S_{t}} q_{t i}(\wh{\vect{\theta}}_t) \vect{\widetilde{x}}_{t i} \vect{\widetilde{x}}_{t i}^\top - \sum_{i \in S_{t}} \sum_{j \in S_{t}} q_{t i}(\wh{\vect{\theta}}_t) q_{t j}(\wh{\vect{\theta}}_t) \vect{\widetilde{x}}_{t i} \vect{\widetilde{x}}_{t j}^\top\\
        &\succcurlyeq \sum_{i \in S_{t}} q_{t i}(\wh{\vect{\theta}}_t) q_{t 0}(\wh{\vect{\theta}}_t) \vect{\widetilde{x}}_{t i} \vect{\widetilde{x}}_{t i}^\top\\
        &\succcurlyeq \omega \sum_{i \in S_{t}} \vect{\widetilde{x}}_{t i} \vect{\widetilde{x}}_{t i}^\top.
    \end{align*}

    \textbf{Case (ii): } When $B_t > \frac{\omega}{t K}$, we have both a lower and upper bound for $B_t$, i.e., $\frac{\omega}{t K} < B_t \leq \gamma t$ since $\|\wh{\vect{\theta}}_t - \vect{\theta}^*\| \leq \gamma$ and $\vect{W}_t \preccurlyeq 1$ for all $t$. Then, we let $\eta_t$ to denote a constant and apply the peeling technique from \cite{Bartlett_Bousquet_Mendelson_2005} to obtain
    \begin{align*}
        \Pr \left( M_t \geq \sqrt{\eta_t B_t}  \right) &= \Pr \left( M_t \geq \sqrt{\eta_t B_t} , \frac{\omega}{t K} < B_t < \gamma t \right)\\
        &= \Pr \left( M_t \geq \sqrt{\eta_t B_t}, \frac{\omega}{t K} < B_t < \gamma t \right)\\
        &= \sum_{j = 1}^{m} \Pr \left( M_t \geq \sqrt{\eta_t B_t}, \frac{2^{j-1}  \omega}{t K} < B_t < \frac{2^{j}\omega}{t K } \right)\\
        &\leq \sum_{j = 1}^{m} \Pr \left( M_t \geq \sqrt{\eta_t \frac{2 \cdot 2^j  \omega}{t K}}, B_t < \frac{2^{j}  \omega}{t K} \right)\\
        &\leq 2 m \exp(- \eta_t)
    \end{align*}
    where we set $m = \lceil \log (K \gamma t^2 / \omega) \rceil$ and use the Freedman's inequality \cite{Freedman_1975} for the last inequality. 
    
    Combining the results from both cases, letting $\eta_t = \log \frac{m t^2}{\delta}$, and taking a union bound over $t$, we have
    \begin{align*}
        M_t &\leq \sqrt{\eta_t B_t} + 1\\
        &\leq 2 \eta_t + \frac{1}{8} B_t + 1
    \end{align*}
    where the last step uses the inequality $u v \leq c u^2 + v^2 /(4c)$.
\end{proof}

Now, we prove Lemma~\ref{lemma_normality_online} by using the previous results. First, we note that $\wh{\vect{\theta}}_t, \vect{\theta}^* \in \mathcal{B}_\gamma$ for $t \geq T_0$ by construction. Then, we use Lemma~\ref{lemma_f_strong_convex} to write
\begin{align*}
f_t(\wh{\vect{\theta}}_t) \leq f_t(\vect{\theta}^*) + g_t(\wh{\vect{\theta}}_t)^\top (\wh{\vect{\theta}}_t - \vect{\theta}^*) - \frac{1}{4} (\wh{\vect{\theta}}_t - \vect{\theta}^*)^\top \vect{W}_t (\wh{\vect{\theta}}_t - \vect{\theta}^*).
\end{align*}
Then, by taking the expectation over $i_t$ on both sides, we obtain
\begin{align*}
\widebar{f}_t(\wh{\vect{\theta}}_t) \leq \widebar{f}_t(\vect{\theta}^*) + \widebar{g}_t(\wh{\vect{\theta}}_t)^\top (\wh{\vect{\theta}}_t - \vect{\theta}^*) - \frac{1}{4} (\wh{\vect{\theta}}_t - \vect{\theta}^*)^\top \vect{W}_t (\wh{\vect{\theta}}_t - \vect{\theta}^*).
\end{align*}

Since $\widebar{f}_{t}(\vect{\theta}) \geq \widebar{f}_{t}(\vect{\theta}^*)$ by Lemma~\ref{lemma_f_order}, we have
\begin{align*}
    0 &\leq \widebar{f}_{t}(\wh{\vect{\theta}}_t) - \widebar{f}_{t}(\vect{\theta}^*)\\
    &\leq \widebar{g}_t(\wh{\vect{\theta}}_t)^\top (\wh{\vect{\theta}}_t - \vect{\theta}^*) - \frac{1}{4} \|\wh{\vect{\theta}}_t - \vect{\theta}^*\|_{\vect{W}_t}^2\\
    &= g_t(\wh{\vect{\theta}}_t)^\top (\wh{\vect{\theta}}_t - \vect{\theta}^*) - \frac{1}{4} \|\wh{\vect{\theta}}_t - \vect{\theta}^*\|_{\vect{W}_t}^2 + \left( \widebar{g}_t(\wh{\vect{\theta}}_t) - g_t(\wh{\vect{\theta}}_t) \right)^\top (\wh{\vect{\theta}}_t - \vect{\theta}^*).
\end{align*}

Then, using Lemma~\ref{lemma_g_relation} and Lemma~\ref{lemma_g_to_context}, we have
\begin{align*}
0 &\leq 2 \|g_t(\wh{\vect{\theta}}_t)\|_{\vect{V}_{t+1}^{-1}}^2 + \frac{1}{8} \|\wh{\vect{\theta}}_t - \vect{\theta}^*\|_{\vect{V}_{t+1}}^2 - \frac{1}{8} \|\wh{\vect{\theta}}_{t+1} - \vect{\theta}^*\|_{\vect{V}_{t+1}}^2 \\
& \qquad - \frac{1}{4} \|\wh{\vect{\theta}}_t - \vect{\theta}^*\|_{\vect{W}_t}^2 + \left( \widebar{g}_t(\wh{\vect{\theta}}_t) - g_t(\wh{\vect{\theta}}_t) \right)^\top (\wh{\vect{\theta}}_t - \vect{\theta}^*)\\
&\leq 2 \max_{i \in S_t} \|\wt{\vect{x}}_{ti}\|_{\vect{V}_{t+1}^{-1}}^2 + \frac{1}{8} \|\wh{\vect{\theta}}_t - \vect{\theta}^*\|_{\vect{V}_{t+1}}^2 - \frac{1}{8} \|\wh{\vect{\theta}}_{t+1} - \vect{\theta}^*\|_{\vect{V}_{t+1}}^2 \\
& \qquad - \frac{1}{4} \|\wh{\vect{\theta}}_t - \vect{\theta}^*\|_{\vect{W}_t}^2 + \left( \widebar{g}_t(\wh{\vect{\theta}}_t) - g_t(\wh{\vect{\theta}}_t) \right)^\top (\wh{\vect{\theta}}_t - \vect{\theta}^*)\\
&= 2 \max_{i \in S_t} \|\wt{\vect{x}}_{ti}\|_{\vect{V}_{t+1}^{-1}}^2 + \frac{1}{8} \|\wh{\vect{\theta}}_t - \vect{\theta}^*\|_{\vect{V}_{t}}^2 - \frac{1}{8} \|\wh{\vect{\theta}}_{t+1} - \vect{\theta}^*\|_{\vect{V}_{t+1}}^2 \\
& \qquad - \frac{1  }{8} \|\wh{\vect{\theta}}_t - \vect{\theta}^*\|_{\vect{W}_t}^2 + \left( \widebar{g}_t(\wh{\vect{\theta}}_t) - g_t(\wh{\vect{\theta}}_t) \right)^\top (\wh{\vect{\theta}}_t - \vect{\theta}^*)
\end{align*}
where the last equality follows by noting that we have
\begin{align*}
\|\wh{\vect{\theta}}_t - \vect{\theta}^*\|_{\vect{V}_{t+1}}^2 = \|\wh{\vect{\theta}}_t - \vect{\theta}^*\|_{\vect{V}_{t}}^2 + \|\wh{\vect{\theta}}_t - \vect{\theta}^*\|_{\vect{W}_{t}}^2
\end{align*}
since $\vect{V}_{t+1} = \vect{V}_t + \vect{W}_t$.

Hence, we have
\begin{align*}
\|\wh{\vect{\theta}}_{t+1} - \vect{\theta}^*\|_{\vect{V}_{t+1}}^2  &\leq \|\wh{\vect{\theta}}_t - \vect{\theta}^*\|_{\vect{V}_{t}}^2 + 16 \max_{i \in S_t} \|\wt{\vect{x}}_{ti}\|_{\vect{V}_{t+1}^{-1}}^2  - \|\wh{\vect{\theta}}_t - \vect{\theta}^*\|_{\vect{W}_t}^2\\
& \qquad  + 8 \left( \widebar{g}_t(\wh{\vect{\theta}}_t) - g_t(\wh{\vect{\theta}}_t) \right)^\top (\wh{\vect{\theta}}_t - \vect{\theta}^*).
\end{align*}

Summing over $\{T_0, \dots, t\}$, we obtain
\begin{align*}
\|\wh{\vect{\theta}}_{t+1} - \vect{\theta}^*\|_{\vect{V}_{t+1}}^2 &\leq \|\wh{\vect{\theta}}_t - \vect{\theta}^*\|_{\vect{V}_{T_0}}^2 + 16 \sum_{\tau = T_0}^{t} \max_{i \in S_\tau} \|\wt{\vect{x}}_{\tau i}\|_{\vect{V}_{\tau +1}^{-1}}^2 - \sum_{\tau = T_0}^{t} \|\widehat{\vect{\theta}}_\tau - \vect{\theta}^*\|_{\vect{W}_\tau}^2 \\
& \qquad + 8 \sum_{\tau = T_0}^{t} \left( \widebar{g}_\tau(\wh{\vect{\theta}}_\tau) - g_\tau(\wh{\vect{\theta}}_\tau) \right)^\top (\wh{\vect{\theta}}_\tau - \vect{\theta}^*).
\end{align*}

Then, Lemma~\ref{lemma_sum_gradients} shows with a probability at least $1 - \delta$,
\begin{align*}
 &\|\wh{\vect{\theta}}_{t+1} - \vect{\theta}^*\|_{\vect{V}_{t+1}}^2 \\
 &\qquad \leq \|\wh{\vect{\theta}}_t - \vect{\theta}^*\|_{\vect{V}_{T_0}}^2 + 16 \sum_{\tau = T_0}^{t} \max_{i \in S_\tau} \|\wt{\vect{x}}_{\tau i}\|_{\vect{V}_{\tau +1}^{-1}}^2 + 16 \log \left( \frac{\lceil \log (K \gamma t^2 / \omega) \rceil t^2}{\delta} \right) + 8\\
&\qquad \leq \gamma^2 \lambda_\mathrm{max}(\vect{V}_{T_0}) + 576 d K \log ( T / d ) + 16 \log \left( \frac{\lceil \log (K \gamma t^2 / \omega ) \rceil t^2}{\delta} \right) + 8\\
&\qquad \leq \gamma^2 T_0 + 576 d K \log ( T / d ) + 16 \log \left( \frac{\lceil \log (K \gamma t^2 / \omega) \rceil t^2}{\delta} \right) + 8
\end{align*}
where we apply Lemma~\ref{lemma_cumulative_max_uncertainty} for the last step.

\section{Proof of Theorem~\ref{thm_regret_lb}}
\label{lb_proofs}

At a high level, we prove Theorem~\ref{thm_regret_lb} in three steps. In the first step, we construct an adversarial set of parameters and reduce the task of lower bounding the worst-case regret of any policy to lower bounding the Bayes risk over the constructed parameter set. In the second step, we use a counting argument similar to the one used in \cite{Chen_Wang_2018} and \cite{Chen_Wang_Zhou_2020} to provide an explicit lower bound on the Bayes risk of the constructed adversarial parameter set. Finally, we apply Pinsker’s inequality to complete the proof. The following sections provide the details for each of these steps.

\subsection{Adversarial construction and the Bayes risk}

Let $\epsilon \in (0, (1-L_0^2)/d\sqrt{d})$ be a small positive parameter to be specified later.  For every subset $W \subseteq [d]$, define the corresponding parameter $\vect{\psi}_W \in \bbR^d$ as $[\vect{\psi}_W]_i = \epsilon$ for all $i \in W$, and $[\vect{\psi}_W]_i = 0$ for all $i \notin W$. Next, define $\vect{\phi}^* \in \bbR^d$ as $[\vect{\phi}^*]_i = L_0\sqrt{1/d}$ for all $i \in [d]$. Finally, for any $W \subseteq [d]$, define the concatenated parameter vectors $\vect{\theta}_W \in \bbR^{2d}$ as $\vect{\theta}_W = (\vect{\psi}_W, \vect{\phi}^*)$. The parameter set that we consider is
\begin{align*}
    \vect{\theta} \in \Theta := \{ \vect{\theta}_W : W \in \mathcal{W}_{d/4}\}
\end{align*}
where $\mathcal{W}_{d/4} := \{W \subseteq [d] : |W| = d/4\}$ denotes the set of all subsets of $[d]$ whose size is $d/4$. Note that $d/4$ is a positive integer because $d$ is divisible by $4$. It is also easy to check that with the condition $\epsilon \in (0, (1-L_0^2)/\sqrt{d})$, we satisfy $\|\vect{\theta}\| \leq 1$ for any $\vect{\theta} \in \Theta$.

The feature vectors $\{\vect{x}_{ti}\}$ are constructed to be invariant over time iterations $t$. For each $t$ and $U \in \mathcal{W}_{d/4}$, $K$ identical feature vectors $\vect{x}_U$ are constructed as $[\vect{x}_U]_i = 2/\sqrt{d}$ for all $i \in U$, and $[\vect{x}_U]_i = 0$ for all $i \notin U$. Furthermore, it is straightforward to verify that $\|\vect{x}_U\| \leq 1$ for any $U \in \mathcal{W}_{d/4}$.

Hence, the worst-case regret of any policy $\pi$ can be lower bounded by the worst-case regret of parameters belonging to $\Theta$, which can be further lower bounded by the average regret over a uniform prior over $\Theta$. Formally,
\begin{align}
    \sup_{\vect{\theta}} \bbE_{\vect{x}, \vect{\theta}}^{\pi} \sum_{t = 1}^{T} R(S_{\vect{\theta}}^*, \vect{p}_{\vect{\theta}}^*) - R(S_t, \vect{p}_t) &= \max_{\vect{\theta} \in \Theta} \bbE_{\vect{x}, \vect{\theta}}^{\pi} \sum_{t = 1}^{T} R(S_{\vect{\theta}}^*, \vect{p}_{\vect{\theta}}^*) - R(S_t, \vect{p}_t)\\
    &= \frac{1}{|\mathcal{W}_{d/4}|} \sum_{W \in \mathcal{W}_{d/4}} \bbE_{\vect{x}, \vect{\theta}_W}^{\pi} R(S_{\vect{\theta}_W}^*, \vect{p}_{\vect{\theta}_W}^*) - R(S_t, \vect{p}_t)
    \label{eq_bayes_risk}
\end{align}

Here, the $R(\cdot)$ function refers to the expected revenue function $R_t(\cdot)$ defined in \eqref{eqn_rev_t}. Since both the context vectors and the feature vectors are invariant over time by construction, we drop the time subscript $t$ to simplify the notation. Additionally, $S_{\vect{\theta}_W}^*$ and $ \vect{p}_{\vect{\theta}_W}^*$ refer to the optimal size-$K$ assortment and pricing that maximizes expected revenue under the feature parameter $\vect{\theta}_W$. By construction, it is easy to verify that $S_{\vect{\theta}_W}^*$ consists of all $K$ items corresponding to feature $\vect{x}_{W}$.

For any fixed assortment $S \in \mathcal{S}_K$, let $\vect{p}^*(S)$ denote the revenue-maximizing price vector to offer with assortment $S$. That is,
\begin{align*}
    \vect{p}^*(S) \in \max_{\vect{p} \in \bbR_{+}^{n}} R(S, \vect{p})
\end{align*}
with entries $p_{i}^*(S)$. Then, the optimum prices $\vect{p}_{\vect{\theta}_W}^* = \vect{p}^*(S_{\vect{\theta}_W}^*)$ can be characterized using the following proposition which is a special case of the Proposition~\ref{prop_opt_assortments_and_prices}.

\begin{proposition}
 Consider that items in an assortment $S$ of size $K$ have utility functions $u_i(p) = \alpha_i - \beta_i \cdot p$. Then, the revenue-maximizing prices for offering assortment $S$ are given by
    \begin{align*}
        p_{i}^*(S) = \frac{1}{\beta_i} + B^0(S)
    \end{align*}
    where $B^0(S)$ is the unique fixed point solution $B$ of the equation
    \begin{align*}
    B = \sum_{i \in S} \frac{1}{\beta} e^{\alpha_i - \beta_i B - 1}.
    \end{align*} 
    Furthermore, the revenue achieved by offering $(S, \vect{p}^*(S))$ is equal to $B^0(S)$.
\label{prop_identical_prices}
\end{proposition}

In particular, if all items in an assortment $S$ have the same utility function $u_i(p) = \alpha - \beta \cdot p$, then we can write $B^0(S)$ as the fixed point solution of
\begin{align*}
B = \frac{K}{\beta} e^{\alpha - \beta B - 1}.
\end{align*} 

\subsection{The counting argument}

In this section, we derive an explicit lower bound on the Bayes risk in \eqref{eq_bayes_risk}. For any sequence $\{(S_t, \vect{p}_t)\}_{t = 1}^{T}$ produced by the policy $\vect{\pi}$, we first describe an alternative sequence $\{(\wt{S}_t, \wt{\vect{p}}_t)\}_{t = 1}^{T}$ that provably enjoys less regret under the feature parameter $\vect{\theta}_W$.

Let $\{\vect{x}_{U_1}, \dots, \vect{x}_{U_M}\}$ be the set of context vectors of items contained in assortment $S_t$ (if $S_t = \emptyset$, then choose an arbitrary feature vector $\vect{x}_{U}$). Let $\wt{U}_t$ be the subset among $U_1, \dots, U_M$ that maximizes $\langle \vect{x}_{\wt{U}_t}, \vect{\psi}_W \rangle$, where $\vect{\theta}_W = (\vect{\psi}_W, \vect{\phi}^*)$ is the underlying parameter. Let $\wt{S}_t$ be the assortment consisting of all $K$ items corresponding to the feature $\vect{x}_{\wt{U}_t}$ and let $\wt{\vect{p}}_t = \vect{p}^*(\wt{S}_t)$ be the optimum prices for assortment $\wt{S}_t$ according to Proposition~\ref{prop_identical_prices}. Then, the following lemma holds true.

\begin{lemma}
    $R(S_t, \vect{p}_t) \leq R(\wt{S}_t, \wt{\vect{p}}_t)$ for feature parameter $\vect{\theta}_W = (\vect{\psi}_W, \vect{\phi}^*)$. 
\end{lemma}

\begin{proof}

First, from the optimality of prices $\vect{p}^*(S_t)$ under $S_t$, we have $R(S_t, \vect{p}_t) \leq R(S_t, \vect{p}^*(S_t))$. Then, by Proposition~\ref{prop_identical_prices}, $R(S_t, \vect{p}^*(S_t))$ is equal to the unique fixed point solution for
\begin{align*}
B = \sum_{i \in S} \frac{1}{\beta} e^{\alpha_i - \beta_i B - 1}.
\end{align*}
Note that the expression on the right-hand side of this equation is monotonically increasing in each $\alpha_i$. Therefore, by replacing all $i \in S_t$ with $i \in \wt{S}_t$, the $\alpha_{i}$ values do not decrease and therefore the fixed point does not increase. That is, the fixed-point solution for
\begin{align}
B = \sum_{i \in \wt{S}_t} \frac{1}{\beta} e^{\alpha_i - \beta_i B - 1}.
\label{eqn_fixed_point_revenue_order}
\end{align}
is greater than or equal to $R(S_t, \vect{p}^*(S_t))$. Since the unique fixed point solution of \eqref{eqn_fixed_point_revenue_order} is equal to $R(\wt{S}_t, \wt{\vect{p}}_t)$, we have $R(S_t, \vect{p}^*(S_t)) \leq R(\wt{S}_t, \wt{\vect{p}}_t)$, completing the proof.

\end{proof}

To simplify notation, we use $\bbE_{W}$ to denote the expectations under parameter $\theta_W$ and policy $\pi$. The following lemma gives a lower bound for $R(S_{\vect{\theta}_W}^*, \vect{p}_{\vect{\theta}_W}^*) - R(\wt{S}_t, \wt{\vect{p}}_t)$. 

\begin{restatable}{lemma}{lemmalbrevgap}
    Suppose $\epsilon \in (0, 1/d\sqrt{d})$ and define $\delta := d/4 - |\wt{U}_t \cap W|$. Then,
    \begin{align*}
        R(S_{\vect{\theta}_W}^*, \vect{p}_{\vect{\theta}_W}^*) - R(\wt{S}_t, \wt{\vect{p}}_t) \geq \frac{\delta \epsilon}{15 L_0 \sqrt{d}}
    \end{align*}
    \label{lemma_lb_rev_gap}
\end{restatable}

Define random variables $\wt{N}_i := \sum_{t = 1}^{T} \1\{ i \in \wt{U}_t \}$. Lemma~\ref{lemma_lb_rev_gap} immediately implies
\begin{align*}
    \bbE_W \left[ R(S_{\vect{\theta}_W}^*, \vect{p}_{\vect{\theta}_W}^*) - R(\wt{S}_t, \wt{\vect{p}}_t) \right] \geq \frac{\epsilon}{15 L_0 \sqrt{d}} \left( \frac{d T}{4} - \sum_{i \in W} \bbE_W[\wt{N}_i] \right), \forall W \in \mathcal{W}_{d/4}.
\end{align*}

Summing both sides of this equation over all $W \in \mathcal{W}_{d/4}$ gives
\begin{align*}
    \sum_{W \in \mathcal{W}_{d/4}} \bbE_W \left[ R(S_{\vect{\theta}_W}^*, \vect{p}_{\vect{\theta}_W}^*) - R(\wt{S}_t, \wt{\vect{p}}_t) \right] \geq \frac{\epsilon}{15 L_0 \sqrt{d}} \sum_{W \in \mathcal{W}_{d/4}} \left( \frac{d T}{4} - \sum_{i \in W} \bbE_W[\wt{N}_i] \right).
\end{align*}

Next, we will upper-bound the term $\sum_{W \in \mathcal{W}_{d/4}} \sum_{i \in W} \bbE_W[\wt{N}_i]$. First, define 
\begin{align*}
    \mathcal{W}^{(i)}_{d/4} := \{W \in \mathcal{W}_{d/4} : i \in W\}.
\end{align*}
Then, we swap the order of summation to write
\begin{align*}
    \sum_{W \in \mathcal{W}_{d/4}} \sum_{i \in W} \bbE_W[\wt{N}_i] &= \sum_{i \in [d]} \sum_{W \in \mathcal{W}^{(i)}_{d/4}} \bbE_W[\wt{N}_i]\\
    &= \sum_{i \in [d]} \sum_{W \in \mathcal{W}_{d/4 - 1}} \bbE_{W \cup \{i\}}[\wt{N}_i]\\
    &\leq |\mathcal{W}_{d/4 - 1}| \max_{W \in \mathcal{W}_{d/4 - 1}} \sum_{i \in [d]} \bbE_{W \cup \{i\}}[\wt{N}_i]\\
    &= |\mathcal{W}_{d/4 - 1}| \max_{W \in \mathcal{W}_{d/4 - 1}} \sum_{i \in [d]} \left( \bbE_{W}[\wt{N}_i] + \bbE_{W \cup \{i\}}[\wt{N}_i] - \bbE_{W}[\wt{N}_i] \right)\\
    &\leq |\mathcal{W}_{d/4 - 1}| \left[\max_{W \in \mathcal{W}_{d/4 - 1}} \sum_{i \in [d]} \left( \bbE_{W \cup \{i\}}[\wt{N}_i] - \bbE_{W}[\wt{N}_i] \right) + \frac{d T}{4} \right]
\end{align*}
where the last step follows from the fact that $\sum_{i \in [d]} \bbE_{W}[\wt{N}_i] \leq dT/4$ for any fixed $W \in \mathcal{W}_{d/4 - 1}$.

Next, we note that
\begin{align*}
    \frac{|\mathcal{W}_{d/4 - 1}|}{|\mathcal{W}_{d/4}|} = \frac{\binom{d}{d/4 - 1}}{\binom{d}{d/4}} = \frac{d/4}{3d/4+1} \leq \frac{1}{3}
\end{align*}
to write
\begin{align*}
    \frac{1}{|\mathcal{W}_{d/4}|} \sum_{W \in \mathcal{W}_{d/4}} &\bbE_W \left[ R(S_{\vect{\theta}_W}^*, \vect{p}_{\vect{\theta}_W}^*) - R(\wt{S}_t, \wt{\vect{p}}_t) \right] \\
    &\geq \frac{1}{|\mathcal{W}_{d/4}|} \frac{\epsilon}{15 L_0 \sqrt{d}} \sum_{W \in \mathcal{W}_{d/4}} \left( \frac{d T}{4} - \sum_{i \in W} \bbE_W[\wt{N}_i] \right)\\
    &\geq \frac{\epsilon}{15 L_0 \sqrt{d}} \left( \frac{d T}{4} - \frac{1}{|\mathcal{W}_{d/4}|} \sum_{W \in \mathcal{W}_{d/4}} \sum_{i \in W} \bbE_W[\wt{N}_i] \right)\\
    &\geq \frac{\epsilon}{45 L_0 \sqrt{d}} \left( \frac{d T}{2} - \max_{W \in \mathcal{W}_{d/4 - 1}} \sum_{i \in [d]} \left| \bbE_{W \cup \{i\}}[\wt{N}_i] - \bbE_{W}[\wt{N}_i] \right| \right)
\end{align*}

\subsection{Pinsker's inequality}

In this section, we upper bound $\left| \bbE_{W \cup \{i\}}[\wt{N}_i] - \bbE_{W}[\wt{N}_i] \right|$ for any fixed $W \in \mathcal{W}_{d/4 - 1}$. Let $\bbP_{W}$ and $\bbP_{W  \cup \{i\}}$ to denote the probability law under parameter  $\theta_{W}$ and $\theta_{W  \cup \{i\}}$, respectively. Then,
\begin{align*}
    \left| \bbE_{W \cup \{i\}}[\wt{N}_i] - \bbE_{W}[\wt{N}_i] \right| &\leq \sum_{n = 0}^{T} n \cdot \left| \bbP_{W}[\wt{N}_i = n] - \bbP_{W  \cup \{i\}}[\wt{N}_i = n] \right|\\
    &\leq T \cdot \sum_{n = 0}^{T} \left| \bbP_{W}[\wt{N}_i = n] - \bbP_{W  \cup \{i\}}[\wt{N}_i = n] \right|\\
    &\leq 2 T \cdot \|\bbP_{W} - \bbP_{W  \cup \{i\}}\|_{\mathrm{TV}}\\
    &\leq T \sqrt{2 \cdot \mathrm{KL}(\bbP_{W} || \bbP_{W  \cup \{i\}})}
\end{align*}
where $\|P-Q\|_{\mathrm{TV}} = \sup_{A} |P(A) - Q(A)|$ is the total variation distance between laws $P$ and $Q$; $\mathrm{KL}(P || Q) = \int (\log dP/ dQ) dP$ is the Kullback-Leibler (KL) divergence between $P$ and $Q$; and the inequality $\|P-Q\|_{\mathrm{TV}} \leq \sqrt{\frac{1}{2} \mathrm{KL}(P || Q) }$ is the Pinsker's inequality.

Recall that $\{\vect{x}_{U_1}, \dots, \vect{x}_{U_M}\}$ denotes the set of context vectors of items contained in assortment $S_t$. Then, for every $i \in [d]$, define a new random variable $N_i := \frac{1}{K} \sum_{t = 1}^{T} \sum_{j = 1}^{M} \1\{i \in U_j\}$. The next lemma is used to upper bound the KL divergence term $\mathrm{KL}(\bbP_{W} || \bbP_{W  \cup \{i\}})$.

\begin{lemma}[Lemma~6 in \cite{Chen_Wang_Zhou_2020}]
    For any $W \in \mathcal{W}_{d/4-1}$ and $i \in [d]$, 
    \begin{align*}
        \mathrm{KL}(\bbP_{W} || \bbP_{W  \cup \{i\}}) \leq C_{\mathrm{KL}} \cdot \bbE_{W} [N_i] \cdot \epsilon^2 / d
    \end{align*}
    for some universal constant $C_{\mathrm{KL}} > 0$.
    \label{lemma_kl_div}
\end{lemma}

Combining Lemma~\ref{lemma_kl_div} with the final result of the previous subsection, we obtain
\begin{align*}
    \frac{1}{|\mathcal{W}_{d/4}|} \sum_{W \in \mathcal{W}_{d/4}} &\bbE_W \left[ R(S_{\vect{\theta}_W}^*, \vect{p}_{\vect{\theta}_W}^*) - R(\wt{S}_t, \wt{\vect{p}}_t) \right] \\
    &\geq \frac{\epsilon}{45 L_0 \sqrt{d}} \left( \frac{d T}{2} - T \sum_{i \in [d]} \sqrt{2 C_{\mathrm{KL}} \cdot \bbE_{W} [N_i] \cdot \epsilon^2 / d} \right)\\
    &\geq \frac{\epsilon}{45 L_0 \sqrt{d}} \left( \frac{d T}{2} - T \epsilon \sqrt{2 C_{\mathrm{KL}} \sum_{i \in [d]} \bbE_{W} [N_i]} \right)\\
    &\geq \frac{\epsilon}{45 L_0 \sqrt{d}} \left( \frac{d T}{2} - T \epsilon \sqrt{ C_{\mathrm{KL}}' dT} \right)
\end{align*}
where $C_{\mathrm{KL}}' = C_{\mathrm{KL}} / 2$. Setting $\epsilon = \sqrt{d / 16 C_{\mathrm{KL}}' T} \in (0, (1-L_0^2)/d\sqrt{d})$ for sufficiently large $T$, we obtain
\begin{align*}
    \sup_{\vect{\theta}} \bbE_{\vect{x}, \vect{\theta}}^{\pi} \sum_{t = 1}^{T} R(S_{\vect{\theta}}^*, \vect{p}_{\vect{\theta}}^*) - R(S_t, \vect{p}_t) \geq C_0 d \sqrt{T} / L_0
\end{align*}
for some universal constant $C_0$, completing the proof of the theorem. 

\subsection{Proofs for Technical Lemmas}

\lemmalbrevgap*

\begin{proof}
    The optimum revenue from offering $K$ identical items with utility functions $u(p) = \alpha - \beta p$ is equal to the unique fixed point solution $B$ of the equation
    \begin{align}
    B = \frac{K}{\beta} e^{\alpha - \beta B - 1}.
    \end{align} 

    Using the product logarithm function $W(\cdot)$, we can express the optimum revenue as
    \begin{align}
        \frac{W(e^{\alpha -1} K)}{\beta}
    \end{align}

    Let $f_K(x) := W(e^{x -1} K)$ and denote its first derivative with $f_K'(x)$ for any $K \geq 1$. Then, by Lemma~\ref{lemma_prod_log}, there exists a constant $C_K < \frac{2}{3} f_K'(0)$ such that
    \begin{align*}
        f_K(0) + f_K'(0) \cdot x \leq f_K(x) \leq f_K(0) + f_K'(0) \cdot x + C_K \cdot x^2 
    \end{align*}
    for all $0 \leq x \leq 1$. For the remainder of this proof, let $\vect{x} = \vect{x}_{W}$, $\wt{\vect{x}} = \vect{x}_{\wt{U}_t}$, and $\vect{\theta} = \vect{\theta}_W$. Then, we can write
    \begin{align*}
        R(S_{\vect{\theta}_W}^*, \vect{p}_{\vect{\theta}_W}^*) = f_K(\vect{x}^\top \vect{\theta}) \quad \text{and} \quad R(\wt{S}_t, \wt{\vect{p}}_t) = f_K(\wt{\vect{x}}^\top \vect{\theta}).
    \end{align*}

    Putting it all together, we can show that
    \begin{align*}
        R(S_{\vect{\theta}_W}^*, \vect{p}_{\vect{\theta}_W}^*) - R(\wt{S}_t, \wt{\vect{p}}_t) &\geq \frac{1}{L_0} \left[ (f_K(0) + f_K'(0) \vect{x}^\top \vect{\theta}) - \left(f_K(0) + f_K'(0) \wt{\vect{x}}^\top \vect{\theta}  + C_K (\wt{\vect{x}}^\top \vect{\theta})^2 \right) \right]\\
        &= \frac{1}{L_0} \left[ f_K'(0) (\vect{x}-\wt{\vect{x}})^\top \vect{\theta} - C_K (\wt{\vect{x}}^\top \vect{\theta})^2 \right]\\
        &\geq \frac{f_K'(0)}{L_0} \left[ (\vect{x}-\wt{\vect{x}})^\top \vect{\theta} - \frac{2}{3} (\wt{\vect{x}}^\top \vect{\theta})^2 \right]\\
        &\geq \frac{f_K'(0)}{L_0} \left[ \frac{\delta \epsilon}{\sqrt{d}} - \frac{2 d \epsilon^2}{3} \right]\\
        &\geq \frac{f_K'(0) \delta \epsilon}{3 L_0\sqrt{d}}
    \end{align*}
    where the last three inequalities use the inequality $0 < f''(0) < f_K'(0)$, the definition of $\delta$, and the inequality $d \epsilon^2 \leq \delta \epsilon/\sqrt{d}$ provided that $\epsilon \in (0, 1/d\sqrt{d})$. Lastly, noting that $f_K'(0) > 1/5$ by Lemma~\ref{lemma_prod_log} for any $K \geq 1$, we conclude the proof.
\end{proof}

\begin{lemma}
    Let $f_K(x) := W(e^{x -1} K)$ and denote its first derivative with $f_K'(x)$. Then, for any $K \geq 1$, 
    \begin{enumerate}[label=(\alph*)]
    \item $f_K'(x) > 1/5$ for all $0 \leq x \leq 1$, and
    \item there exists a constant $C_K < \frac{2}{3} f_K'(0)$ such that
    \begin{align*}
        f_K(0) + f_K'(0) \cdot x \leq f_K(x) \leq f_K(0) + f_K'(0) \cdot x + C_K \cdot x^2 
    \end{align*}
    for all $0 \leq x \leq 1$.
    \end{enumerate}
    \label{lemma_prod_log}
\end{lemma}

\begin{proof}
    Let $f_K''(x)$ and $f_K^{(3)}(x)$ denote the second and third derivatives of $f_K(x)$ respectively. Using the properties of the product logarithm function, it is easy to show that 
\begin{align*}
    f_K'(x) = \frac{f_K(x)}{1 + f_K(x)}, \qquad f_K''(x) = \frac{f_K(x)}{(1 + f_K(x))^3}, \qquad f_K^{(3)}(x) = \frac{(1 - 2f_K(x)) f_K(x)}{(1 + f_K(x))^5}.
\end{align*}

For any $K \geq 1$, $f_K(x)$ is a positive and increasing function of $x$. Hence, $\min_{0 \leq x \leq 1} f_K'(x) = f_K'(0)$. Furthermore, we can show that
\begin{align*}
    \min_{K \geq 1} f_K'(0) = \min_{K \geq 1} \frac{W(K/e)}{1+W(K/e)} = \frac{W(1/e)}{1+W(1/e)} > 1/5
\end{align*}
proving the first part of the lemma.

To prove the second part of the lemma, we use Taylor's Theorem to write
\begin{align*}
        f_K(x) &= f_K(0) + f_K'(0) \cdot x + \frac{f_K''(0)}{2} \cdot x^2 + R_K(\zeta; x)\\
        R_K(\zeta; x) &= \frac{f_K^{(3)}(\zeta)}{6} x^3
\end{align*}
for some $\zeta$ between $0$ and $x$. For any $K \geq 3$, we can easily show that $f_K(x) \geq 1/2$ for all $0 \leq x \leq 1$. Therefore, $R_K(\zeta; x) \leq 0$ for all $0 \leq \zeta \leq x \leq 1$ and we can set $C_K = f_K''(0)/2$ to satisfy the upper bound inequality.

On the other hand, for $K=1$ and $K=2$, we can numerically show that
\begin{align*}
    \max_{0 \leq \zeta \leq 1} f_K^{(3)}(\zeta) = f_K^{(3)}(0).
\end{align*}
and $f_K^{(3)}(0) \leq f_K''(0)$. Therefore, we have
\begin{align*}
    R_K(\zeta; x) \leq \frac{f_K''(0)}{6} \cdot x^2
\end{align*}
for all $0 \leq \zeta \leq x \leq 1$ when $K = 1$ or $K=2$. As a result, we can set $C_K = 2 f_K''(0)/3$ to satisfy the upper bound inequality.

Since $f_K''(0) < f_K'(0)$ for any $K \geq 1$, the selected constant $C_K$ also satisfies $C_K < \frac{2}{3} f_K'(0)$.

\end{proof}

\section{Experimental Details}
\label{appendix_exp_details}

We numerically evaluate our algorithms over $20$ independently generated problem instances and provide our results in Figure \ref{fig:experimental_results}. We run experiments with $n = 100$ items for various assortment sizes $K$ and various numbers of feature dimensions $d$. In each instance, the parameter $\vect{\psi}^*$ is uniformly chosen from ${\{\vect{\psi} : \|\vect{\psi}\|_2 = 1/2 \}}$. On the other hand, price sensitivity parameter $\vect{\phi}^*$ is generated by independently drawing its entries from a uniform distribution over $[\sqrt{L_0}/\sqrt{d}, 1/\sqrt{2d}]$ for some parameter $L_0 > 0$. Each context vector $\vect{x}_{ti}$ is generated by independently drawing its entries over $[\sqrt{L_0}/\sqrt{d}, 1/\sqrt{2d}]$. This construction ensures that we satisfy both Assumptions~\ref{assumption_positive_sens} and \ref{assumption_stochastic_contexts}. 

Figure~\ref{fig:regret_growth} demonstrates that the regret of CAP algorithm follows a $T^\alpha$ dependency with an empirically observed slope of $\alpha \approx 0.5$. This result aligns with the theoretical regret rate of $O(\sqrt{T})$ we obtained in this work.

\begin{figure}[ht]
    \centering
    \includegraphics[width=0.4\linewidth]{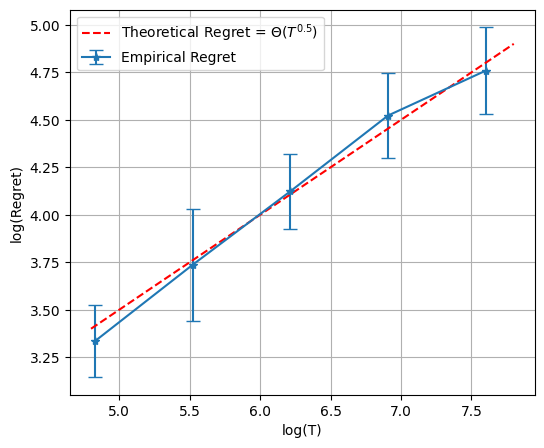}
    \caption{Log-log plot illustrating the dependency of regret for our proposed algorithm CAP. The slope of the curve reflects the empirical growth rate of regret with respect to time horizon $T$.}
    \label{fig:regret_growth}
\end{figure}

\end{document}